\documentclass[10pt,journal,compsoc]{IEEEtran}
\ifCLASSOPTIONcompsoc
  \usepackage[nocompress]{cite}
\else
  \usepackage{cite}
\fi

\usepackage{graphicx}

\graphicspath{{../pdf/}{../jpeg/}}
\usepackage{amsthm}
\usepackage{amsmath}
\usepackage{amsfonts}
\usepackage{amssymb}
\usepackage{algorithmic}
\usepackage{bbm}
\usepackage{mathtools}
\usepackage[ruled,linesnumbered]{algorithm2e}
\usepackage{array}
\usepackage{url}
\usepackage{booktabs}
\usepackage{cases}

\newtheorem{lemma}{Lemma}
\newtheorem{proposition}{Proposition}
\usepackage{verbatim}
\usepackage{bm}
\usepackage{epstopdf}

\begin{document}

\title{Globally Optimal Vertical Direction Estimation \\in Atlanta World}

\author{Yinlong~Liu, Guang~Chen  and~Alois~Knoll

\IEEEcompsocitemizethanks{\IEEEcompsocthanksitem Yinlong~Liu and Alois~Knoll are with the Department of Informatics, Technische Universit\"at M\"unchen, M\"unchen,
Germany, 85748.\protect\\
E-mail: Yinlong.Liu@tum.de and knoll@in.tum.de
\IEEEcompsocthanksitem Guang Chen is with Tongji University, Shanghai, China, and with Technische Universit\"at M\"unchen, M\"unchen,
Germany. \protect\\
E-mail: guang@in.tum.de
}
\thanks{Manuscript received August 19, 20**; revised August 26, 20**.}}

\markboth{Journal of \LaTeX\ Class Files~20**}%
{Shell \MakeLowercase{\textit{et al.}}: Globally Optimal Vertical Direction Estimation in Atlanta World}
%


\IEEEtitleabstractindextext{%
\begin{abstract}
In man-made environments, such as indoor and urban scenes, most of the objects and structures are organized in the form of orthogonal and parallel planes. These planes can be approximated by the Atlanta world assumption, in which the normals of planes can be represented by the Atlanta frames. Atlanta world assumption, which can be considered as a generalized Manhattan world assumption, has one vertical frame and multiple horizontal frames. Conventionally, given a set of inputs such as surface normals, the Atlanta frame estimation problem can be solved in one-time by branch-and-bound (BnB). However, the runtime of the BnB algorithm will increase greatly when the dimensionality (i.e., the number of horizontal frames) increases. In this paper, we estimate only the vertical direction instead of all Atlanta frames at once.  Accordingly, we propose a vertical direction estimation method by considering the relationship between the vertical frame and horizontal frames. Concretely, our approach employs a BnB algorithm to search the vertical direction guaranteeing global optimality without requiring prior knowledge of the number of Atlanta frames. Four novel bounds by mapping 3D-hemisphere to a 2D region are investigated to guarantee convergence. We verify the validity of the proposed method in various challenging synthetic and real-world data.

\end{abstract}

}

\maketitle

\IEEEdisplaynontitleabstractindextext

%
\IEEEpeerreviewmaketitle

\IEEEraisesectionheading{\section{Introduction}\label{sec:introduction}}

%
%
%
%
\IEEEPARstart{I}{n} man-made environments, scenes usually have structural forms (e.g., the layout of buildings and many indoor objects such as furniture), which can be represented by a set of parallel and orthogonal planes~\cite{straub2018manhattan}. Atlanta world makes an assumption that
the man-made scene can be modeled by a horizontal plane (e.g., ground plane) and many vertical planes (e.g., buildings and walls), then the normals of the planes, which are called world frames, can describe the scenes abstractly. In other words, one vertical frame and multiple horizontal frames could represent Atlanta world~\cite{schindler2004atlanta,joo2018globally}. Therefore, it is a crucial step to estimate these vertical and horizontal frame directions in computer vision applications, which is named Atlanta frame estimation~\cite{joo2018globally,joo2019globally}. More specifically, structural world frame estimation could be utilized as key modules for various high-level vision applications such as scene understanding~\cite{hedau2009recovering,straub2018manhattan} and SLAM~\cite{sunderhauf2012switchable,zhou2015structslam}.

Mathematically, an orientation in 3D Euclidean space corresponds to a point in the 3D unit sphere (i.e., $\mathbb{S}^2$). This means that the Atlanta frame estimation which estimates multiple orientations is a multiple-clustering (also multi-model fitting) problem in $\mathbb{S}^2$. There have been lots of general multiple-clustering algorithms~\cite{magri2016multiple,barath2018multi,amayo2018geometric} and some of them have been applied in structural world frame estimation~\cite{kim2017multi,tardif2009non}. However, Atlanta frame estimation is not exactly the same as the general multiple-clustering problem. It has some special constraints that all horizontal frames are in a plane and the vertical frame is parallel to the normal of the plane. These special constraints reflect essential properties of the Atlanta world assumption. If these constraints are omitted, it will not only lead to a significant decrease in accuracy but also increase the dimensionality of the problem. Furthermore, most of the multiple-clustering algorithms cannot guarantee global optimality when there are lots of outliers in observations~\cite{bazin2012globally-ACCV,bazin2012globally-CVPR}. Therefore, recent developments in structural world frame estimation highlight the imminent need for robust and globally optimal methods by considering the above special orthogonal constraints~\cite{joo2019robust,joo2019globally}.

Recently, Manhattan frame estimation~\cite{straub2018manhattan}, which is a special case of the Atlanta frame estimation, is solved efficiently by a branch-and-bound (BnB) method with the orthogonal constraints~\cite{joo2019robust}. 
However, when the BnB method is extended to the Atlanta world~\cite{joo2018globally,joo2019globally}, two problems appear,
\begin{enumerate}

\item The algorithm requires the number of Atlanta frames to be specified, which can rarely be known in advance. Although an automatic method  is proposed to estimate the number of horizontal directions in \cite{joo2019globally}, if it is over- or under-estimated, the global optimum may not occur in the correct direction.
\item It will suffer the curse of dimensionality. There are a considerable number of horizontal directions, whose relationships are unknown which is different from Manhattan world assumption. Consequently, the dimensionality of the problem will increase with the number of  horizontal frames, and thus the runtime of the BnB algorithm will increase greatly. 
\end{enumerate}

In this paper, we focus on estimating the unique vertical direction instead of all directions in Atlanta world at once. There are two advantages in comparison with the one-time solving all directions methods as follows:
\begin{enumerate}
\item More flexible. The vertical direction is unique in Atlanta frames, and we can estimate the vertical direction even though we don't know the total number of the horizontal directions. Additionally, we can also estimate the vertical direction from some irregular Atlanta world scenes (e.g., cylindrical buildings in Atlanta, whose horizontal directions number $\rightarrow \infty$).

\item More efficient. Vertical direction estimation is solved in a closed two dimensional space $\mathbb{S}^2$, which is a low-dimensional problem. In other words, only estimating vertical direction can significantly avoid the curse of dimensionality in Atlanta world.


\end{enumerate}

Furthermore, estimating the vertical direction first is always favorable to following operations in practical applications (e.g., scene classification~\cite{gupta2013perceptual}, parsing indoor scenes~\cite{taylor2013parsing} and point set registration~\cite{cai2019practical}). Specially, it is also helpful for estimating other horizontal Atlanta frames, because given the vertical direction, all other horizontal directions will be in a plane, and estimating the other horizontal directions will be a one-dimensional clustering problem in angular space~\cite{joo2019globally}.  In other words, given the vertical direction in Atlanta world, it is easy to estimate other horizontal directions with or without knowing the number of  horizontal frames ~\cite{bishop2006pattern,joo2019globally}.

\subsection{Related Work}

There is a large body of literature concerned with structural world frame estimation~\cite{joo2019globally,straub2018manhattan,joo2019robust,lee2017line}. Since it is a clustering problem in $\mathbb{S}^2$ with some orthogonal constraints, we first review the works that apply the classical clustering or fitting method. With the definition of Atlanta world, \textit{Expectation Maximization} (EM) type algorithms, which are popular for solving the chicken-and-egg problems~\cite{li20073d}, are applied in direction estimation~\cite{schindler2004atlanta}. However, the EM-type algorithms are local methods and have no guarantee of the global optimality. Therefore, there is an evident risk of local minima, and their performances rely heavily on a good initialization~\cite{antunes2013global}. Besides, the \textit{RANdom SAmple Consensus} (RANSAC)~\cite{fischler1981random,choi1997performance,raguram2013usac} based multi-structure estimation algorithms (e.g., T-linkage~\cite{magri2014t} and J-linkage~\cite{toldo2008robust}) are applied in structural direction estimation~\cite{tardif2009non,kim2017multi}. These RANSAC-type methods are fast, accurate and have the best performances in many cases, but the their solution is sub-optimal due to their obvious heuristic nature~\cite{joo2019globally}. More recently, Straub  \textit{et al}.~\cite{straub2018manhattan} propose a real-time capable inference algorithm by considering the orthogonal constraints, which uses an adaptive Markov-Chain Monte-Carlo sampling algorithm.


To assure global optimality, J. Bazin \textit{et al}. propose globally optimal methods~\cite{joo2019globally,bazin2012globally-CVPR,bazin2012globally-ACCV,joo2018globally,joo2019robust} by applying branch-and-bound algorithm to solve a consensus set maximization problem. The fundamental theory of these global methods is rotation search~\cite{hartley2009global,parra2016_thesis}. Specifically, the problem is solved by combining \textit{Interval Analysis} theory with BnB algorithm in~\cite{bazin2012globally-CVPR}. By contrast, the method in~\cite{bazin2012globally-ACCV} is a natural application of Hartley and Kahl's rotation search theory in $SO(3)$~\cite{hartley2009global}. Furthermore, 2D-EGI (Extended Gaussian Image) and its integral image are applied in~\cite{joo2019robust} to accelerate the calculation of the bounds in rotation search. Besides, rotation search theory is also extended to Atlanta frame estimation in~\cite{joo2018globally,joo2019globally}. 

However, Atlanta world is more complex than Manhattan world geometrically, since it has more than three frames. Consequently, the globally searching method in~\cite{joo2018globally} requires the number of horizontal directions to be hand-tuned according to the scene, which seems unrealistic in practical applications. Therefore, an automatic two-stage method (meta-BnB) is proposed in~\cite{joo2019globally} to estimate the number of directions. 
Concretely, it first searches the vertical direction and the horizontal plane in $SO(3)$, then it estimates the horizontal directions in one dimensional angle space. It is worth noting that the meta-BnB is also based on rotation search theory in $SO(3)$. However, searching vertical directions is inherently optimized in $\mathbb{S}^2$, whose dimensionality is less than that of $SO(3)$.

Since the rotation search theory is closely related to our work, we then briefly review the rotation search theory in computer vision filed. 
The rotation search theory has achieved great success in geometric vision problems, for example, point set registration~\cite{yang2016go,campbell2018globally}, camera calibration~\cite{seo2009branch,heller2016globally} and relative pose estimation~\cite{yang2014optimal,hartley2009global}. Because of the great success of rotation search, there have been several works focusing on improving the efficiency of the algorithm~\cite{bustos2016fast,Straub_2017_CVPR,joo2019robust,campbell2018globally}. 

More specifically, most of the rotation search methods rely on the two following lemmas~\cite{hartley2009global}:
\begin{lemma}\label{Lemma 1}
For $\forall\bm{x}\in \mathbb{S}^2$, $R_a, R_b \in SO(3)$, then
\begin{equation}
\angle(R_a\bm{x},R_b\bm{x})\leq d_\angle(R_a,R_b)
\end{equation}
\end{lemma}
\noindent where $d_\angle(R_a,R_b)$ is the angle lying in the range $[0,\pi] $ of the rotation $R_a R_b^{-1}$ and $\angle(\cdot,\cdot) $ denotes the angular distance between vectors.
\begin{lemma}\label{Lemma 2}
For $\forall R_a,R_b\in SO(3)$, then
\begin{equation}
d_\angle(R_a,R_b)\leq\Vert \bm{r}_a-\bm{r}_b \Vert
\end{equation}
\end{lemma}
\noindent where $\bm{r}_a$ and $\bm{r}_b$ are their corresponding angle-axis representations. In Lemma~\ref{Lemma 2}, there is a clear indication that the angle distance of two rotations is less than the Euclidean distance in their angle-axis representation. These two lemmas are the basis for the success of the rotation search theory.

Additionally, it is also worth noting that the rotation search usually means optimization in $SO(3)$, which is closely related to $\mathbb{S}^3$. Precisely, the homomorphism from a unit quaternion sphere (i.e., $\mathbb{S}^3$) to $SO(3)$ is a two-to-one mapping, and then the searching domain $SO(3)$ may be expressed as a hemisphere (including equator) of the unit quaternion sphere~\cite{Straub_2017_CVPR,hartley2013rotation}. However, the estimation of directions in three-dimensional Euclidean space (i.e., Manhattan or Atlanta frame) is inherently optimized in $\mathbb{S}^2$. Unfortunately, there is still a lack of rigid theories regarding globally optimal optimization in $\mathbb{S}^2$. In order to estimate the vertical directions in Atlanta world, we originally propose some new and solid mathematical conclusions about searching in $\mathbb{S}^2$.


\subsection{Our Contribution}

In this paper, to overcome the curse of the dimensionality and avoid the difficulty of requiring the user to specify the number of Atlanta frames,
we propose a novel  method for vertical direction estimation in Atlanta world.
The contributions of this work are mainly as follows:
\begin{itemize}
 
 \item We propose a global searching method for estimating vertical direction, which is different from conventional rotation search in $SO(3)$~\cite{joo2019globally}. Since the domain of the vertical directions is inherently in $\mathbb{S}^2$, then our searching method is more efficient in vertical direction estimation. 
 
 \item Four new different bounds for BnB algorithm are investigated.  In contrast to rotation search theory in $SO(3)$, more parametrizations for hemisphere are considered, including exponential mapping, stereographic projection and sphere coordinate system.  To  the best of our knowledge, it is the first to propose such bounds in $\mathbb{S}^2$ to solve structural world frame estimation problem.
\end{itemize} 
 
 \section{Methods}
 \subsection{Problem Formulation}
In this paper, we estimate the vertical direction from the surface normals in Atlanta world. We denote the input normal set as $\mathcal{N}=\lbrace \bm{n}_j\rbrace_{j=1}^N$,  where $\bm{n}_j \in \mathbb{S}^2$ is the  $j$-th effective unit normal, and $N$ is the number of input normals. 
In addition, the unknown-but-sought vertical direction is denoted as $\bm{v}$. It is 
in a hemisphere ($\mathbb{S}^2_\ominus$), which is defined as:
\begin{equation}
\mathbb{S}^2_\ominus=\lbrace \bm{x}\in\mathbb{S}^2\vert x_3\geq0\rbrace
\end{equation}
where $\bm{x}=[x_1,x_2,x_3]^T$ is a unit vector in $\mathbb{R}^3$.
Accordingly, the angle of vertical direction and one of the surface normals is lying in range $[0,\pi]$.

To estimate vertical direction robustly, we then apply the inlier maximisation approach to formulate the objective function as
\begin{subequations}
\begin{equation}
\max_{\bm{v}\in \mathbb{S}^2_\ominus}\sum_{j=1}^{N} \mathbbm{1}\Big(\mathcal{S}^\perp_j \vee\mathcal{S}_j^{\parallel^+} \vee\mathcal{S}_j^{\parallel^-} \Big)
\end{equation}
\begin{align}
\mathcal{S}^{\parallel^+}_j &=\mathbbm{1}\big(\angle(\bm{v},\bm{n}_j)\leq\tau\big) 
\label{Eq:inlier_par_0}
\\
\mathcal{S}^{\parallel^-}_j &=\mathbbm{1}\big( \angle(\bm{v},\bm{n}_j)\geq\pi-\tau \big)
\label{Eq:inlier_par_pi}
\\
\mathcal{S}^\perp_j&=\mathbbm{1}\big( \vert\angle(\bm{v},\bm{n}_j)-\dfrac{\pi}{2}\vert\leq\tau\big) \label{Eq:inlier_perp}
\end{align}	
\end{subequations}

\noindent where $\mathbbm{1}(\cdot)$ is an indicator function which returns 1 if the condition · is true and 0 otherwise and $\vee$ is the logical \textit{OR}   operation.  $\vert\cdot\vert$ is abs function and $0<\tau<\pi/2$ is the inlier threshold.  Eq.~\eqref{Eq:inlier_par_0}, \eqref{Eq:inlier_par_pi} and \eqref{Eq:inlier_perp}
mean that only the surface normals, which are parallel or perpendicular to vertical direction, are inliers. Additionally, because $\bm{n}_j^T\bm{v}=\cos(\angle(\bm{v},\bm{n}_j))$,  and when $x\in[0,\pi]$, $\cos(x)$ is a monotonically decreasing function, then
an equivalent formulation can be given by
\begin{subequations}\label{obj-sum}
\begin{equation}
\max_{\bm{v}\in \mathbb{S}^2_\ominus}\sum_{j=1}^N\mathbbm{1}\Big ( Q_j^{\parallel}\vee Q_j^{\perp}
\Big )
\end{equation}
\begin{align}
Q_j^{\parallel}&= \mathbbm{1}\big(\vert\bm{n}_j^T\bm{v}\vert\geq\cos(\tau) \big)
\\
Q_j^{\perp}&=\mathbbm{1}\big( \vert\bm{n}_j^T\bm{v}\vert\leq\sin(\tau) \big)
\end{align}	
\end{subequations}

\noindent Since there is no $\arccos$ operation to solve angle in the reformulations, it is more efficient than operating angle inequations.

In rotation search~\cite{joo2019globally}, it finds an optimal rotated motion $R$ rather than the optimal direction vector $\bm{v}$ directly. Concretely, given a initial direction vector $\bm{v}_0=[0,0,1]^T$ and because $R\in SO(3)$, then $\bm{v}=R\bm{v}_0\in\mathbb{S}^2$. For estimating vertical direction, it is sufficient to search the entire rotation domain and find the optimal $R$ to satisfy that $R\bm{v}_0$ is the optimal vertical direction.

\subsection{Branch-and-Bound}

Finding the optimal $\bm{v} \in \mathbb{S}^2_\ominus$ to maximize the cardinality of the inlier set is by no mean a trivial problem~\cite{fisher1993statistical,Johnson1977The}. Additionaly, the outlier observations, which are unavoidable in the real applications, increase the “hardness” of the estimation problem. Because it is well known that a general robust estimation with outlier observations is an NP-hard problem~\cite{Chin2018RobustFI}. 

To obtain the robust optimal vertical direction, we then use the BnB algorithm. The BnB algorithm is one of the most commonly used tools for solving NP-hard optimization problems, and it is widely applied in many global optimization problems~\cite{Morrison2016BranchandboundAA}. Briefly, the BnB algorithm recursively divides the searching space into smaller spaces and estimates the upper bound and lower bound of the optimum in each subspace. Then, it removes the sub-spaces which cannot produce a better solution than the best one found so far by the algorithm. The above process is repeated until the best optimum is found within the desired accuracy. The BnB algorithm for estimating vertical direction globally in Atlanta world is outlined in \textbf{Algorithm~\ref{algorithm}}. It is worth noting that the algorithm only needs the surface normals and the inlier threshold as the inputs without the prior knowledge of the number of horizontal frames.

\begin{algorithm}
	\caption{Estimating vertical direction globally }\label{algorithm}
	\KwIn{surface normal set $\lbrace \bm{n}_j\rbrace_{j=1}^N$, inlier threshold $\tau$}
	\KwOut{optimal vertical direction $\bm{v}^*$}
	
	Initialize the searching domain $\mathbb{D}$,  upper bound $U\leftarrow N$, lower bound $L\leftarrow0$, the best branch $\mathbb{B}\leftarrow\mathbb{D}$ and a queue $q\leftarrow\emptyset$\;
	
	\While{$\vert U-L \vert\geq1$}
	{
		Divide the best branch $\mathbb{B}$ into sub-branches\;
		Add the sub-branches with their bounds into $q$\;
		Update $L\leftarrow\max\{L_i\}$, $U\leftarrow\max\{U_i\}$\ for all branches in $q$\;
		Remove the branch that $U_i<L$ in $q$ \;
		Update the best branch $\mathbb{B}$, which has the maximum upper bound in $q$\;
		Remove the best branch from $q$\;
	}
	$\bm{v}^*\leftarrow$ center point in best branch $\mathbb{B}$ 
\end{algorithm}

The key of the BnB algorithm is estimating the upper and lower bounds of the optimum in each subspace tightly and efficiently. In this paper, two general bounds are proposed as follows:

\begin{proposition}[General bounds-1]\label{Pro:General bounds-1}
	Given a branch $\mathbb{B}$, if $\exists\bm{v}_c\in\mathbb{B}$, $\forall\bm{v}\in\mathbb{B},  \overline{\phi}_j\triangleq\max\angle(\bm{v},\bm{n}_j)$, and $\underline{\phi}_j\triangleq\min\angle(\bm{v},\bm{n}_j)$ then the upper bound  can be: 
	
	\begin{subequations}\label{Eq:general upper bound-1}
		
		\begin{equation}
		U_{s}(\mathbb{B})= \sum_{j=1}^N \mathbbm{1}\Big( 
		\overline{\mathcal{S}}_j^{\perp }
		\vee		\overline{\mathcal{S}}_j^{\parallel^+}
		\vee
		\overline{\mathcal{S}}_j^{\parallel^-}
		 \Big)
		\end{equation}
		
		\begin{align}
		\overline{\mathcal{S}}_j^{\parallel^+}&=\mathbbm{1}\big(\underline{\phi}_j\leq\tau\big)
		\\
		\overline{\mathcal{S}}_j^{\parallel^-}&=\mathbbm{1}\big(\overline{\phi}_j\geq\pi-\tau\big)
		\\
		\overline{\mathcal{S}}_j^{\perp}&=\mathbbm{1}\big(\underline{\phi}_j-\tau\leq\frac{\pi}{2}\leq\overline{\phi}_j+\tau\big)
		\end{align}
	\end{subequations}

	\noindent  the lower bound can be: 
	\begin{subequations}\label{Eq:general lower bound-1}
		\begin{equation}
		L_{s}(\mathbb{B})= \sum_{j=1}^N 
		\mathbbm{1}\Big(
		\underline{\mathcal{S}}_j^{\perp }
		\vee
		\underline{\mathcal{S}}_j^{\parallel^+}
		\vee
		\underline{\mathcal{S}}_j^{\parallel^-}
		\Big)
		\end{equation}
		
		\begin{align}
		\underline{\mathcal{S}}_j^{\parallel^+}&=\mathbbm{1}\big(\angle(\bm{v}_c,\bm{n}_j)\leq\tau\big)
		\\
		\underline{\mathcal{S}}_j^{\parallel^-}&=\mathbbm{1}\big(\angle(\bm{v}_c,\bm{n}_j)\geq\pi-\tau\big)
		\\
		\underline{\mathcal{S}}_j^{\perp}&=\mathbbm{1}\big(|\angle(\bm{v}_c,\bm{n}_j)-\frac{\pi}{2}|\leq\tau\big)
		\end{align}
	\end{subequations}

\end{proposition}
\begin{proof}
	The rigorous proof can be found in appendix~A.
\end{proof}
\begin{proposition}[General bounds-2]\label{Pro:General bounds-2}
	Given a branch $\mathbb{B}$, if $\exists\bm{v}_c\in\mathbb{B}$, $\forall\bm{v}\in\mathbb{B},  \angle(\bm{v},\bm{v}_c)\leq\max\angle(\bm{v},\bm{v}_c)\triangleq\psi$, then the upper bound  can be: 
	
	\begin{subequations}\label{Eq:general upper bound-2}
		
		\begin{equation}
		U_{q}(\mathbb{B})= \sum_{j=1}^N \mathbbm{1}\Big(
			\overline{Q}_j^{\parallel}
		\vee
		\overline{Q}_j^{\perp }	
		\Big)
		\end{equation}
		
		\begin{align}
		\overline{Q}_j^{\parallel}&=\mathbbm{1}\big(\vert\bm{n}_j^T\bm{v}_c\vert\geq\cos(\lfloor\tau+\psi\rfloor)\big)
		\\
		\overline{Q}_j^{\perp}&=\mathbbm{1}\big(\vert\bm{n}_j^T\bm{v}_c\vert\leq\sin(\lfloor\tau+\psi\rfloor)\big)
		\end{align}
	\end{subequations}
where 
\begin{equation}
\lfloor\tau+\psi\rfloor=\begin{cases}
\tau+\psi,& \tau+\psi<\pi/2 \\
\pi/2, &\tau+\psi\geq\pi/2 
\end{cases}
\end{equation}
\noindent  the lower bound can be: 
	\begin{subequations}\label{Eq:general lower bound-2}
		\begin{equation}
		L_{q}(\mathbb{B})= \sum_{j=1}^N 
		\mathbbm{1}\Big(
		\underline{Q}_j^{\parallel}
		\vee
		\underline{Q}_j^{\perp }
		\Big)
		\end{equation}
		
		\begin{align}
		\underline{Q}_j^{\parallel}&=\mathbbm{1}\big(\vert\bm{n}_j^T\bm{v}_c\vert\geq\cos(\tau)\big)
		\\
		\underline{Q}_j^{\perp}&=\mathbbm{1}\big(\vert\bm{n}_j^T\bm{v}_c\vert\leq\sin(\tau)\big)
		\end{align}
	\end{subequations}

\end{proposition}
\begin{proof}
	The completed proof can be found in appendix~B.
\end{proof}

Actually, if they have the same $\bm{v}_c$ in both general bounds, then $L_s=L_q$. The main difference between general bounds-1 and general bounds-2 is the calculation of the upper bound. More specifically, given a subspace $\mathbb{B}$, $U_s(\mathbb{B})\leq U_q(\mathbb{B})$, which means general bounds-1 is tighter than general bounds-2 (Rigorous mathematical proof can be found in appendix~C). 
In the next sections, we introduce how to calculate the upper bound in detail.


\subsection{Parametrizing the Searching Domain }

Before estimating the bounds in BnB algorithm, we must first parametrize the searching space. In this section, we first recall the parametrization of $SO(3)$ in rotation search theory~\cite{hartley2009global,joo2019globally}, and introduce three different parametrizations of $\mathbb{S}^2_\ominus$. Furthermore, we analysis the similarities and differences of the parametrizations between $SO(3)$ and $\mathbb{S}^2_\ominus$.
\subsubsection{Parametrization of $SO(3)$}

It is well known that rotation space $SO(3)$ can be minimally parametrized with the angle-axis vector, whose norm is the angle of rotation and direction is the axis of the rotation. Therefore, the space of all 3D rotations can be represented by a solid ball of radius ${\pi}$ in $\mathbb{R}^3$~\cite{yang2016go}. Furthermore, the $\pi$-ball is usually relaxed to a 3D cube for ease of manipulation in the BnB algorithm. Thus Lemma~\ref{Lemma 1} and~\ref{Lemma 2} can be used to efficiently estimate the bounds of rotation search theory.

The Lemma~\ref{Lemma 2} may seem like one of the most fundamental parts in rotation search theory. Let us get down to the details of Lemma~\ref{Lemma 2}, and introduce the quaternion to build the connection with the parametrization of $\mathbb{S}^2_\ominus$. Geometrically, the mapping from  quaternions  ($ \mathbb{S}^3$) to rotations ($ SO(3)$) is a two-to-one mapping. We then denote a hyper-hemisphere as follows:
\begin{equation}
\mathbb{S}^3_\ominus=\lbrace{\bm{q}\in \mathbb{S}^3\vert q_1\geq0}\rbrace
\end{equation}
where $\bm{q}=[q_1,q_2,q_3,q_4]^T$ is a unit vector in $\mathbb{R}^4$.
Thus the ``upper'' hemisphere $\mathbb{S}^3_\ominus$ of the unit quaternion sphere is in one-to-one correspondence with the rotation $\pi$-ball, except at the boundary, where the correspondence is two-to-one~\cite{hartley2009global}. Therefore a conclusion follows as
\begin{lemma} \label{Lemma 3}
For $\forall\bm{q}_a,\bm{q}_b\in\mathbb{S}^3_\ominus$, then
\begin{equation}\label{Eq:Lemma 3}
d_\angle(R_a,R_b)=2\angle(\bm{q}_a,\bm{q}_b)
 \end{equation} 
\end{lemma}
\noindent where $R_a,R_b \in SO(3)$ are corresponding rotations of $\bm{q}_a$ and $\bm{q}_b$.
This lemma implies that the angle of two rotations is twice the angle between their corresponding quaternions. 

Additionally, a unit quaternion $\bm{q}\in \mathbb{S}^3_\ominus$ can be represented by an angle-axis vector $\bm{r}$ as follows:
\begin{equation}\label{Eq:3D-Exp-mapping}
\bm{q}^T=\Big[\cos(\frac{\Vert\bm{r}\Vert}{2}),\sin(\frac{\Vert\bm{r}\Vert}{2})\hat{\bm{r}}^T\Big]
\end{equation}
where $\hat{\bm{r}}=\bm{r}/ \Vert\bm{r}\Vert$ is a unit vector representing the axis of the rotation, and $\Vert\bm{r}\Vert$ is the angle of rotation. It is an exponential mapping from the upper quaternion hemisphere to the solid $\pi$-ball. Therefore, there is an important inequation as follows:
\begin{lemma}\label{Lemma 4}
For $\forall\bm{q}_a,\bm{q}_b\in\mathbb{S}^3_\ominus$, then
\begin{equation}\label{Eq:Lemma 4}
2\angle(\bm{q}_a,\bm{q}_b)\leq\Vert\bm{r}_a-\bm{r}_b\Vert
\end{equation}
\end{lemma}
\noindent where $\bm{r}_a$ and $\bm{r}_b$ are angle-axis representations of $\bm{q}_a$ and $\bm{q}_b$. The complete proofs of Lemma~\ref{Lemma 3} and Lemma~\ref{Lemma 4}  can be found in~\cite{hartley2009global} and~\cite{hartley2013rotation}. 

According to Eq.~(\ref{Eq:Lemma 3}) and~(\ref{Eq:Lemma 4}), we can easily obtain Lemma~\ref{Lemma 2}. In other words, Lemma~\ref{Lemma 2} is  separated into two parts, and Lemma~\ref{Lemma 4} inspires us to parametrize the $\mathbb{S}^2_\ominus$ by exponential mapping.

\subsubsection{Parametrization of $\mathbb{S}^2_\ominus$: Exponential Mapping}
Geometrically, $\mathbb{S}^2_\ominus$ is a hemisphere in three-dimensional Euclidean space, and it is inherently a two-dimensional closed space. In order to parametrize $\mathbb{S}^2_\ominus$ minimally, we are inspired from Lemma~\ref{Lemma 4} and propose an exponential mapping method to map the hemisphere to a 2D-disk. 

Concretely, let  $\bm{v}=[v_1,v_2,v_3]^T\in\mathbb{S}^2_\ominus $, then it can be represented by a corresponding point $\bm{d}\in\mathbb{R}^2$ in the disk,
\begin{equation}\label{Eq:2D-Exp-mapping}
\bm{v}^T=\Big[\sin(\theta)\hat{\bm{d}}^T,\cos(\theta)\Big]
\end{equation}
where $\theta\in[0,\pi/2]$, $\hat{\bm{d}}$ is a unit vector in $\mathbb{R}^2$ and $\bm{d}=\theta \hat{\bm{d}}$. Note that the domain of $\theta$ is corresponding to $v_3\geqslant 0$, and geometrically, $\theta$ is the radius of the disk. In BnB algorithm, a square (side=$\pi$) circumscribing the mapped disk area is used as the vertical direction domain for ease of manipulation. 

The mapping $\bm{v}\rightleftharpoons\bm{d}$ is similar to the mapping from $\mathbb{S}^3_\ominus $ to the 3D solid $\pi$-ball. Similarly, the mapping is one-to-one except the boundary (i.e., the equator) where it is two-to-one. More specifically, the exponential mapping is closely related to  Lie theory~\cite{abbaspour2007basic,sola2018micro}. However, in this paper we will not rely on any knowledge of the Lie groups theory without distracting readers' attention and focus on the direction estimation problem.

Because of the similarity between Eq.~(\ref{Eq:2D-Exp-mapping}) and Eq.~(\ref{Eq:3D-Exp-mapping}), we then propose a similar inequation as follows:

\begin{figure}
\centering
\includegraphics[width=0.99\linewidth]{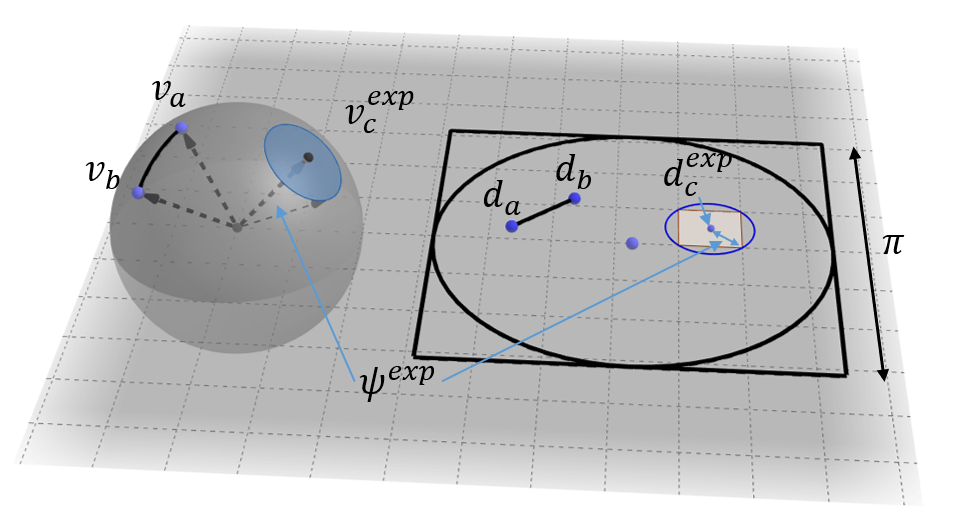}
\caption{Visualization of exponential mapping. Two points $\bm{v}_a,\bm{v}_b\in\mathbb{S}_{\ominus}^2  $ are corresponding to two points $\bm{d}_a,\bm{d}_b\in\mathbb{R}^2 $ and $\angle(\bm{v}_a,\bm{v}_b)\leq\Vert\bm{d}_a-\bm{d}_b\Vert$. A divided square-shaped branch, whose center is $\bm{d}_c^{exp}$, is relaxed into a circle in $\mathbb{R}^2$. Then the \textit{preimage} of the circle is relaxed into a spherical patch, whose center is $\bm{v}_c^{exp}$, in $\mathbb{S}^2$. $\psi^{exp} $ is the radius of the relaxed circle in 2D plane.}
\label{Fig:Visualization of Exp-Mapping}
\end{figure}

\begin{proposition}\label{2D-exp-inequation}
For $\forall\bm{v}_a,\bm{v}_b\in\mathbb{S}^2_\ominus$, then
\begin{equation}\label{Eq:2D-exp-inequation}
\angle(\bm{v}_a,\bm{v}_b)\leq\Vert\bm{d}_a-\bm{d}_b\Vert
\end{equation}
\end{proposition}
\noindent where $\bm{d}_a,\bm{d}_b$ are corresponding points of $\bm{v}_a,\bm{v}_b$ in the 2D disk.

\begin{proof}
The complete proof is in the appendix~D, and the visualization can be found in Fig.~\ref{Fig:Visualization of Exp-Mapping}.
\end{proof}

The exponential parametrization obtains great success in $SO(3)$ and builds the foundation of Lemma~\ref{Lemma 2}. In this paper, we extend exponential mapping to $\mathbb{S}^2_\ominus$  and apply Proposition~\ref{2D-exp-inequation} as one of the fundamental parts in our globally optimal vertical estimation method.

\subsubsection{Parametrization of $\mathbb{S}^2_\ominus$: Stereographic Projection}

\begin{figure}
\centering

\includegraphics[width=0.9\linewidth]{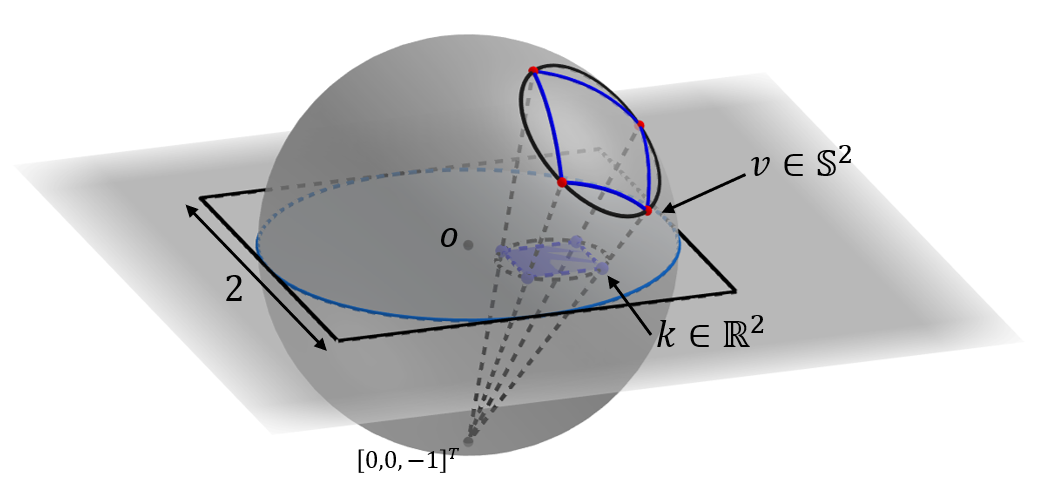}

\caption{Visualization of stereographic projection.  A point $\bm{v}\in\mathbb{S}_{\ominus}^2  $ is corresponding to a point $\bm{k}\in\mathbb{R}^2 $ and the divided square-shaped branch is relaxed to a circle, which is corresponding to an umbrella-shaped region in $\mathbb{S}_{\ominus}^2$.}
\label{Fig:Visualization of Stereo-Mapping}
\end{figure}

In geometry, the stereographic projection is a particular mapping that could project a hemisphere to a disk in plane, which  means we can also represent the $\mathbb{S}^2_\ominus$ minimally by applying stereographic projection.

The stereographic projection  is described in Fig.~\ref{Fig:Visualization of Stereo-Mapping}. We denote a point $\bm{k}=[k_1,k_2]^T\in\mathbb{R}^2$ in the equatorial plane and its corresponding point $\bm{v}=[v_1,v_2.v_3]^T\in \mathbb{S}^2_\ominus$, and if the projection pole is at $[0,0,-1]^T$ (South Pole)(see~\cite{stereo-formu}), then we have: 
\begin{align}
\bm{k}&=\Big[\frac{{v}_1}{1+{v}_3},\frac{{v}_2}{1+{v}_3}\Big]^T\\
\bm{v}&=\Big[\frac{2{k}_1}{1+{k}_1^2+{k}_2^2},\frac{2{k}_2}{1+{k}_1^2+{k}_2^2},\frac{1-{k}_1^2-{k}_2^2}{1+{k}_1^2+{k}_2^2}\Big]^T
\end{align}
Clearly, we can parametrize $\mathbb{S}^2_\ominus$ minimally by stereographic projection. 
Similarly, a square (side=$2$) circumscribing the mapped disk area is used as the vertical direction domain in the BnB algorithm. 
It worth noting that the stereographic projection was also applied to accelerate the calculation in rotation search~\cite{bustos2016fast}, which inspires our work.

\subsubsection{Parametrization of $\mathbb{S}^2_\ominus$: Spherical Coordinate System}

The $\mathbb{S}^2_\ominus$ can also be parameterized by spherical coordinate system (Wikipedia: ``spherical coordinate system"). Geometrically, the hemisphere is flattened to a rectangle (see Fig.~\ref{Fig:scs}). In the BnB algorithm, the rectangle region can be set as the initial searching domain.

For $\forall \bm{v}=[v_1,v_2.v_3]^T\in\mathbb{S}^2_{\ominus}$ and its corresponding point $\bm{h}=[h_1,h_2]$, the mapping from three-dimensional Cartesian coordinate system to spherical coordinate system is~\cite{cart2sph}
\begin{align}
h_1&= \arctan({v_2},{v_1})
\\
h_2 &= \arctan({v_3},{\sqrt{v_1^2 + v_2^2}})
\end{align}
where $\arctan(\cdot,\cdot)$ is the four-quadrant inverse tangent function\footnote{https://ww2.mathworks.cn/help/matlab/ref/atan2.html}. Conversely~\cite{sph2cart},
\begin{align}
v_1 &=  \cos(h_2) \cos(h_1)\\
v_2 &= \cos(h_2) \sin(h_1)\\
v_3 &=  \sin(h_2)
\end{align}
where $-\pi\leq h_1\leq \pi$ and $0\leq h_2\leq\pi/2$ are azimuth angle  and elevation angle, respectively.

\begin{figure}
	\centering
	\includegraphics[width=\linewidth]{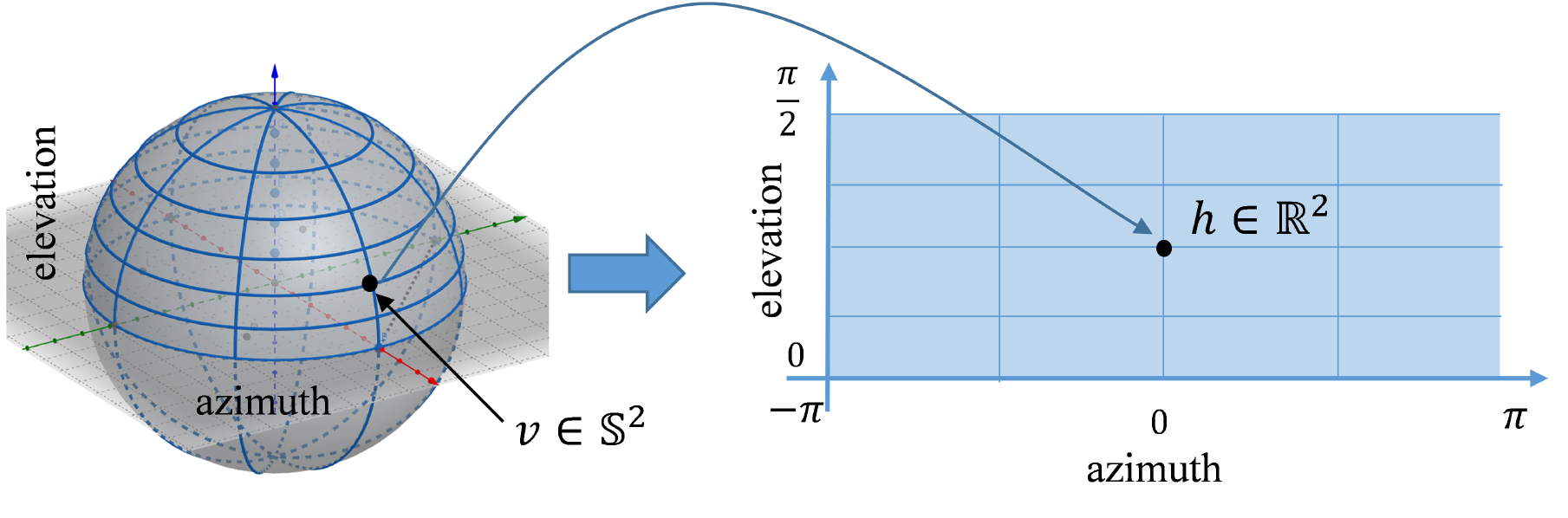}
	\caption{Visualization of spherical coordinate system. The hemisphere is flattened to a rectangle, which leads to significant distortion, especially near the Pole.}
	\label{Fig:scs}
\end{figure}

In summary, the space of $SO(3)$ is parametrized and relaxed to a 3D cube, thus in BnB algorithm, the cube is  recursively divided into eight sub-cubes. By contrast, the parametrizations of $\mathbb{S}^2_\ominus$ are both disks in exponential mapping and stereographic projection, after that the disk is relaxed to a solid square. In the BnB algorithm, we recursively subdivide it into four smaller squares and calculate the estimation of the upper bound and lower bound for the optimum in each sub-branch. Lastly, $\mathbb{S}^2_\ominus$  can be parameterized by a azimuth-elevation rectangle using spherical coordinate system. Similarly, in the BnB algorithm, the rectangle is recursively divided into four smaller rectangles.

%
%

For ease of understanding, we call the point in the solid disk/rectangle \textit{image-point}, meanwhile we call its corresponding point \textit{preimage-point}  in $\mathbb{S}^2$. 

\subsection{Estimating Bounds}
In this section, we show how to calculate the bounds with different  parametrizations in detail.

 \subsubsection{Bounds of  Rotation Search }
We first recall the bounds applied in rotation search. According to Lemma~\ref{Lemma 1} and Lemma~\ref{Lemma 2},
\begin{equation}
\angle(R_a\bm{x},R_b\bm{x})\leq d_\angle(R_a,R_b)=2\angle(\bm{q}_a,\bm{q}_b)\leq\Vert\bm{r}_a-\bm{r}_b\Vert
\end{equation}
\begin{equation}
\Rightarrow
\angle(R_a\bm{x},R_b\bm{x})\leq\Vert\bm{r}_a-\bm{r}_b\Vert
\end{equation}
Then, we have the following Lemma.

\begin{lemma}[rotation uncertainty angle bound]\label{Lemma:uncertainty angle bound}
Given a divided cube-shaped rotation branch $\mathbb{B}^{rot}$, whose center is $\bm{c}^{rot}$, half-side is $\sigma^{rot}$. For $\forall R\in \mathbb{B}^{rot},\forall\bm{v}_0\in\mathbb{S}^2 $,
\begin{equation}\label{Eq:uncertaintay angle bound}
\angle\big(R\bm{v}_0,R_c\bm{v}_0\big)\leq \sqrt{3}\sigma^{rot}\triangleq \psi^{rot}
\end{equation}
\end{lemma}
\noindent where $R_c$ is matrix representation of $\bm{c}^{rot}$. 
Let initial vertical direction $\bm{v}_0=[0,0,1]^T$, $R\bm{v}_0=\bm{v}^{rot}$ and $R_c\bm{v}_0\triangleq \bm{v}_c^{rot} $. Then,
\begin{equation}
\angle(R\bm{v}_0,R_c\bm{v}_0 ) \leq\psi^{rot}
\Rightarrow\angle(\bm{v}^{rot},\bm{v}_c^{rot})\leq\psi^{rot}
\end{equation}
\noindent Observe that it satisfies the conditions of Proposition~\ref{Pro:General bounds-2}: $\bm{v}^{rot}\rightleftharpoons\bm{v}$, $\bm{v}_c^{rot}\rightleftharpoons\bm{v}_c$, $\psi^{rot}\rightleftharpoons\psi$ and $\mathbb{B}\rightleftharpoons \{R\bm{v}_0|R\in\mathbb{B}^{rot}\}$.

Then, given a divided cube-shaped rotation branch $\mathbb{B}^{rot}$, the bounds  can be 
\begin{subequations}
\begin{equation}
U^{rot}(\mathbb{B}^{rot})=\sum_{j=1}^N\mathbbm{1}\big(\overline{Q}_j^{\parallel{rot}}\vee\overline{Q}_j^{\perp{rot}}\big)
\end{equation}
\begin{align}
\overline{Q}_j^{\parallel{rot}}&=\mathbbm{1}\big(\vert\bm{n}_j^T {\bm{v}^{rot}_c} \vert\geq\cos(\lfloor\tau+\psi^{rot}\rfloor)\big)
\\
\overline{Q}_j^{\perp{rot}}&=\mathbbm{1}\big(\vert\bm{n}_j^T \bm{v}_c^{rot}\vert\leq\sin(\lfloor\tau+\psi^{rot}\rfloor)\big)
\end{align}
\end{subequations}
\begin{subequations}
\begin{equation}
L^{rot}(\mathbb{B}^{rot})=\sum_{j=1}^N\mathbbm{1}\big((\underline{Q}_j^{\parallel{rot}}\vee\underline{Q}_j^{\perp{rot}}\big)
\end{equation}
\begin{align}
\underline{Q}_j^{\parallel{rot}}&=\mathbbm{1}\big(\vert\bm{n}_j^T {\bm{v}^{rot}_c} \vert\geq\cos(\tau)\big)
\\
\underline{Q}_j^{\perp{rot}}&=\mathbbm{1}\big(\vert\bm{n}_j^T \bm{v}_c^{rot}\vert\leq\sin(\tau)\big)
\end{align}
\end{subequations}



Note that the bounds are widely used in many geometrical vision problems~\cite{joo2019globally,joo2019robust}, which are not our original contributions. Besides, it is worth noting that there seems a tighter bound than Eq.~(\ref{Eq:uncertaintay angle bound}) in \cite{campbell2018globally}, however, to calculate the bound efficiently, it is based on  two unproven assumptions.
\subsubsection{Bounds Using Exponential Mapping} 
According to Proposition~\ref{2D-exp-inequation}, we have,
\begin{proposition}\label{Pro:2-dim uncertainty bound}
Given a divided square-shaped branch $\mathbb{B}^{exp}$ in exponential mapping plane, whose center is $\bm{d}_c^{exp}$, and half-side is $\sigma^{exp}$. For $\forall \bm{d}\in \mathbb{B}^{exp}$,
\begin{equation}
\angle(\bm{v},\bm{v}_c^{exp})\leq\sqrt{2}\sigma^{exp}\triangleq \psi^{exp}
\end{equation}
\end{proposition}
\noindent where $\bm{v}_c^{exp},\bm{v}\in\mathbb{S}^2$ are \textit{preimage points} of $\bm{d}_c^{exp}$ and $\bm{d}$.
\begin{proof}
This proposition can be derived as follows:
\begin{align}
\angle(\bm{v},\bm{v}_c^{exp})&\leq \Vert\bm{d}-\bm{d}_c^{exp}\Vert\leq \sqrt{2}\sigma^{exp} 
\end{align}
which follows Proposition~\ref{2D-exp-inequation} (see Fig.~\ref{Fig:Visualization of Exp-Mapping}).
\end{proof}


Proposition~\ref{Pro:2-dim uncertainty bound} and Lemma~\ref{Lemma:uncertainty angle bound} have similar formulations. However, to the best of our knowledge, it is the first time  Proposition~\ref{Pro:2-dim uncertainty bound} has been explicitly introduced  to the computer vision field. Obviously, given divided square-shaped branch $\mathbb{B}^{exp}$ in exponential mapping plane, according to Proposition~\ref{Pro:General bounds-2}, the bounds can be: 

\begin{subequations}
	\begin{equation}
	U_q^{exp}(\mathbb{B}^{ exp})=\sum_{j=1}^N\mathbbm{1}\big(\overline{Q}^{\parallel exp}_j\vee \overline{Q}^{\perp exp}_j\big)
	\end{equation}
	\begin{align}
	\overline{Q}^{\parallel exp}_j&=\mathbbm{1}\big(\vert\bm{n}_j^T {\bm{v}^{exp}_c} \vert\geq\cos(\lfloor\tau+\psi^{exp}\rfloor)\big)
	\\
	\overline{Q}^{\perp exp}_j&=\mathbbm{1}\big(\vert\bm{n}_j^T \bm{v}_c^{exp}\vert\leq\sin(\lfloor\tau+\psi^{exp}\rfloor)\big)
	\end{align}
\end{subequations}
\begin{subequations}
	\begin{equation}
	L_q^{exp}(\mathbb{B}^{exp})=\sum_{j=1}^N\mathbbm{1}\big(\underline{Q}_j^{\parallel exp}\vee\underline{Q}_j^{\perp exp} \big)
	\end{equation}
	\begin{align}
		\underline{Q}_j^{\parallel exp}&=\mathbbm{1}\big(\vert\bm{n}_j^T {\bm{v}^{exp}_c} \vert\geq\cos(\tau)\big)
		\\
		\underline{Q}_j^{\perp exp}&=\mathbbm{1}\big(\vert\bm{n}_j^T \bm{v}_c^{exp}\vert\leq\sin(\tau)\big)
	\end{align}
\end{subequations}
\subsubsection{Bounds Using Stereographic Projection}

Stereographic projection has a crucial property that circles  are projected as circles (circle preserving~\cite{bustos2016fast,needham1998visual}). We use this   property to calculate the first bound based on stereographic projection.

\begin{proposition}\label{Pro:circle_bound}
Given a divided  square-shaped branch $\mathbb{B}^{ste}$ in stereographic projection plane, and its circumscribed circle is $\mathbb{C}^{ste}_{2D}$.  The preimage of $\mathbb{C}^{ste}_{2D}$  is $\mathbb{C}^{ste}$ in $\mathbb{S}^2$, whose radius is $\sigma^{ste}$ and the direction of its center point is $\bm{v}_c^{ste}$; $\forall\bm{k}\in\mathbb{B}^{ste}$, $\bm{v}$ is its \textit{preimage-point},
\begin{equation}
\angle(\bm{v},\bm{v}_c^{ste})\leq\arcsin (\sigma^{ste})\triangleq\psi^{ste}
\end{equation}
\end{proposition}

\begin{proof}
Because $\bm{k}\in \mathbb{B}^{ste}\subset\mathbb{C}^{ste}_{2D}$, then its \textit{preimage-point} $\bm{v}\in\mathbb{C}^{ste} $. The angle of $\bm{v}$ and $\bm{v}_c^{ste}$ must be no greater than the  maximum angle $\psi^{ste}$ (see Fig.~\ref{Fig:stereo-mapping-detail}).
\end{proof}

We then explain how to calculate $\psi^{ste}$ and $\bm{v}_c^{ste}$ in detail. Given a divided  square-shaped branch $\mathbb{B}^{ste}$ in mapped plane,  its four vertexes ($\bm{k}_1,\bm{k}_2,\bm{k}_3,\bm{k}_4$) must be in the edge of circumscribed circle (see Fig.~\ref{Fig:stereo-mapping-detail}). Then the\textit{ preimage-points} of the vertexes ($\bm{v}_1,\bm{v}_2,\bm{v}_3,\bm{v}_4$) must be in the edge of $\mathbb{C}^{ste}$, and the edge of $\mathbb{C}^{ste}$ is a circle. The direction of the center point $\bm{v}_c^{ste}$ is perpendicular to the plane crossing the circle. Hence, $\bm{v}_c^{ste}$ is perpendicular to any vector in the circle-plane, which means $\bm{v}_c^{ste}\perp(\bm{v}_1-\bm{v}_2)$ and $\bm{v}_c^{ste}\perp(\bm{v}_1-\bm{v}_3)$. Let $\bm{v}_{cross}\triangleq$ $(\bm{v}_1-\bm{v}_2)\times(\bm{v}_1-\bm{v}_3)$. Then, $\bm{v}_c^{ste}=\bm{v}_{cross}/\Vert\bm{v}_{cross}\Vert$ and $\psi^{ste}=\angle(\bm{v}_c^{ste},\bm{v}_1)=\arcsin (\sigma^{ste})$.

\begin{figure}
	\includegraphics[scale=0.5]{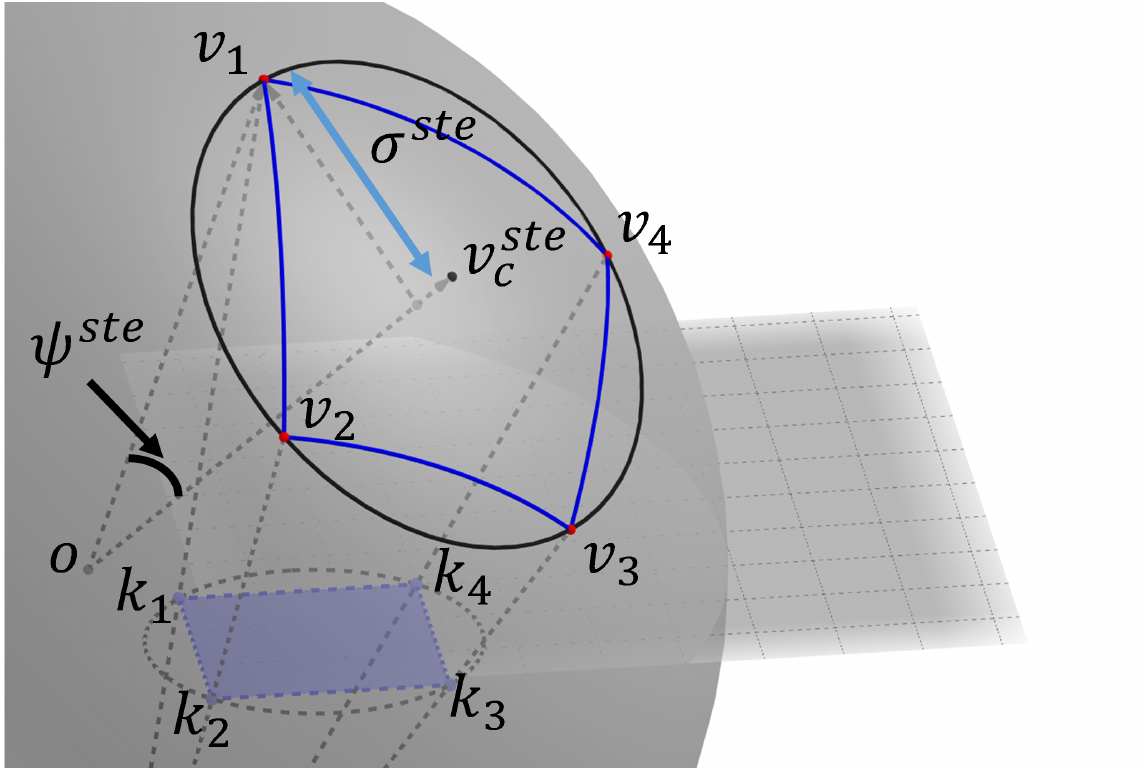}
	\caption{ The geometry of a divided square in stereographic projection plane. A divided branch ($\bm{k}_1,\bm{k}_2,\bm{k}_3,\bm{k}_4$) is projected to a domain ($\bm{v}_1,\bm{v}_2,\bm{v}_3,\bm{v}_4$) in $\mathbb{S}^2$. The radius of its circumscribed circle is $\sigma^{ste}$, and the direction of the center point is $\bm{v}_c^{ste}$, then $ \psi^{ste}=\angle(\bm{v}_1,\bm{v}^{ste}_c)=\arcsin(\sigma^{ste})$.}
	\label{Fig:stereo-mapping-detail}
\end{figure}

Intuitively, Proposition~\ref{Pro:circle_bound} shows that a divided square-shaped branch in the stereographic projection plane is relaxed to a circle, meanwhile the corresponding domain in the 3D sphere is also relaxed to a umbrella-shaped patch surrounded by a circle, whose radius is $\sigma^{ste}$.

Given a divided square-shaped branch $\mathbb{B}^{ste}$ in stereographic projection plane, according to Proposition~\ref{Pro:General bounds-2}, the bounds  can be

\begin{subequations}\label{Eq:circ-upper-bounds}
	\begin{equation}
	U_q^{ste}(\mathbb{B}^{ste})=\sum_{j=1}^N\mathbbm{1}\big(\overline{Q}_j^{\parallel ste}\vee\overline{Q}_j^{\perp ste} \big)
	\end{equation}
	\begin{align}
	\overline{Q}_j^{\parallel ste}&=\mathbbm{1}\big(\vert\bm{n}_j^T {\bm{v}^{ste}_c} \vert\geq\cos(\lfloor\tau+\psi^{ste}\rfloor)\big)
	\\
	\overline{Q}_j^{\perp ste}&=\mathbbm{1}\big(\vert\bm{n}_j^T \bm{v}_c^{ste}\vert\leq\sin(\lfloor\tau+\psi^{ste}\rfloor)\big)
	\end{align}
\end{subequations}

\begin{subequations}\label{Eq:circ-lower-bounds}
\begin{equation}
	L_q^{ste}(\mathbb{B}^{ste})=\sum_{j=1}^N\mathbbm{1}\big(\underline{Q}_j^{\parallel ste}\vee\underline{Q}_j^{\perp ste} \big)
\end{equation}

	\begin{align}
	\underline{Q}_j^{\parallel ste}&=\mathbbm{1}\big(\vert\bm{n}_j^T {\bm{v}^{ste}_c} \vert\geq\cos(\tau)\big)
	\\
	\underline{Q}_j^{\perp ste}&=\mathbbm{1}\big(\vert\bm{n}_j^T \bm{v}_c^{ste}\vert\leq\sin(\tau)\big)
	\end{align}
\end{subequations}


The bounds (Eq.~\eqref{Eq:circ-upper-bounds} and Eq.~\eqref{Eq:circ-lower-bounds}) are called circle-bounds using stereographic projection as the divided square is relaxed  to its circumscribed  circle.

\subsubsection{Tighter Bounds Using Stereographic Projection} 

For the stereographic projection, a tighter bound can be found without relaxing the divided square, and therefore, it does not apply the circle-preserving property.

Given a divided square-shaped branch $\mathbb{B}^{ste}$ in stereographic projection plane, the \textit{preimage} of its  center is $\bm{v}_t^{ste}$. $\bm{v}\in\mathbb{S}^2$ is the \textit{preimage-point} of $\bm{k}\in\mathbb{B}^{ste}$. 
\begin{align}
\underline{\phi}^{ste}_{j}&\triangleq\min\angle(\bm{v},\bm{n}_j)\\
\overline{\phi}^{ste}_{j}&\triangleq\max\angle(\bm{v},\bm{n}_j)
\end{align}

\noindent Considering the Proposition~\ref{Pro:General bounds-1}, the bounds can be 
\begin{subequations}
	\begin{equation}
	U_s^{ste}(\mathbb{B}^{ste})=\sum_{j=1}^N \mathbbm{1}\Big(\overline{\mathcal{S}}_j^{\parallel^+ ste}\vee \overline{\mathcal{S}}_j^{\parallel^-ste}\vee \overline{\mathcal{S}}_j^{\perp ste}\Big)
	\end{equation}
	\begin{align}
	\overline{\mathcal{S}}_j^{\parallel^+ ste}&=\mathbbm{1}\Big(\underline{\phi}_j^{ste}\leq\tau\Big)
	\\
	\overline{\mathcal{S}}_j^{\parallel^-ste}&=\mathbbm{1}\Big(\overline{\phi}_j^{ste}\geq\pi-\tau\Big)
	\\
	\overline{\mathcal{S}}_j^{\perp ste}&=\mathbbm{1} \Big(\underline{\phi}_j^{ste}-\tau\leq\frac{\pi}{2}\leq\overline{\phi}_j^{ste}+\tau\Big)
	\end{align}
\end{subequations}

\begin{subequations}
	\begin{equation}
	L_s^{ste}(\mathbb{B}^{ste})=\sum_{j=1}^N \mathbbm{1}\Big(\underline{\mathcal{S}}_j^{\parallel^+ ste}\vee \underline{\mathcal{S}}_j^{\parallel^- ste} \vee \underline{\mathcal{S}}_j^{\perp ste}\Big)
	\end{equation}
	\begin{align}
	\underline{\mathcal{S}}_j^{\parallel^+ ste}&=\mathbbm{1}\Big({\phi}_j^{ste}\leq\tau\Big)
	\\
	\underline{\mathcal{S}}_j^{\parallel^-ste}&=\mathbbm{1}\Big({\phi}_j^{ste}\geq\pi-\tau\Big)
	\\
	\underline{\mathcal{S}}_j^{\perp ste}&=\mathbbm{1} \Big(|{\phi}_j^{ste}-\frac{\pi}{2}|\leq\tau\Big)
	\end{align}
\end{subequations}

\noindent where $\phi_j^{ste}=\angle(\bm{v}_t^{ste},\bm{n}_j)$.

The detailed implementation for calculating the $\underline{\phi}_j^{ste}$ and $\overline{\phi}_j^{ste}$ can be found in appendix~E. Note that the  bounds are tighter than circle-bounds using stereographic projection. The reason is simple that a divided square is relaxed to a circle in the circle-bounds but no relaxation in the tighter bounds.  Additionally, because the bounds are based on the divided square, then we call the tighter bounds square-bounds using stereographic projection.  

\subsubsection{Bounds Using Sphere Coordinate System}
In this part, we introduce the upper and lower bounds using sphere coordinate system  according to Proposition~\ref{Pro:General bounds-1}. 

Given a divided rectangle-shaped branch $\mathbb{B}^{scs}$ in azimuth-elevation rectangle, the \textit{preimage} of its center is $\bm{v}_c^{scs}$. $\bm{v}\in\mathbb{S}^2$ is the \textit{preimage-point} of $\bm{h}\in\mathbb{B}^{scs}$. 
\begin{align}
\underline{\phi}^{scs}_{j}&\triangleq\min\angle(\bm{v},\bm{n}_j)\\
\overline{\phi}^{scs}_{j}&\triangleq\max\angle(\bm{v},\bm{n}_j)
\end{align}


\noindent Then, the bounds  can be 
\begin{subequations}
	\begin{equation}
	U_s^{scs}(\mathbb{B}^{scs})=\sum_{j=1}^N \mathbbm{1}\Big(\overline{\mathcal{S}}^{\parallel^+ scs}_j \vee\overline{\mathcal{S}}^{\parallel^- scs}_j \vee \overline{\mathcal{S}}^{\perp scs}_j \Big)
	\end{equation}
	\begin{align}
	\overline{\mathcal{S}}^{\parallel^+ scs}_j&=\mathbbm{1}\Big(\underline{\phi}_j^{scs}\leq\tau\Big)
	\\
	\overline{\mathcal{S}}^{\parallel^- scs}_j&=\mathbbm{1}\Big(\overline{\phi}_j^{scs}\geq\pi-\tau\Big)
	\\
	\overline{\mathcal{S}}^{\perp scs}_j &=\mathbbm{1} \Big(\underline{\phi}_j^{scs}-\tau\leq\frac{\pi}{2}\leq\overline{\phi}_j^{scs}+\tau\Big)
	\end{align}
\end{subequations}

\begin{subequations}
	\begin{equation}
	L_s^{scs}(\mathbb{B}^{scs})=\sum_{j=1}^N \mathbbm{1}\Big(\underline{\mathcal{S}}_j^{\parallel^+ scs}\vee \underline{\mathcal{S}}_j^{\parallel^- scs} \vee \underline{\mathcal{S}}_j^{\perp scs}\Big)
	\end{equation}
	\begin{align}
	\underline{\mathcal{S}}_j^{\parallel^+ scs}&=\mathbbm{1}\Big({\phi}_j^{scs}\leq\tau\Big)
	\\
	\underline{\mathcal{S}}_j^{\parallel^-scs}&=\mathbbm{1}\Big({\phi}_j^{scs}\geq\pi-\tau\Big)
	\\
	\underline{\mathcal{S}}_j^{\perp scs}&=\mathbbm{1} \Big(|{\phi}_j^{scs}-\frac{\pi}{2}|\leq\tau\Big)
	\end{align}
\end{subequations}

\noindent where $\phi_j^{scs}=\angle(\bm{v}_c^{scs},\bm{n}_j)$.

The implementation for calculating $\underline{\phi}_j^{scs}$ and $\overline{\phi}_j^{scs}$ can be found in appendix~F. 

\subsubsection{Comparison of the Bounds }
It is well known that the success of a BnB algorithm is mainly predicated on the quality of its bounds. To show the relaxation and the tightness, in this section, we compare these bounds (Table~\ref{table-1}) geometrically.


\textbf{Bounds of rotation search}. The searching domain is parametrized as a 3D cube. In BnB, for each divided sub-cube, it is first relaxed to its circumscribed ball and then relaxed to a region in quaternion sphere (Lemma~\ref{Lemma 2}). Lastly, it is relaxed to a spherical patch in $\mathbb{S}^2$ using Lemma~\ref{Lemma 1}.

\textbf{Bounds using exponential mapping(exp bounds)}. Similarly, the searching domain is parametrized as a 2D square. The divided sub-square is first relaxed to its circumscribed circle and then relaxed to a spherical patch in $\mathbb{S}^2$ (Proposition~\ref{Pro:2-dim uncertainty bound}). Therefore, it has a two-step geometrical relaxation.

\textbf{Circle-bounds using stereographic projection (ste-cirlce bounds)}. The searching domain is parametrized as a 2D square. The divided sub-square is first relaxed to its circumscribed circle, which is corresponding to a spherical patch in $\mathbb{S}^2$ (circle preserving). In geometric, it has only one relaxation processing.

\textbf{Square-bounds using stereographic projection (ste-square bounds)}.
The searching domain is the same as that of the ste-circle bounds, however, the ste-square bounds have no geometrical relaxations.

\textbf{Bounds using sphere coordinate system (SCS bounds)}.
The searching domain is parametrized as a 2D azimuth-elevation rectangle, which leads to significant distortions. Nonetheless, they have no geometrical relaxations.

Note that what we say about geometrical relaxation is only for one specific input. There is a relaxation for the objective, which relaxes the connections among the inputs. In other words, for a large branch, it hardly obtains the upper bound simultaneously for all inputs.

\textbf{Computational efficiency}. The exp-bounds and the ste-circle bounds are calculated more efficiently than the ste-square bounds and the scs-bounds. This is because that to estimate $\overline{\phi}_j$ and $\underline{\phi}_j$, it is needed to calculating the angle range between the $\bm{n}_j$ and four edges of the branch in the ste-square bounds and the scs-bounds. However, given a branch $\mathbb{B}$, all $\{\bm{n}_j\}_{j=1}^N$ share the same $\psi$ in the exp-bounds and the ste-circle bounds.

\begin{table}
	\caption{Different settings for different bounds in Algorithm~\ref{algorithm}. }
	\centering \label{table-1}
	\begin{tabular}{lccl}
		\toprule 
		{Methods}& upper &lower &searching domain\\
		\midrule 
		RS&$U^{rot}$ &$L^{rot}$ & 3D cube (side=$2\pi$)\\
		
		Exp-BnB& $U_q^{exp}$&$L_q^{exp}$& 2D square (side=$\pi$)\\
		
		Ste-circle-BnB &$U_q^{ste}$ &$L_q^{ste}$&2D square (side=$2$)\\
		
		Ste-square-BnB&$U_s^{ste}$&$L_s^{ste}$&2D square (side=$2$) \\
		SCS-BnB&$U_s^{scs}$ & $L_s^{scs}$ &2D rectangle($2\pi\times \pi/2 $)
		\\
		\bottomrule
		
	\end{tabular}
\end{table}

\section{Experiments}
In this section, we verify the validity of the proposed method on  challenging synthetic and real-world data. Firstly,  we compared our proposed methods with RANSAC and rotation search method to show robustness and efficiency. Then, full Atlanta frame estimation experiments were conducted to verify that estimating vertical direction was helpful to estimating all Atlanta frames. Lastly, we tested proposed methods in two real-world datasets to verify the practicality. All methods were implemented\footnote{{https://github.com/Liu-Yinlong/Globally-optimal-vertical-direction-estimation-in-Atlanta-world}} in Matlab 2019a and executed on an AMD Ryzen 7 2700X 3.7GHz CPU.

\subsection{Experimental Setting}
 The settings of approaches/pipelines run on experiments were as follows:
\begin{itemize}
\item \textbf{RANSAC}: The number of minimal sample subsets was 2. It could get three directions from two inlier-inputs (two inlier directions and its cross product direction), and one of them might be the vertical direction. Besides, the confidence level  $\zeta=0.99$ was used for the stopping criterion\cite{fischler1981random}. The number of iterations was typically taken as
\begin{equation}
\Omega=\Big\lceil\frac{\log(1-\zeta)}{\log(1-(1-\rho)^2)}\Big\rceil\label{Eq:omega}
\end{equation}
where $\rho$ was the outlier proportion,  $\lceil\cdot\rceil$ returned the nearest integer greater than or equal to the input.

\item \textbf{RS}: 
Algorithm~\ref{algorithm} with the rotation search bounds. Note that the bounds were also used in meta-BnB in~\cite{joo2019globally}.
We did not use the Extended Gaussian Image (EGI) and its integral image~\cite{joo2019globally,joo2019robust}, because we focused on the geometry and the validity of the proposed bounds. There might be more efficient bounds calculation methods for the proposed bounds but it is out of the scope of this paper.
\item \textbf{Exp-BnB}: Algorithm~\ref{algorithm} with the proposed bounds using exponential mapping.

\item \textbf{Ste-circle-BnB}: Algorithm~\ref{algorithm} with the proposed circle-bounds using stereographic projection.

\item \textbf{Ste-square-BnB}: Algorithm~\ref{algorithm} with the proposed square-bounds using stereographic projection.

\item \textbf{SCS-BnB}: Algorithm~\ref{algorithm} with the proposed bounds using sphere coordinate system.

\end{itemize}

\begin{figure*}
	\centering
	\includegraphics[scale=0.35]{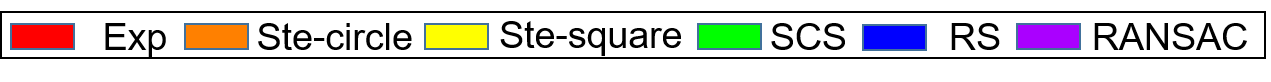}
	\begin{tabular}{ccc}
		$\kappa=0.005$&$\kappa=0.010$&$\kappa=0.020$
		\\

		\includegraphics[width=0.3\textwidth]{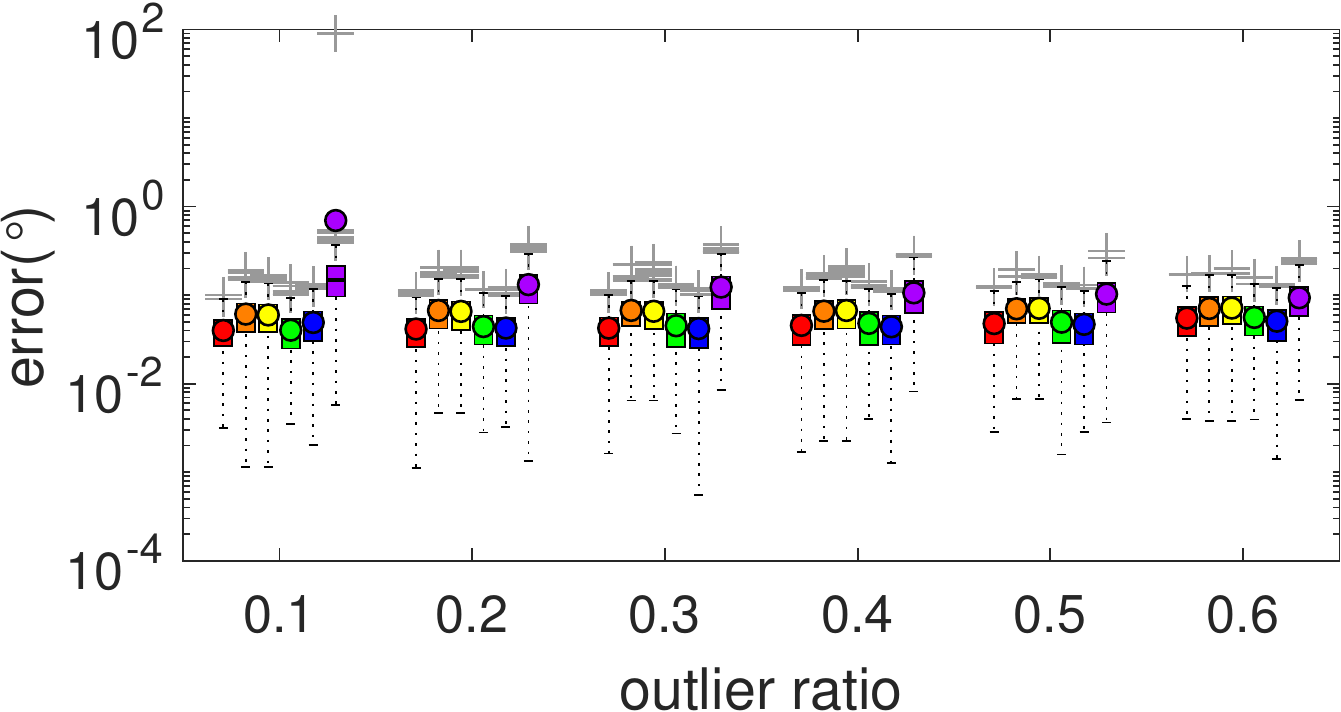}
		&\includegraphics[width=0.3\textwidth]{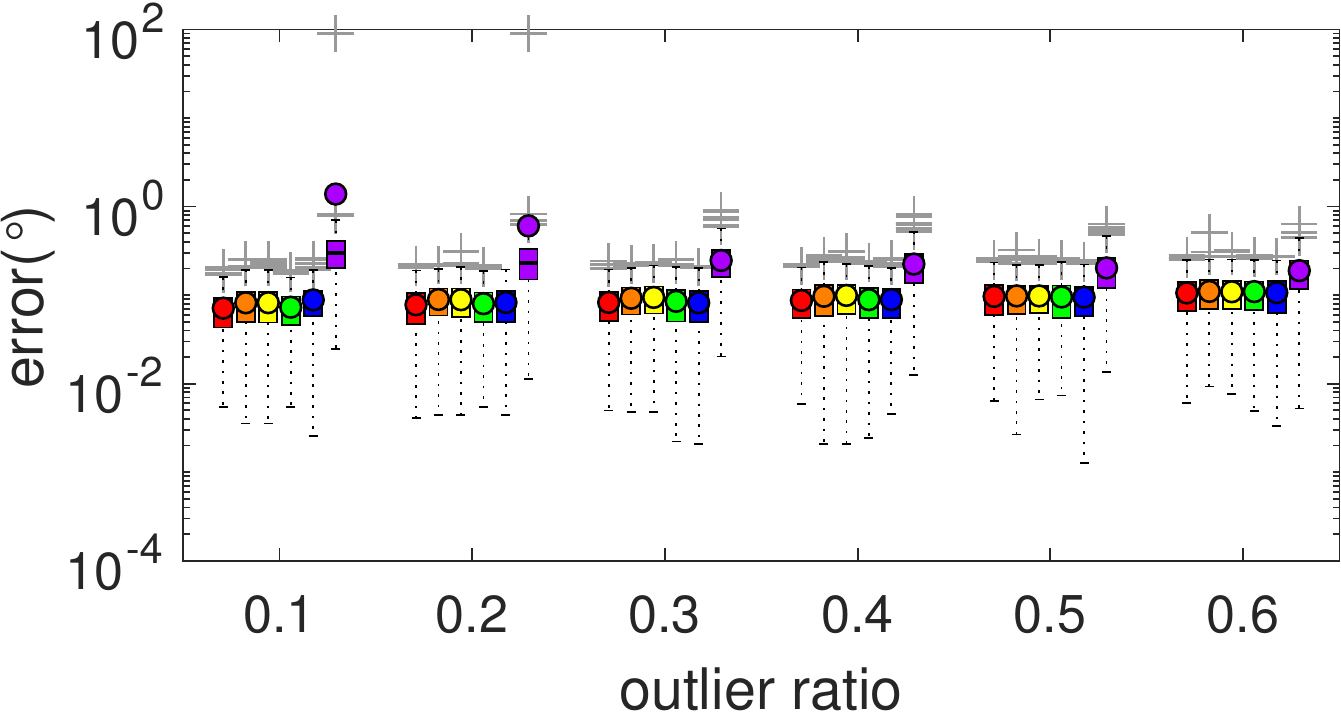}
		&\includegraphics[width=0.3\textwidth]{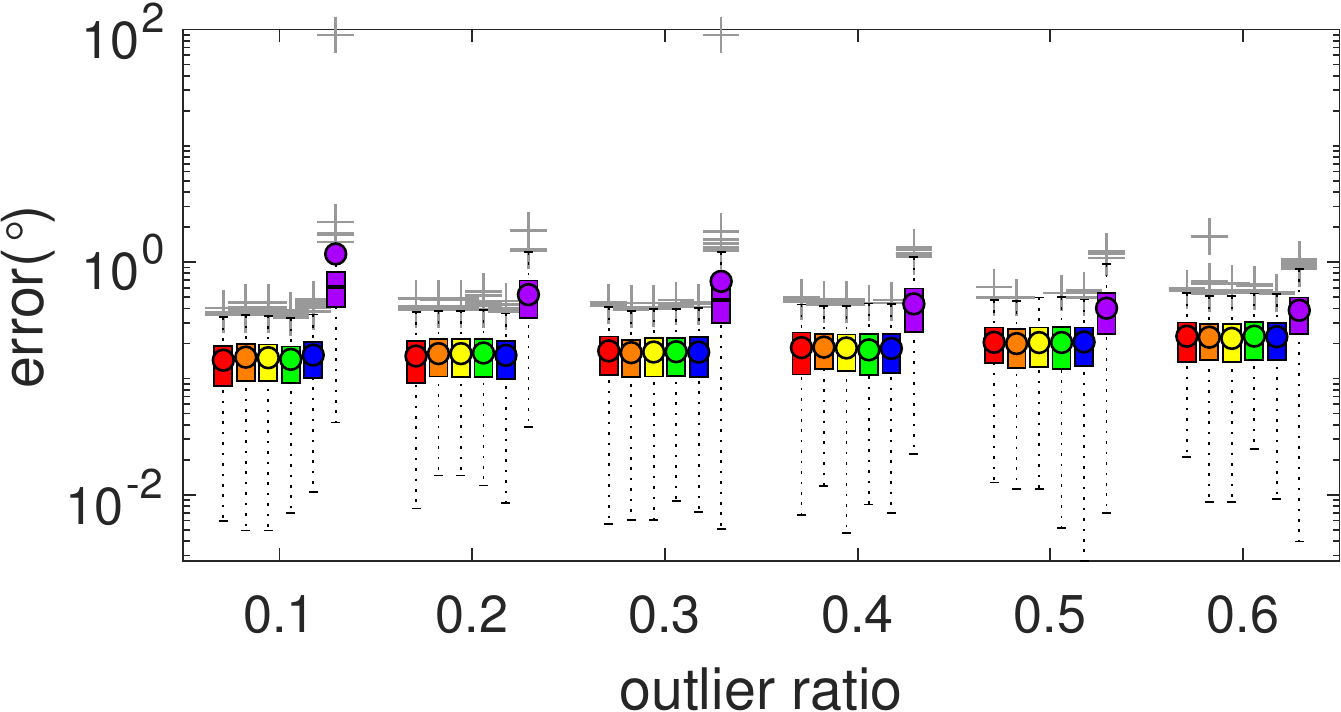}
		\\
		\includegraphics[width=0.3\textwidth]{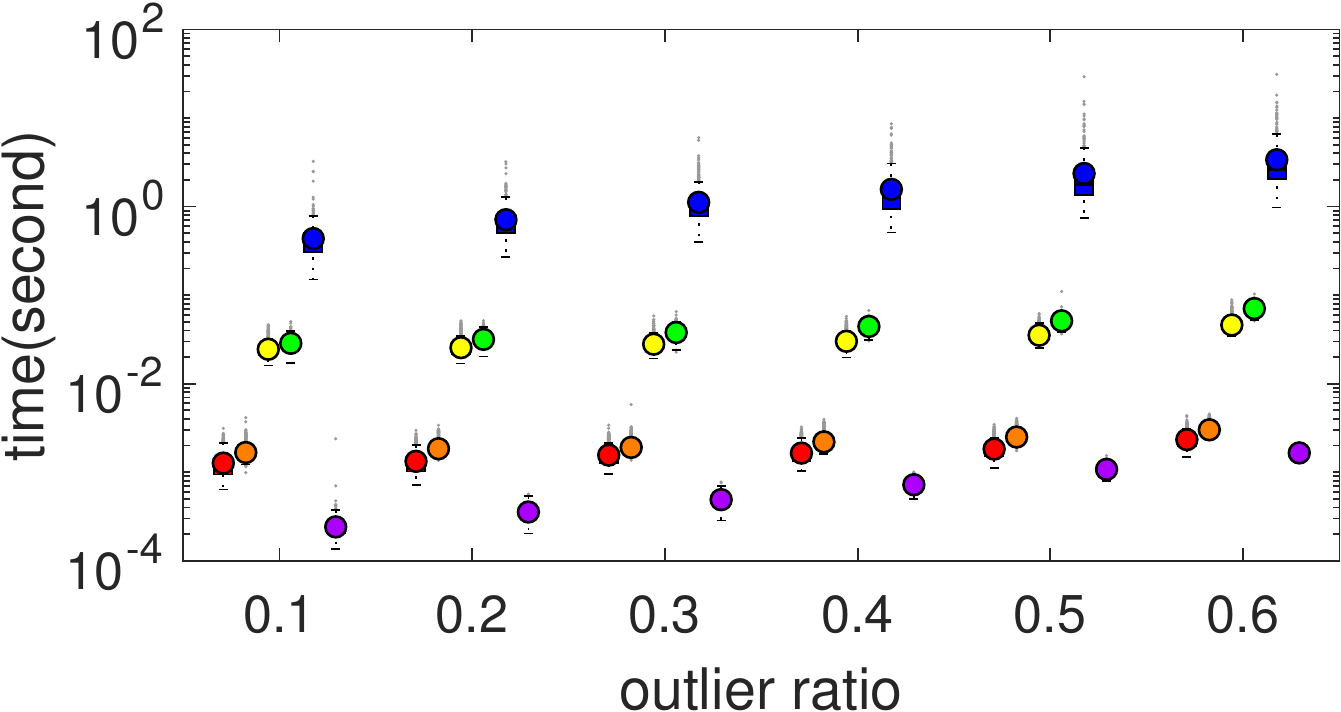}
		&\includegraphics[width=0.3\textwidth]{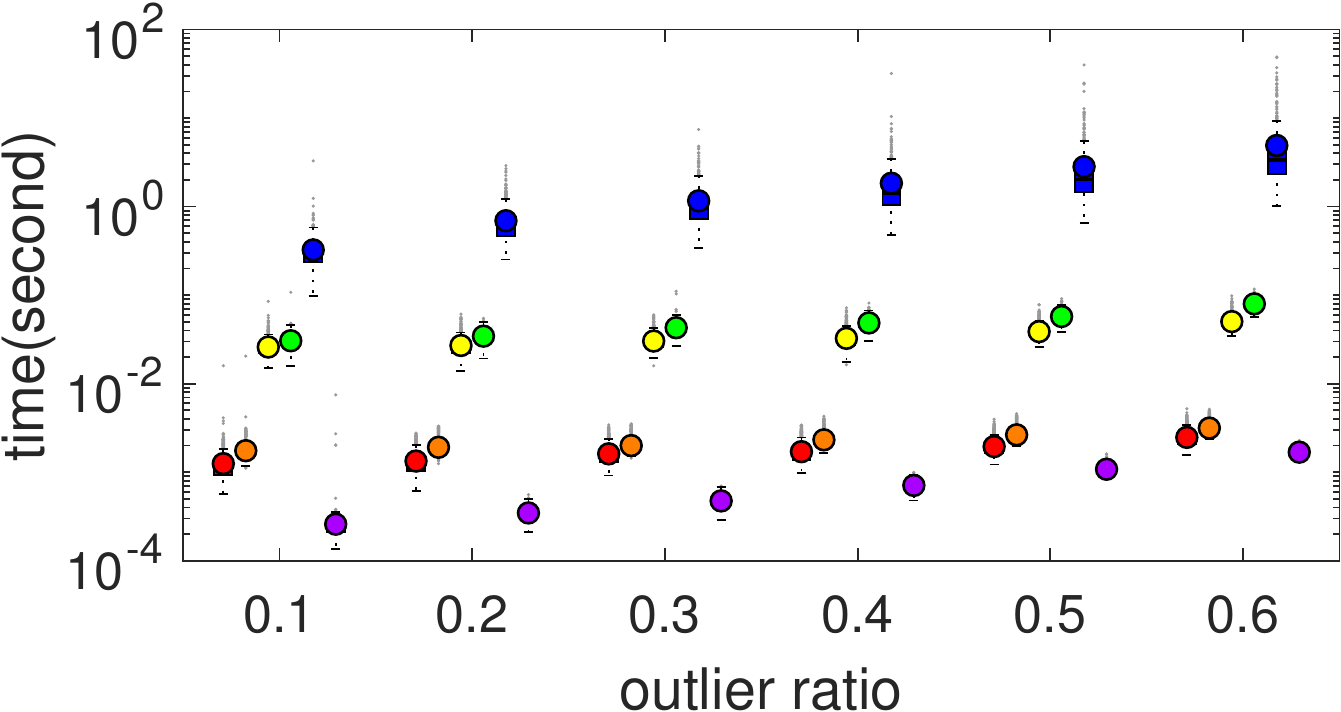}
		& \includegraphics[width=0.3\textwidth]{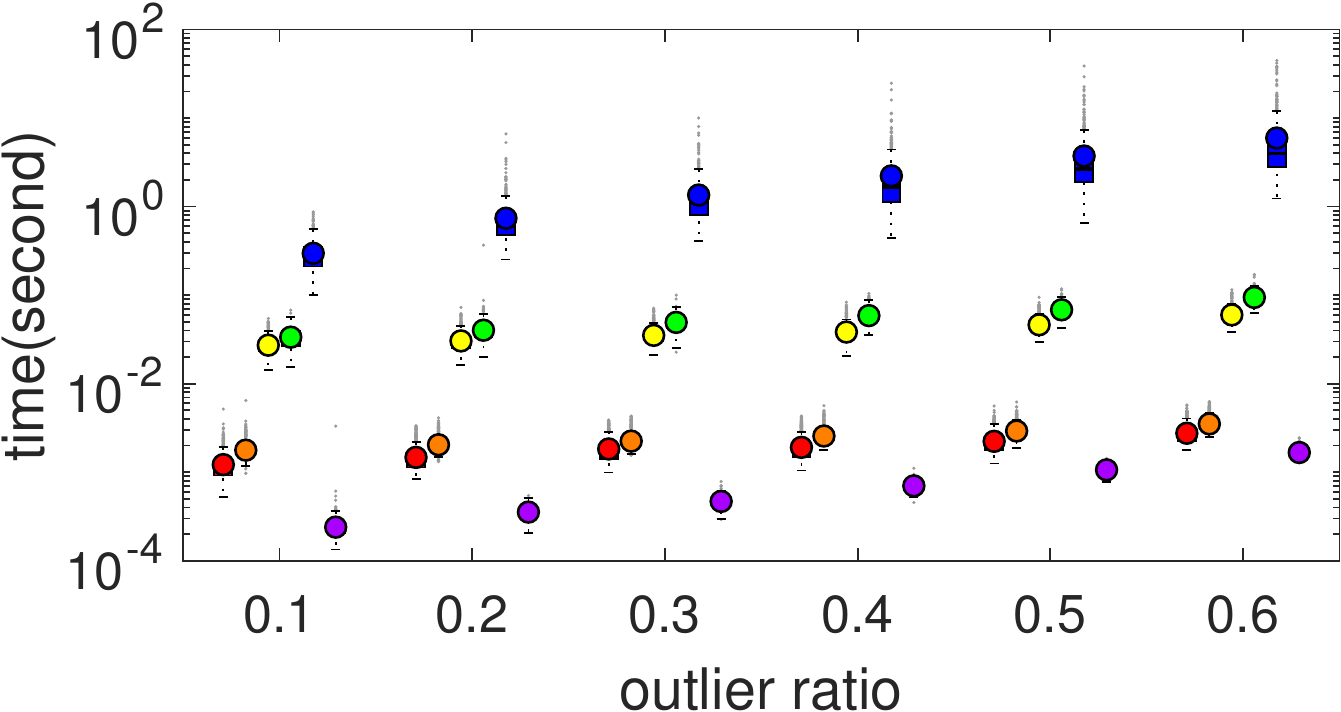}
		\\
		\includegraphics[width=0.3\textwidth]{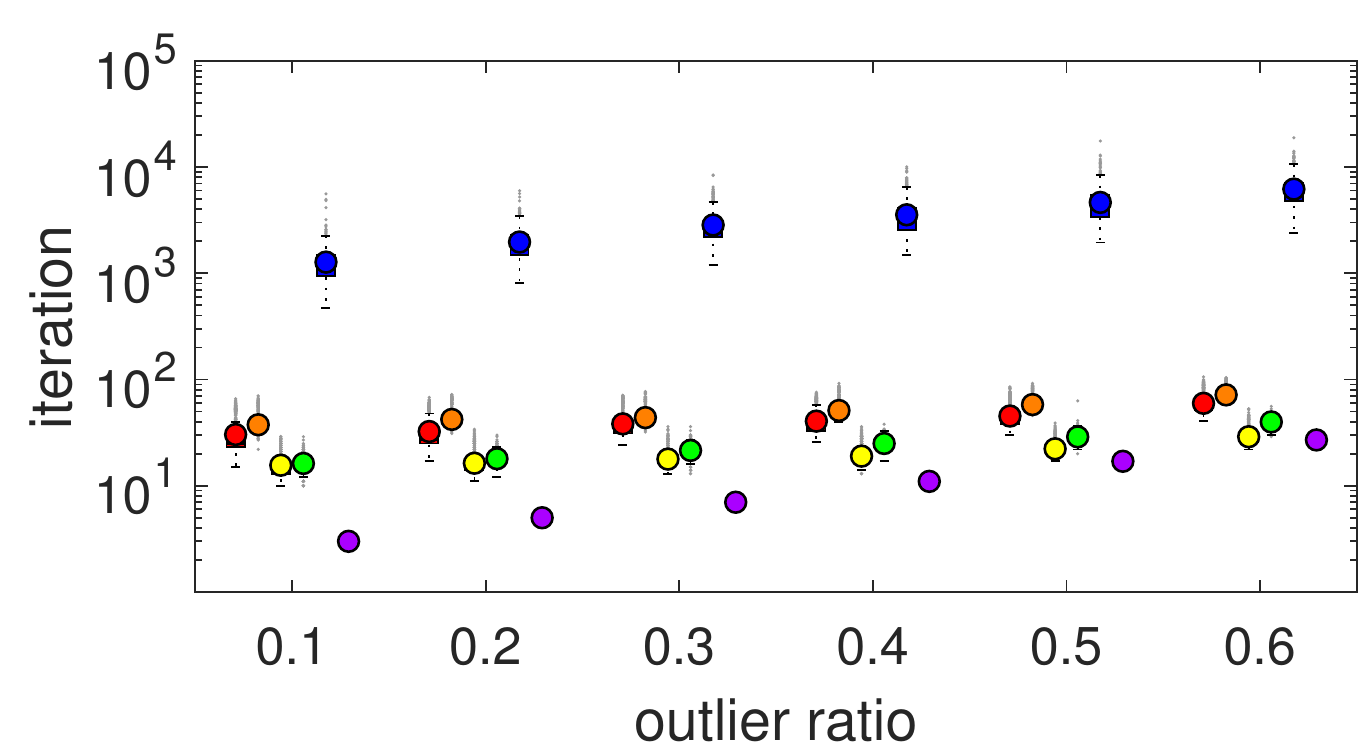}
		& \includegraphics[width=0.3\textwidth]{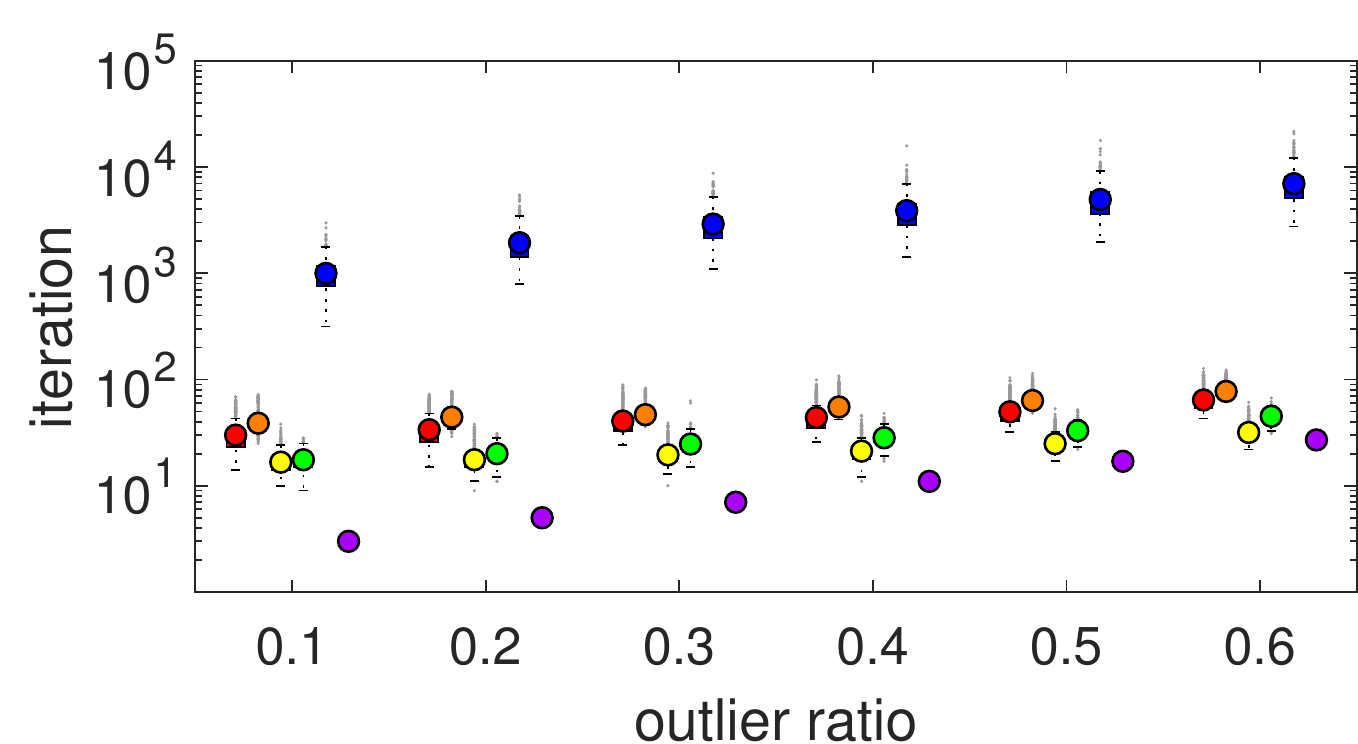}
		&\includegraphics[width=0.3\textwidth]{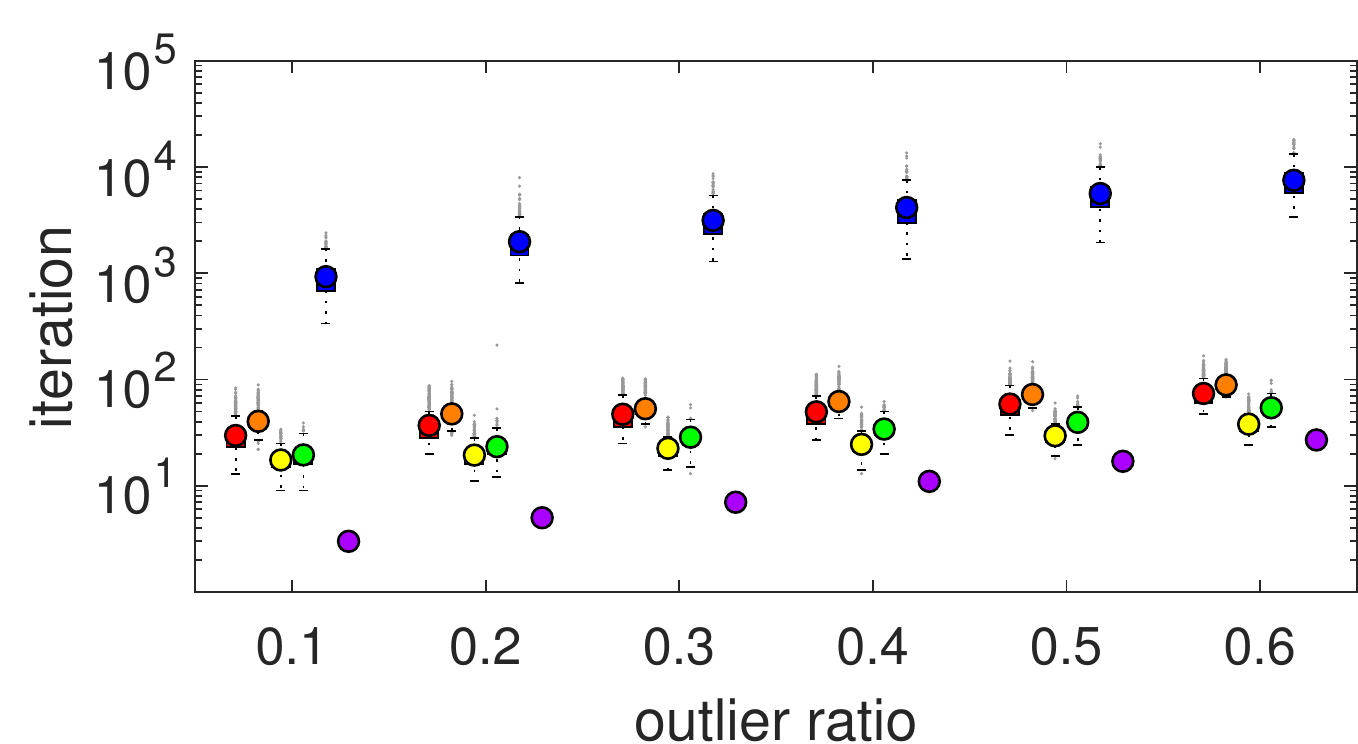}
	\end{tabular}
	\caption{Controlled experiments. The first row shows the vertical direction error $\varepsilon$ ($^\circ$) in different noise levels. The second row shows the runtime (second) in different noise levels. The third row shows the iteration count in different noise levels.  }
	\label{Fig:controlled-experiments}
\end{figure*}

In addition, to simulate the corrupted inputs in the synthetic experiments, noise and outlier were added. For noise, $\bm{e}_j\in \mathbb{R}^3$ was the $j$-th random vector, whose elements were randomly uniformly distributed in the interval $[-1,1]$. The noise was simulated by
\begin{equation}
\bm{n}_j \longleftarrow \dfrac{\bm{n}_j+\kappa\bm{e}_j}{\Vert\bm{n}_j+\kappa\bm{e}_j\Vert}
\end{equation}
where $\kappa$ was the amplitude of noise. For outliers, random orientations were added into the inputs. The total number of inputs was denoted $N$ and the number of outlier inputs  was denoted $N_o$, then $ \rho=N_{o}/N$ was the outlier proportion.

\subsection{Synthetic Data Experiments}

\subsubsection{Synthetic Atlanta World}

To simulate synthetic Atlanta world data, a random orientation was generated as the vertical direction ($\bm{v}_{gt}$). Except where otherwise specified, 20\% inlier inputs were parallel to vertical direction, and the other 80\% inlier inputs were randomly generated to be perpendicular to the vertical direction and thus in the "horizontal plane". Note that the number of the horizontal frames were not specified. The inlier threshold was   $\tau=\arctan(\kappa)$ according to the noise level in all the synthetic experiments. Once the vertical direction was estimated as $\bm{v}^*$, the error was calculated by 
\begin{equation}\label{Eq:error}
\varepsilon=\arccos(abs(\bm{v}_{gt}^T \bm{v}^*)) 
\end{equation}
To evaluate the results of the experiment, the vertical error and runtime were recorded. Additionally, because the iteration of the BnB algorithm reflected the tightness of the bounds, the iterations of BnB algorithm with different bounds were also recorded. Moreover, to reduce the randomness, 500 trials were repeated in each setting.

\textbf{Controlled experiments.} We first tested all the methods with different outlier ratios $\rho=\{0.1,\cdots,0.6\}$ and different noise levels $\kappa=\{0.005,0.010,0.020\}$. The number of input was   $N=500$.  The results are shown in Fig.~\ref{Fig:controlled-experiments}. From the results, we can draw the following conclusions:
\begin{itemize}
	\item All the four types of bounds in $\mathbb{S}^2$ and the bounds  of rotation search could be nested into the BnB algorithm to estimate the vertical direction globally in Atlanta world.  
	
	\item The four bounds in $\mathbb{S}^2$ had different efficiency. Nevertheless, the proposed bounds in $\mathbb{S}^2$ were more efficient than the bounds of rotation search. 
	
	\item Broadly, the exp-BnB and the ste-circle-BnB had similar efficiency.	The ste-square-BnB and the SCS-BnB had similar efficiency. More specifically, the first two were more efficient than the last two.  
	
	\item  Generally, the ste-square-BnB and the SCS-BnB had fewer iterations than  the exp-BnB and the ste-circle-BnB. It revealed the  ste-square bounds and the SCS bounds were tighter, which was consistent with the previous theoretical analysis.
\end{itemize}

\begin{figure*}[htp]
	\centering
	\includegraphics[scale=0.35]{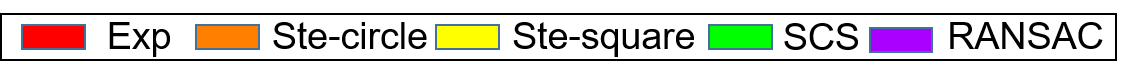}
	\begin{tabular}{ccc}
		
		$\kappa=0.005$&$\kappa=0.010$&$\kappa=0.020$
		
		\\
		
		\includegraphics[width=0.3\textwidth]{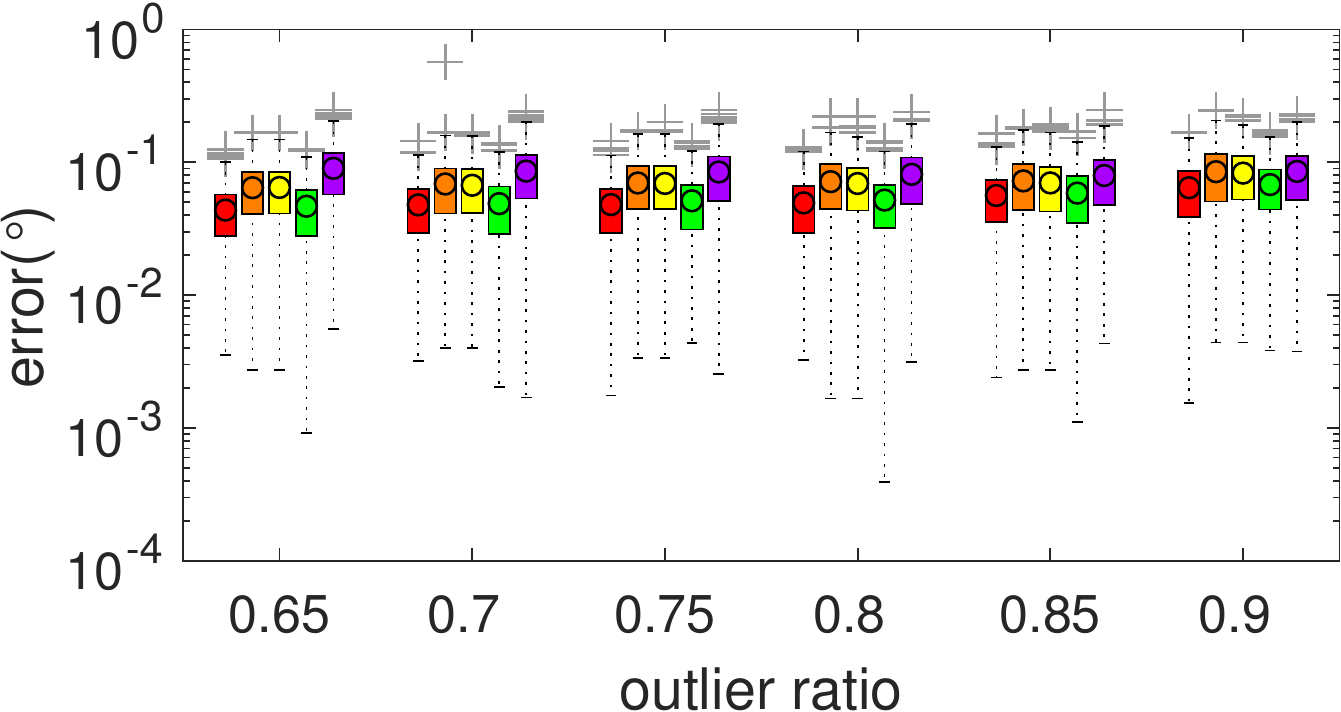}
		&\includegraphics[width=0.3\textwidth]{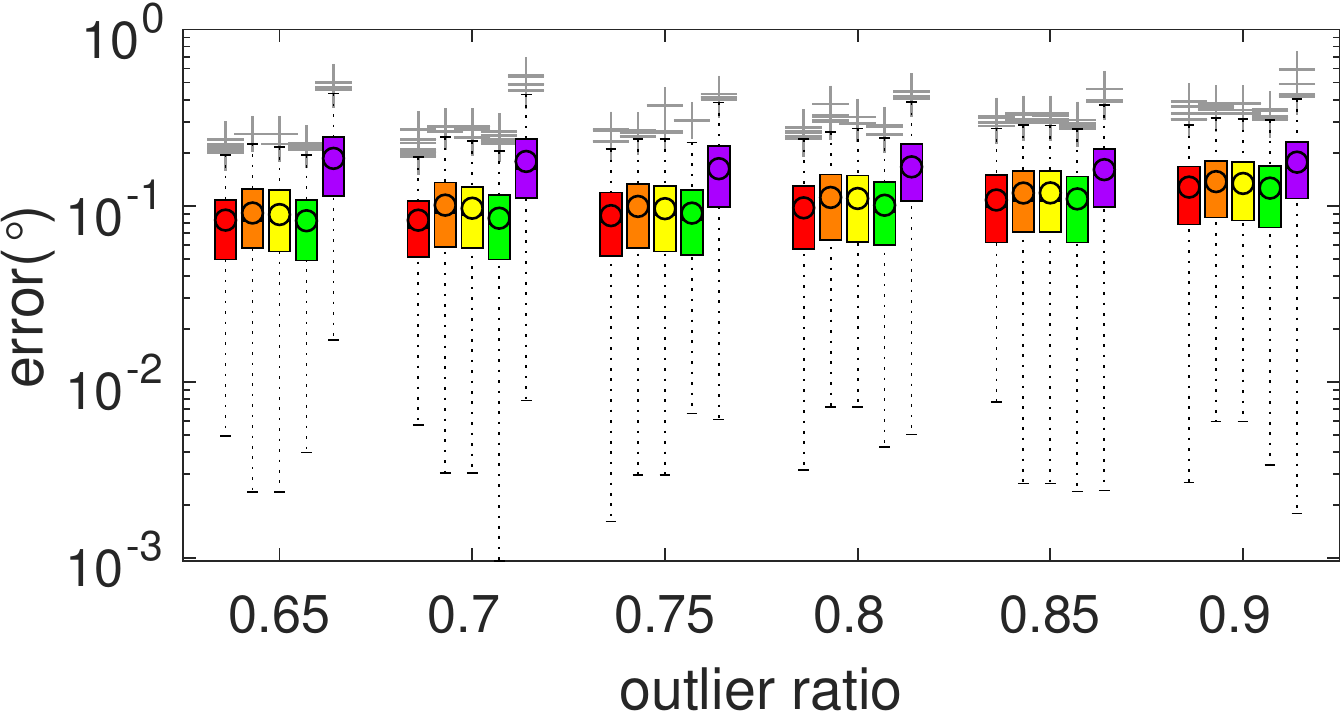}
		&\includegraphics[width=0.3\textwidth]{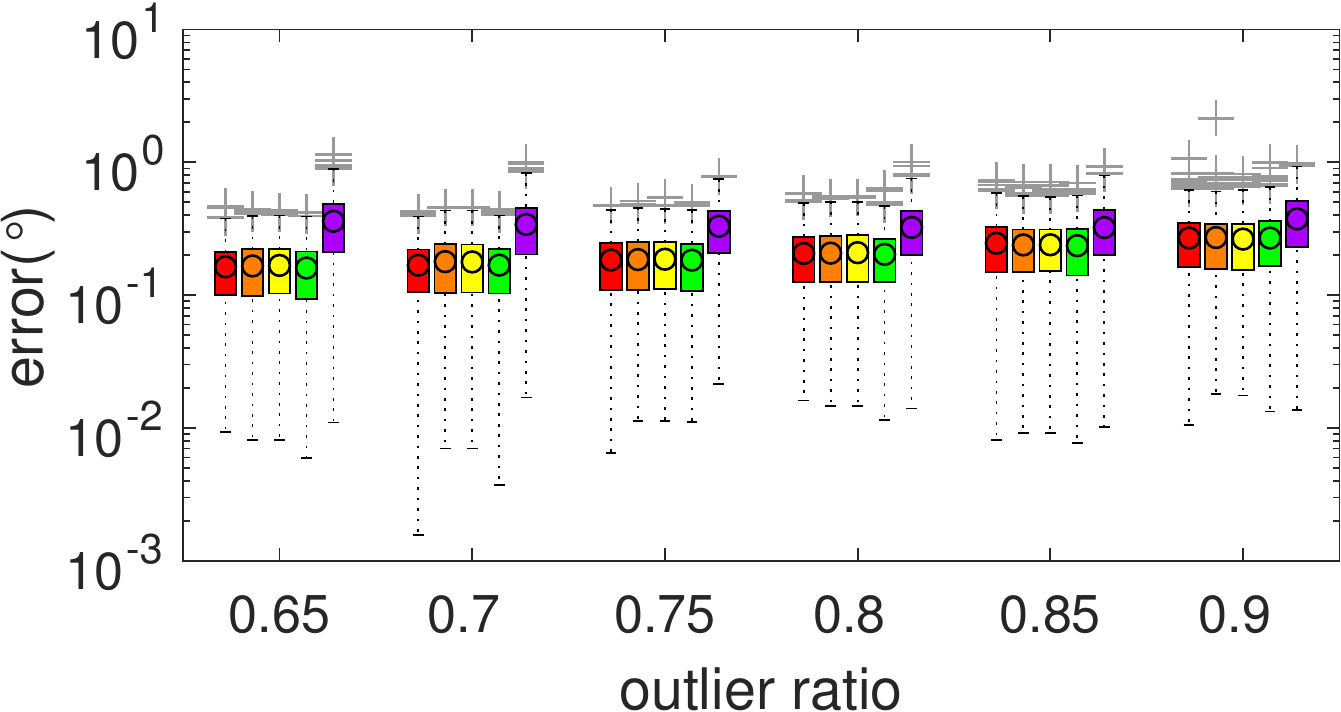}
		\\
		\includegraphics[width=0.3\textwidth]{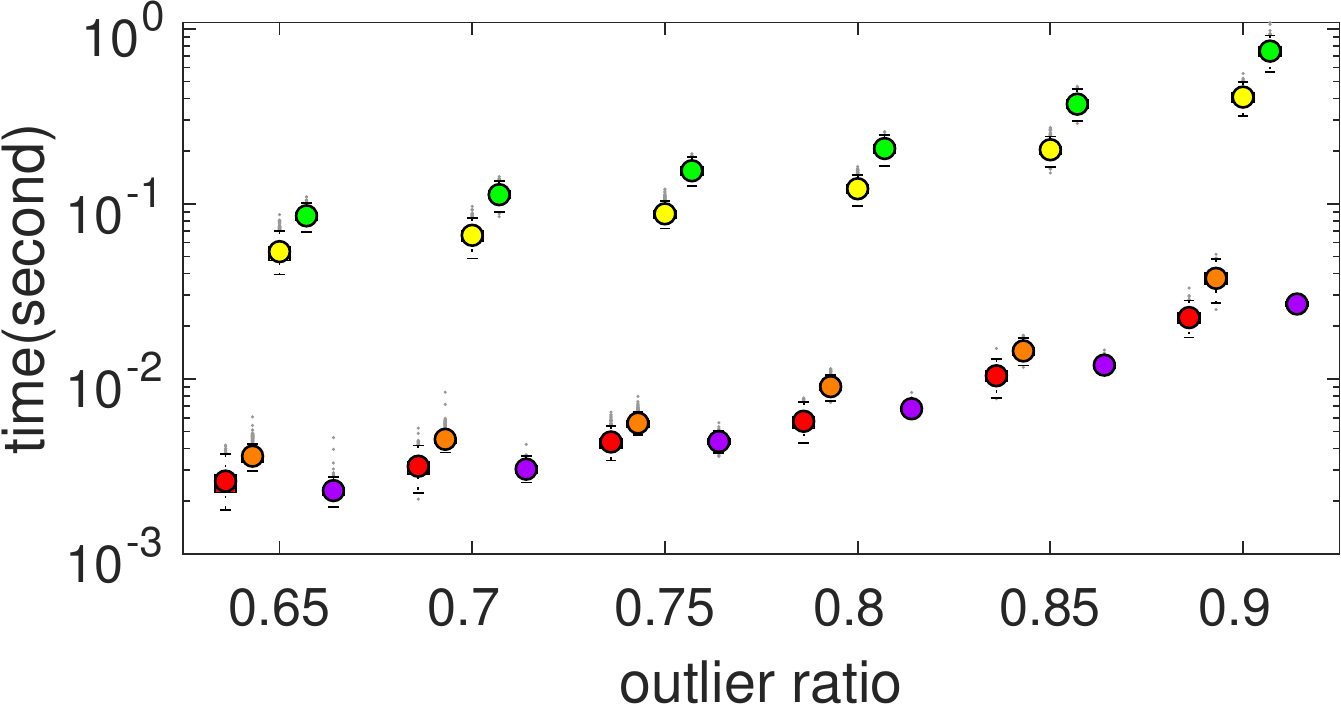}
		&\includegraphics[width=0.3\textwidth]{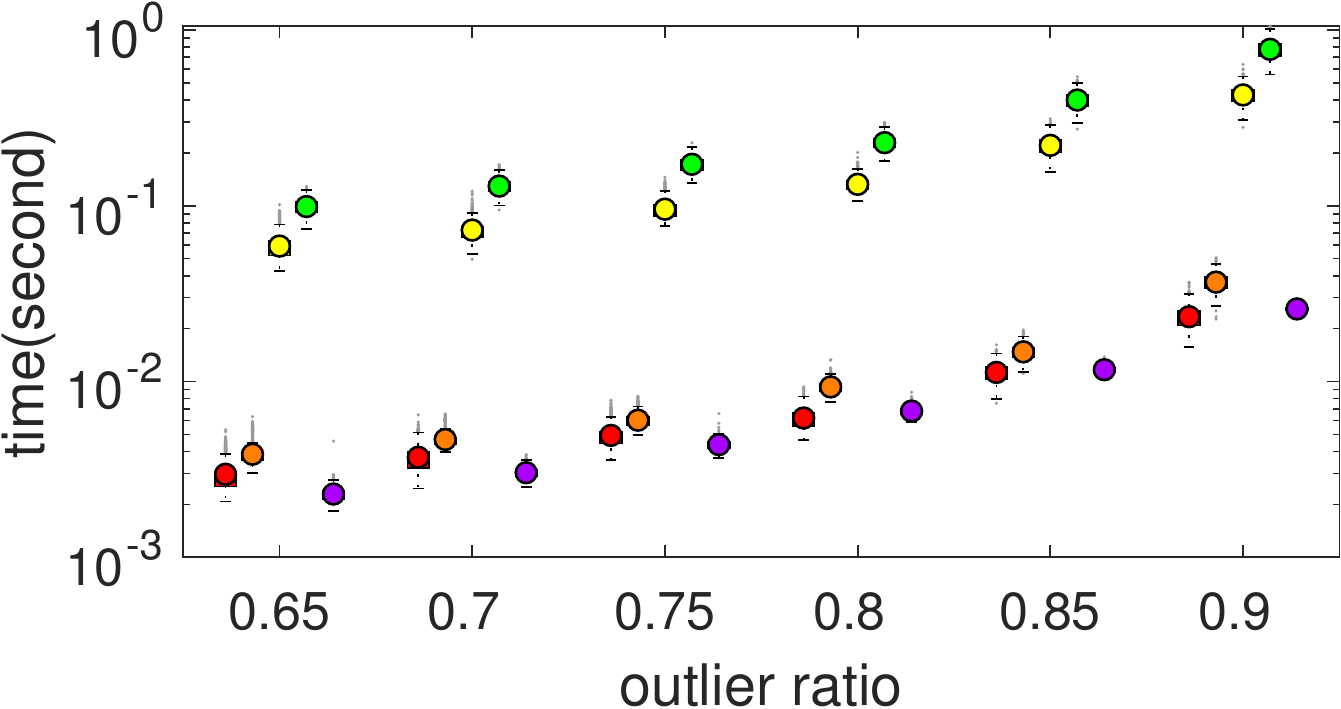}
		& \includegraphics[width=0.3\textwidth]{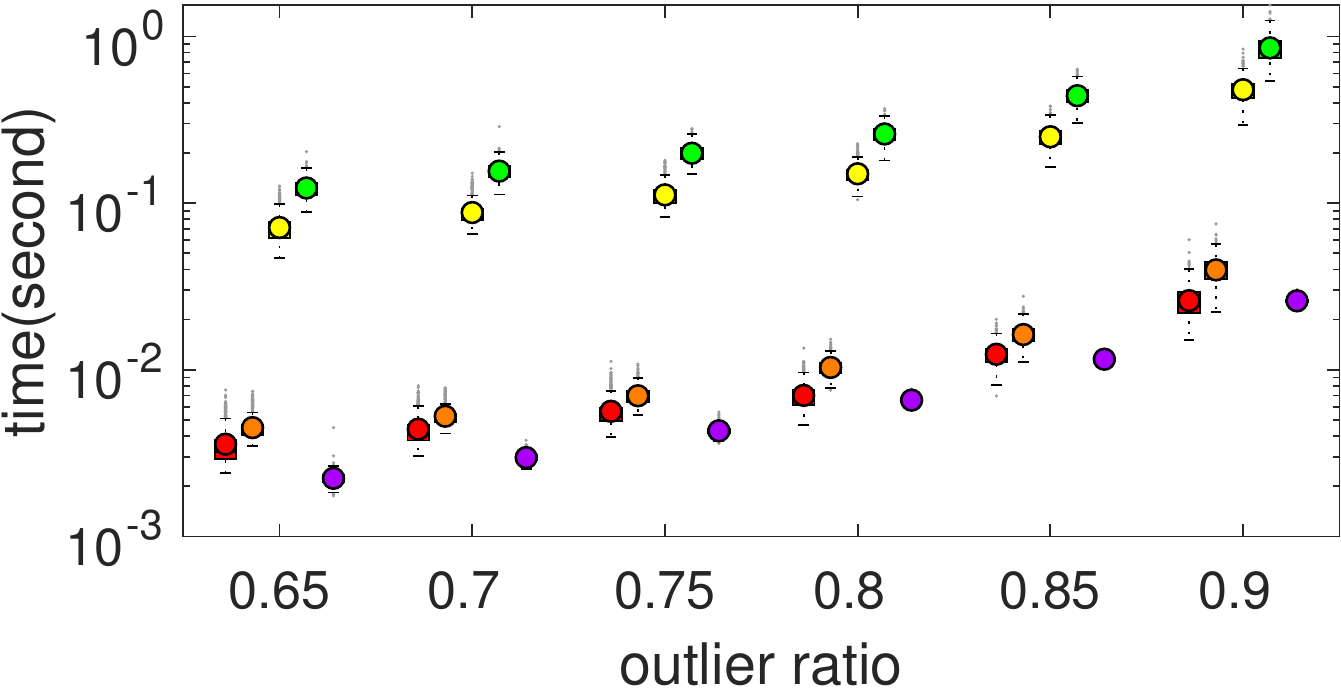}
		\\
		\includegraphics[width=0.3\textwidth]{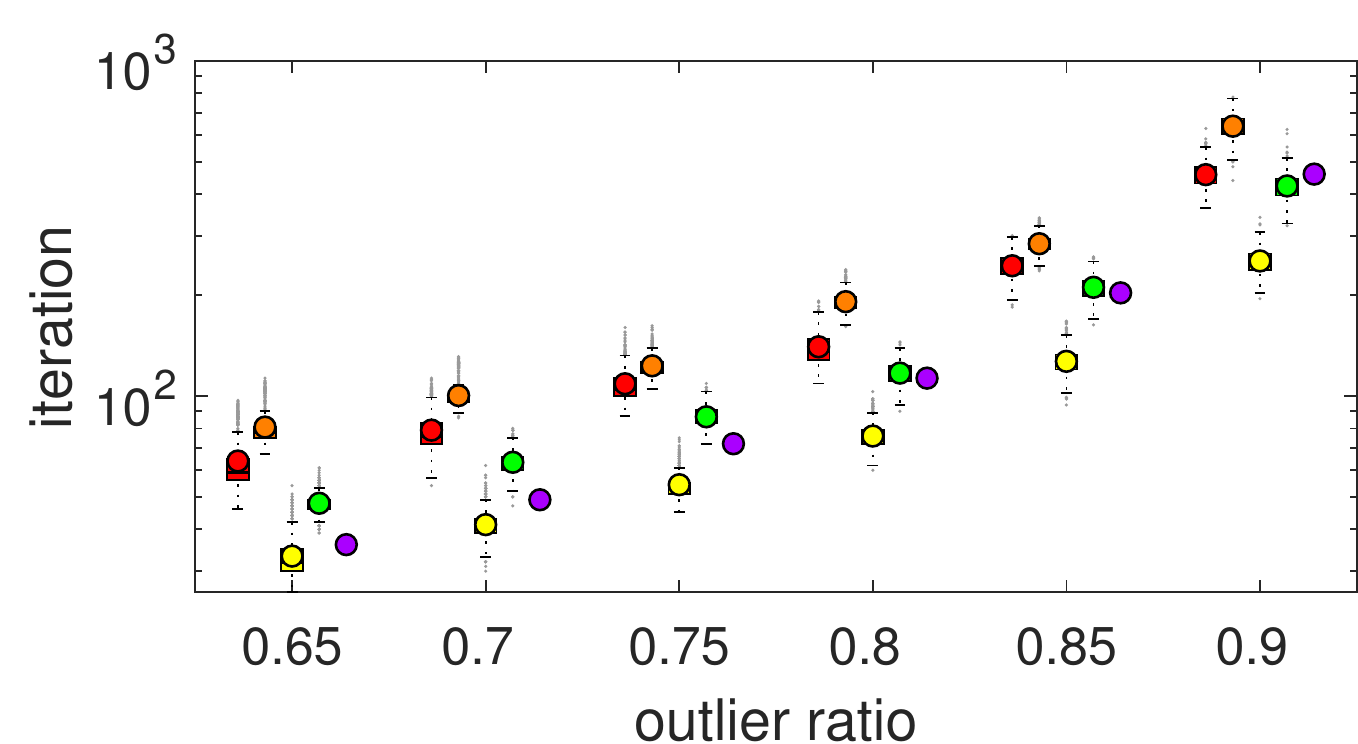}
		& \includegraphics[width=0.3\textwidth]{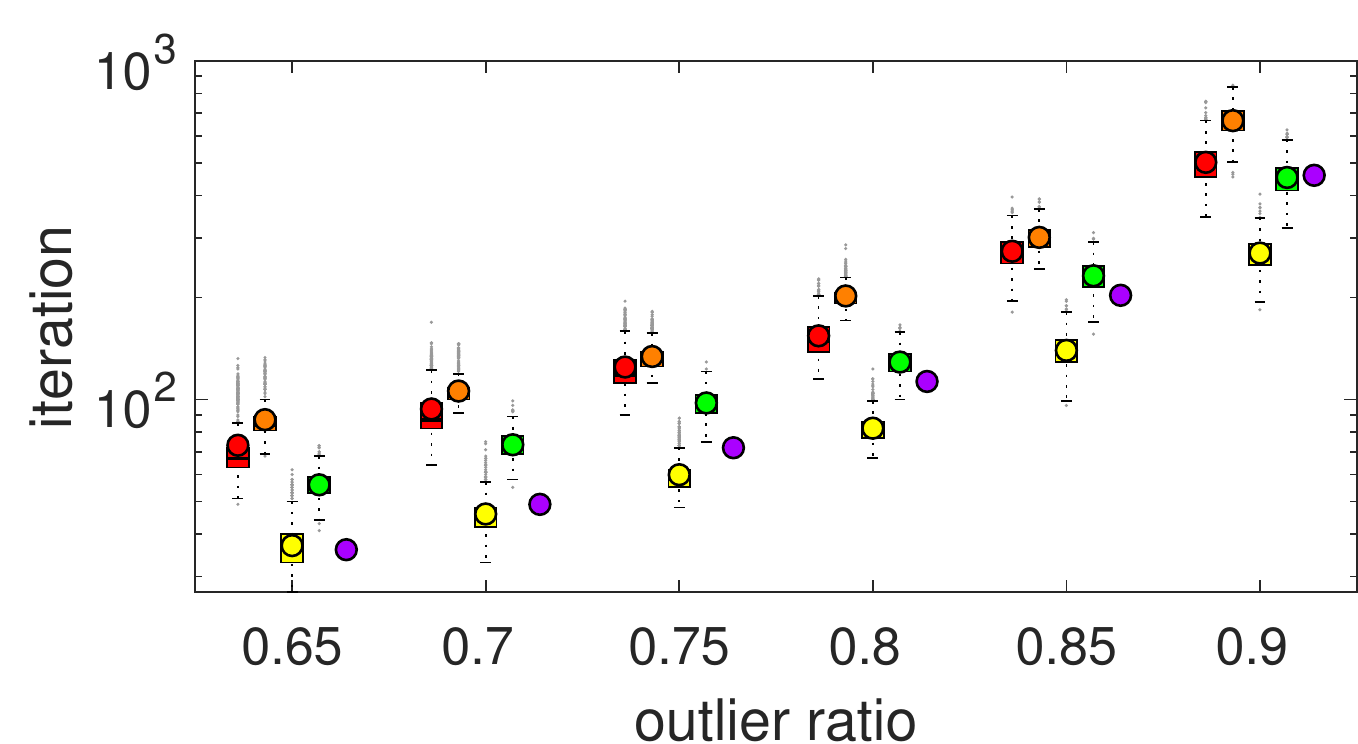}
		&\includegraphics[width=0.3\textwidth]{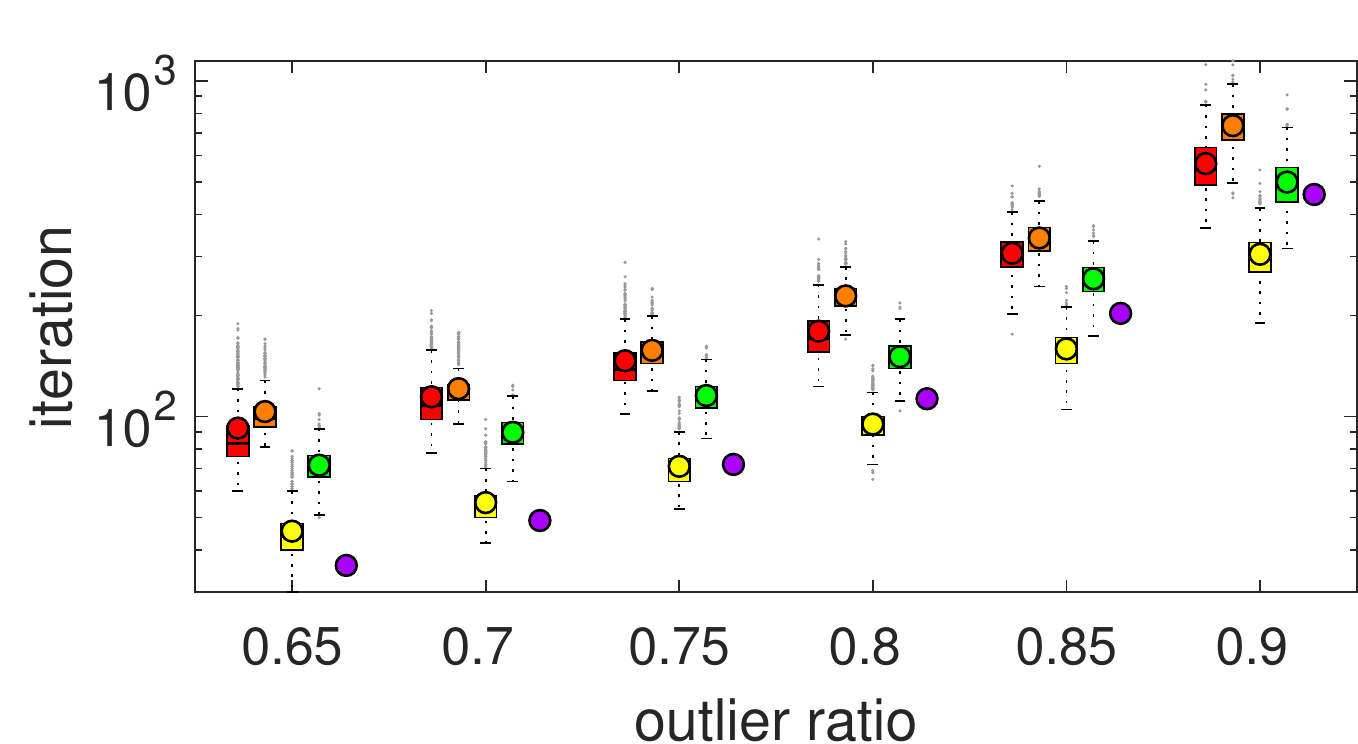}
	\end{tabular}
	\caption{High outlier ratio experiments.The first row shows the vertical direction error $\varepsilon$ ($^\circ$) in different noise levels. The second row shows the runtime (second) in different noise levels. The third row shows the iteration count in different noise levels.}\label{Fig:high-outlier}
\end{figure*}

\begin{figure*}
	\centering
	\includegraphics[scale=0.35]{./img-eps/all-RS.png}
	\begin{tabular}{ccc}
		$\kappa=0.050$&$\kappa=0.100$&$\kappa=0.200$
		\\
		\includegraphics[width=0.3\textwidth]{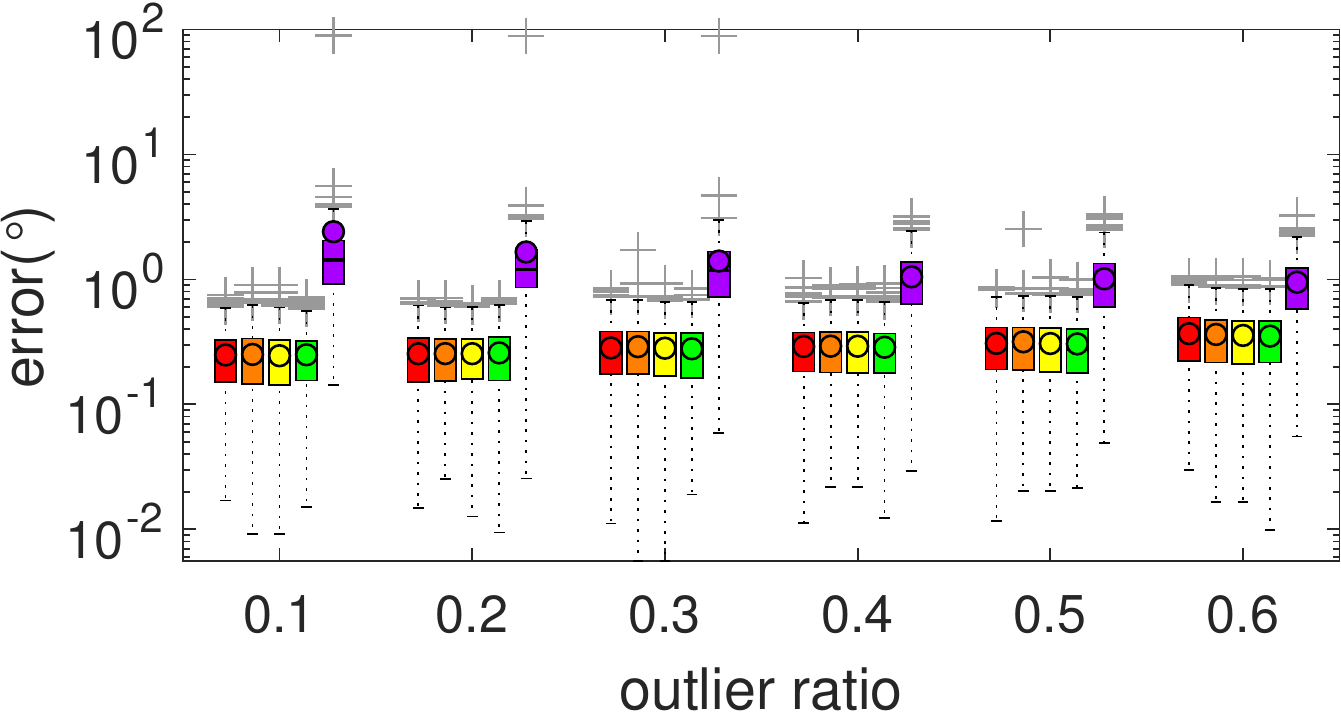}
		&\includegraphics[width=0.3\textwidth]{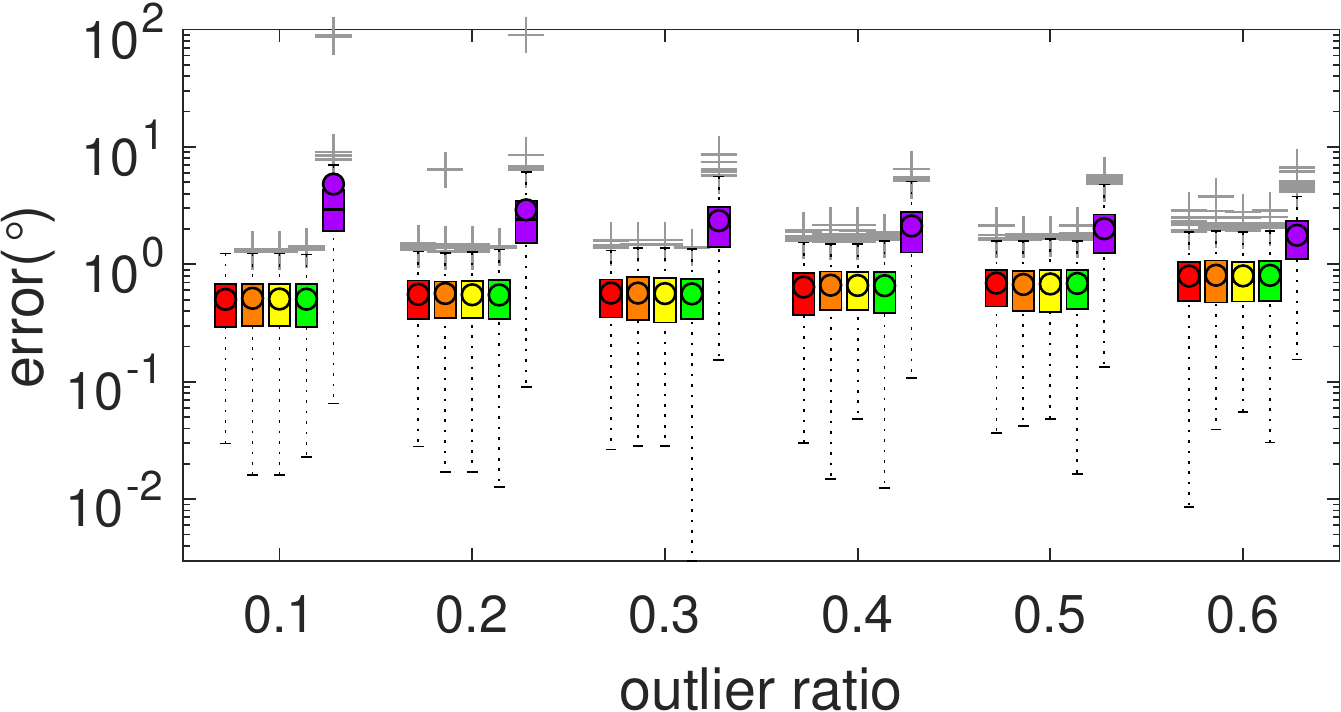}
		&\includegraphics[width=0.3\textwidth]{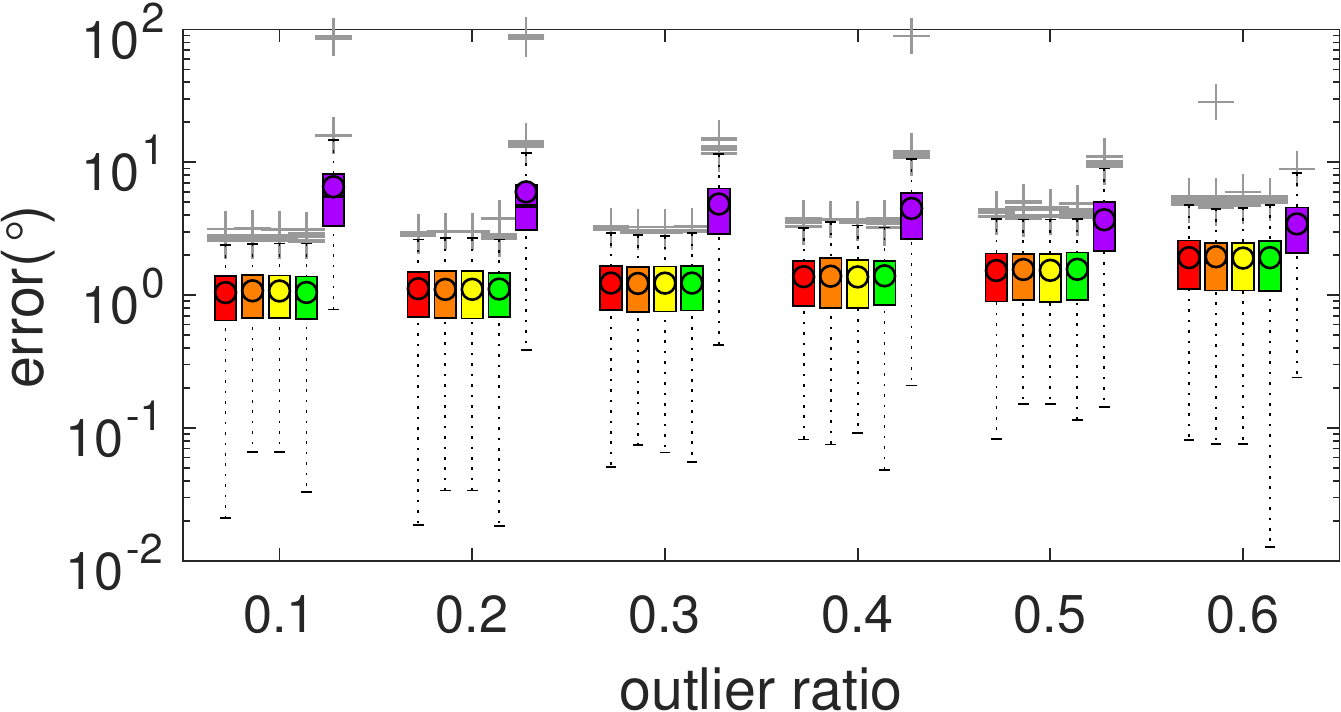}
		\\
		\includegraphics[width=0.3\textwidth]{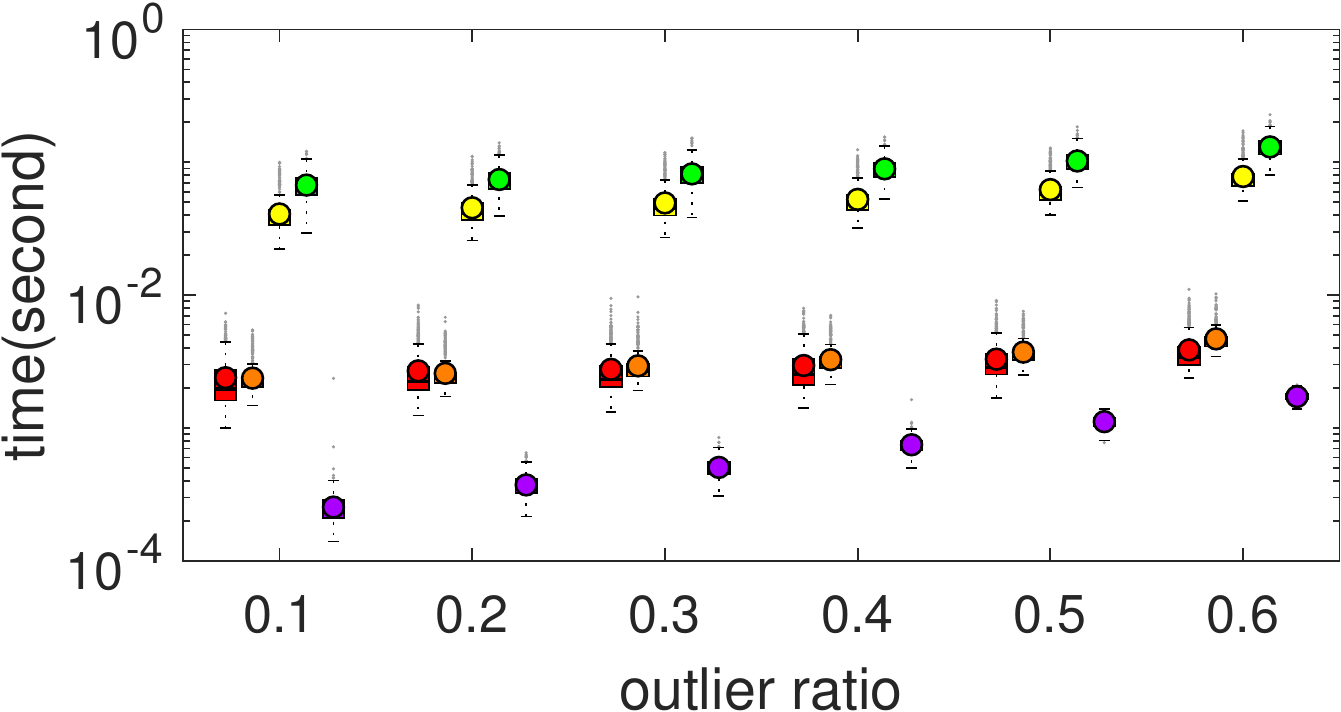}
		&\includegraphics[width=0.3\textwidth]{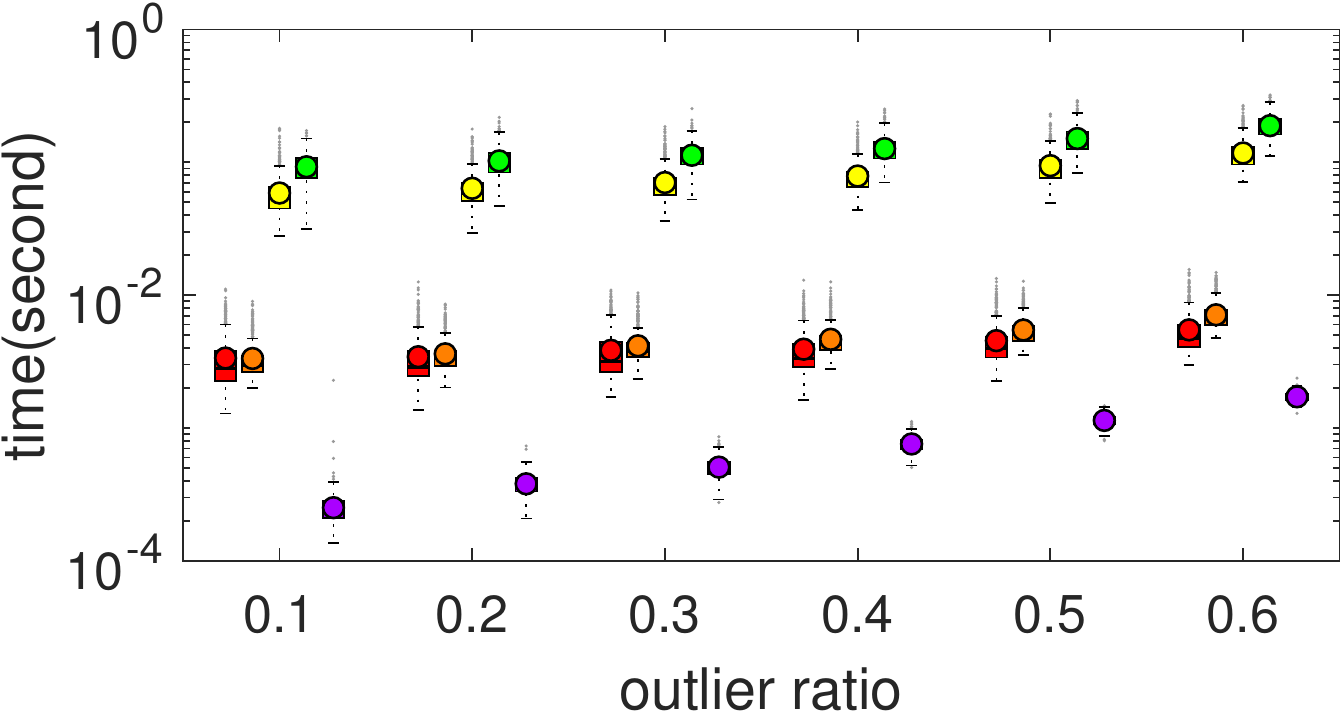}
		& \includegraphics[width=0.3\textwidth]{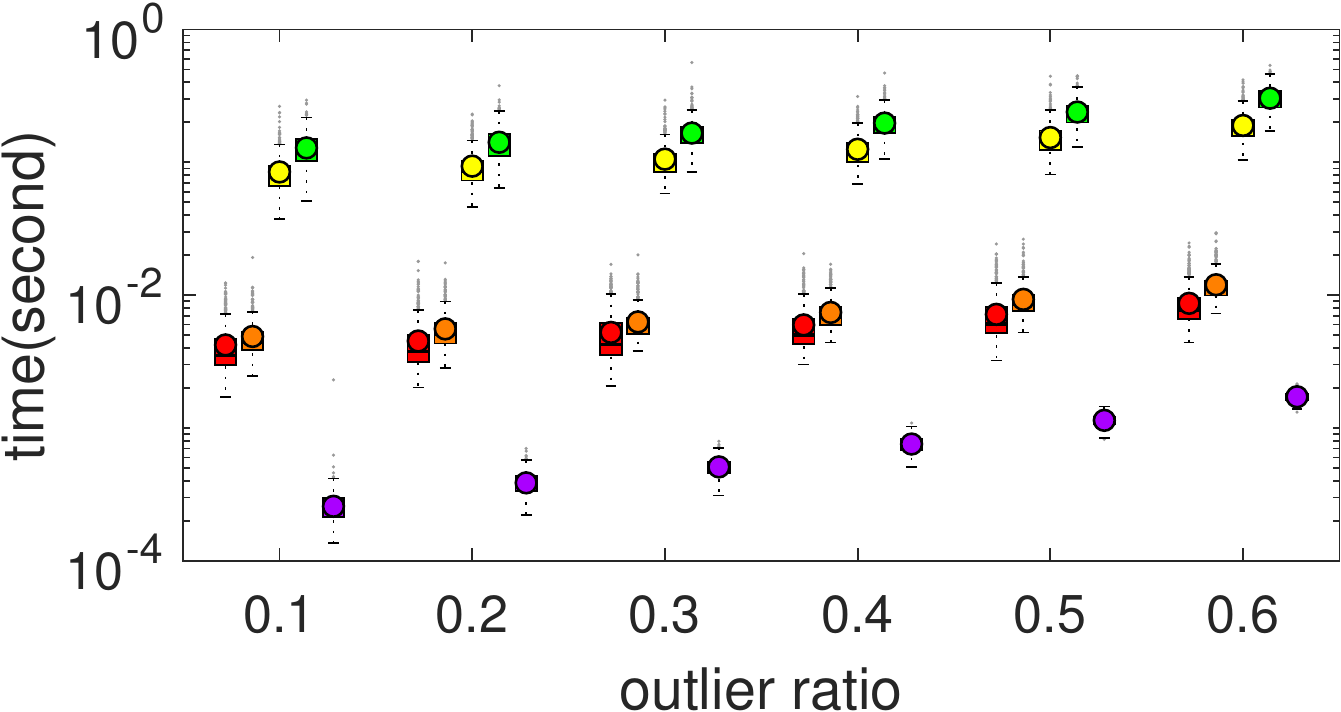}
		\\
		\includegraphics[width=0.3\textwidth]{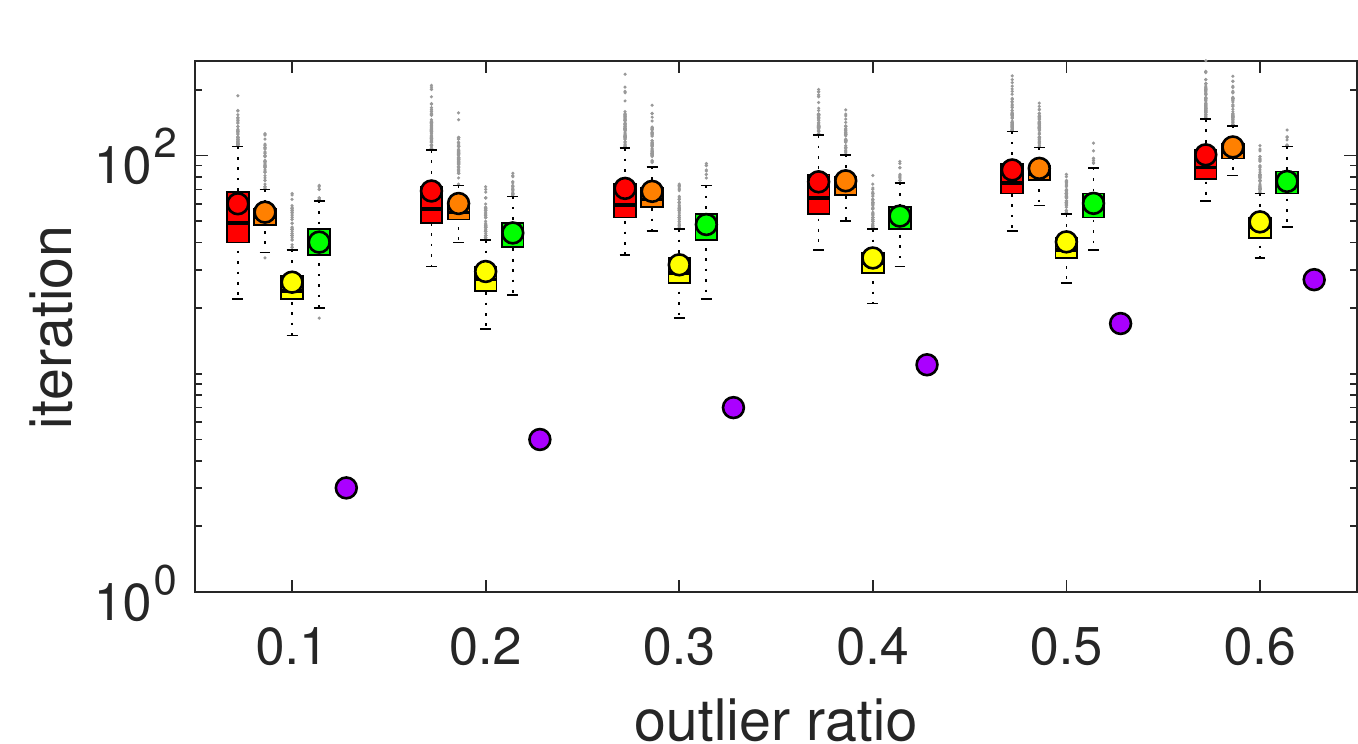}
		& \includegraphics[width=0.3\textwidth]{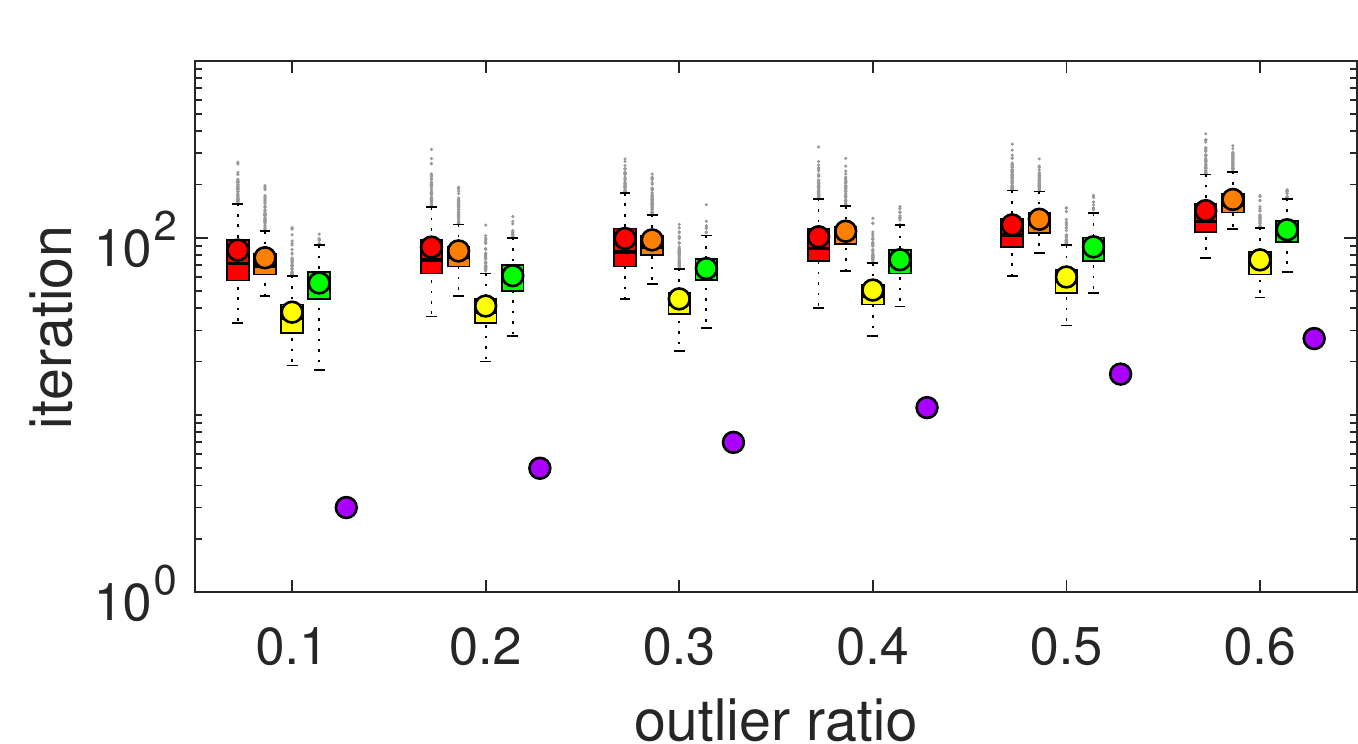}
		&\includegraphics[width=0.3\textwidth]{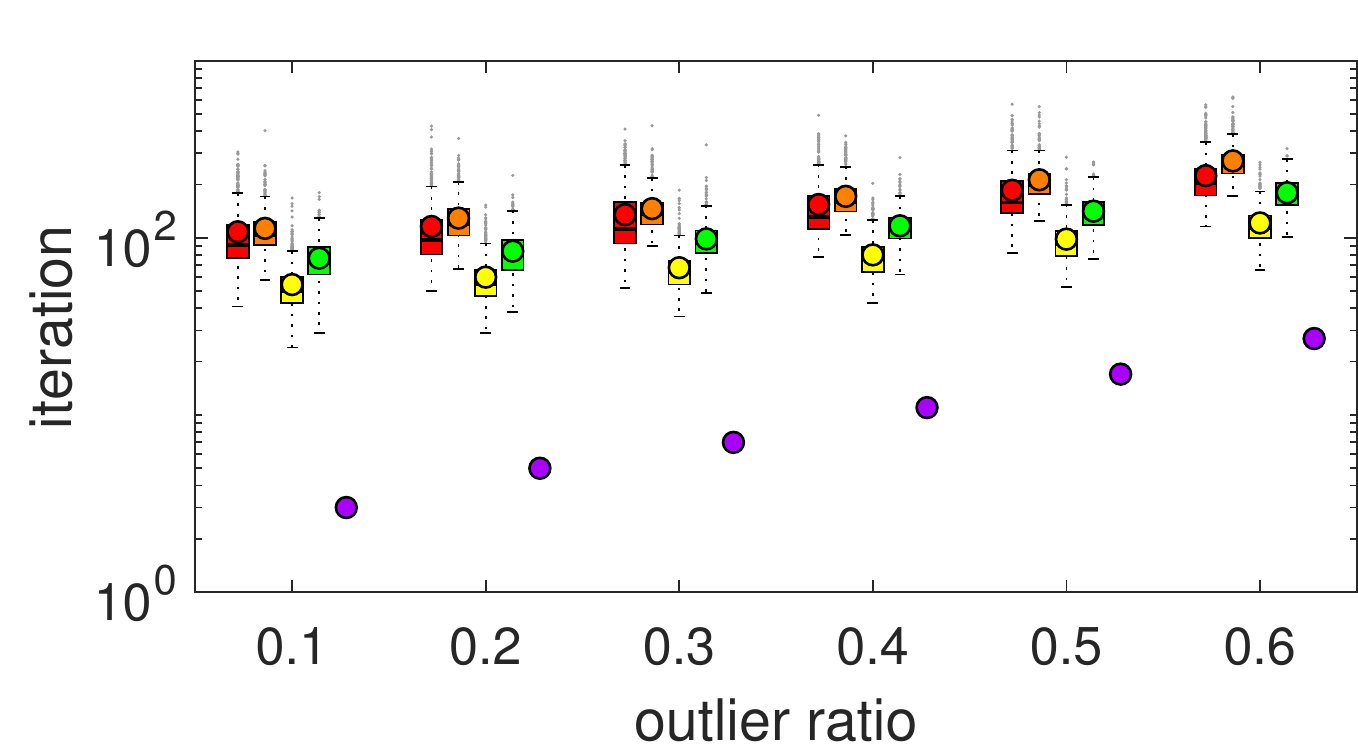}
	\end{tabular}
	\caption{Large noise experiments. The first row shows the vertical direction error $\varepsilon$ ($^\circ$) in different noise levels. The second row shows the runtime (second) in different noise levels. The third row shows the iteration count in different noise levels.}\label{Fig:large-noise}
\end{figure*}

\begin{figure*}
	\centering
	\includegraphics[scale=0.35]{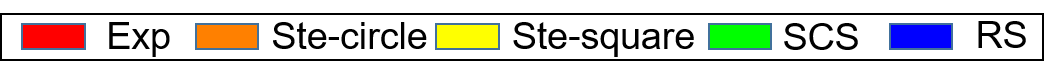}
	\begin{tabular}{ccc}
		$\kappa=0.005$&$\kappa=0.010$&$\kappa=0.020$
		\\
		\includegraphics[width=0.3\textwidth]{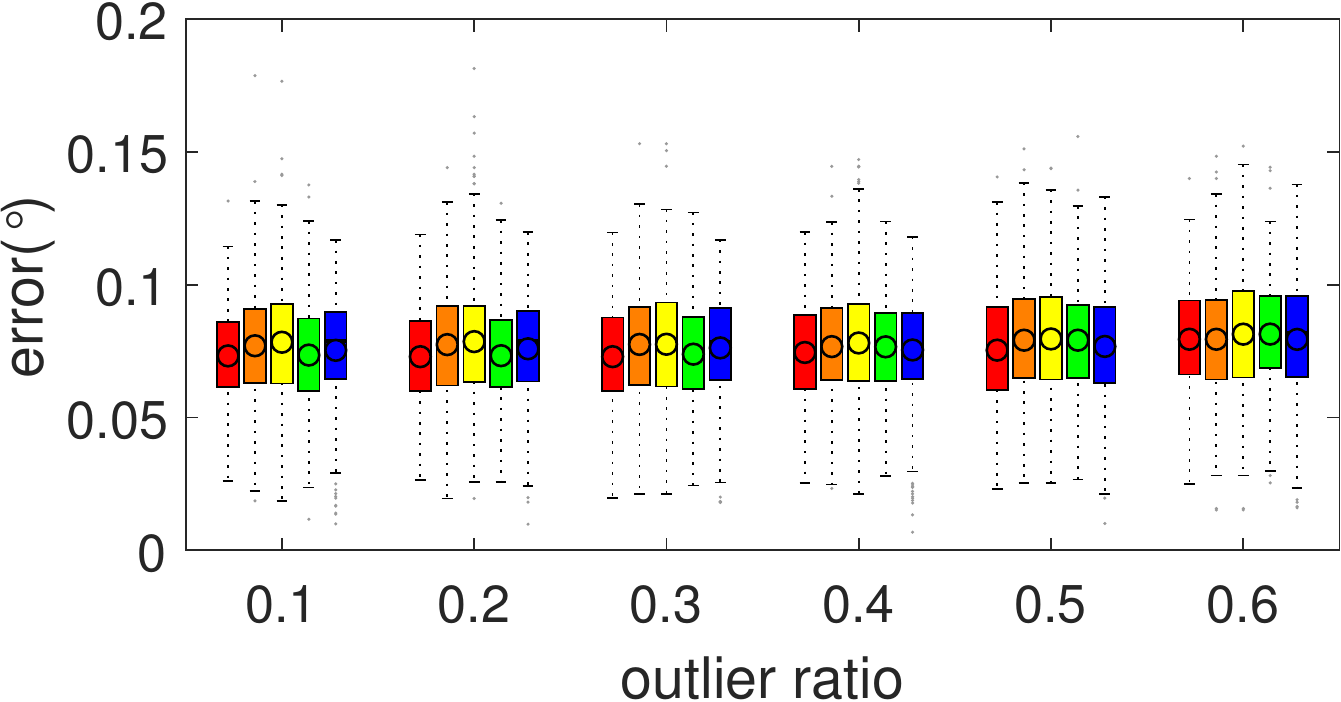}
		&\includegraphics[width=0.3\textwidth]{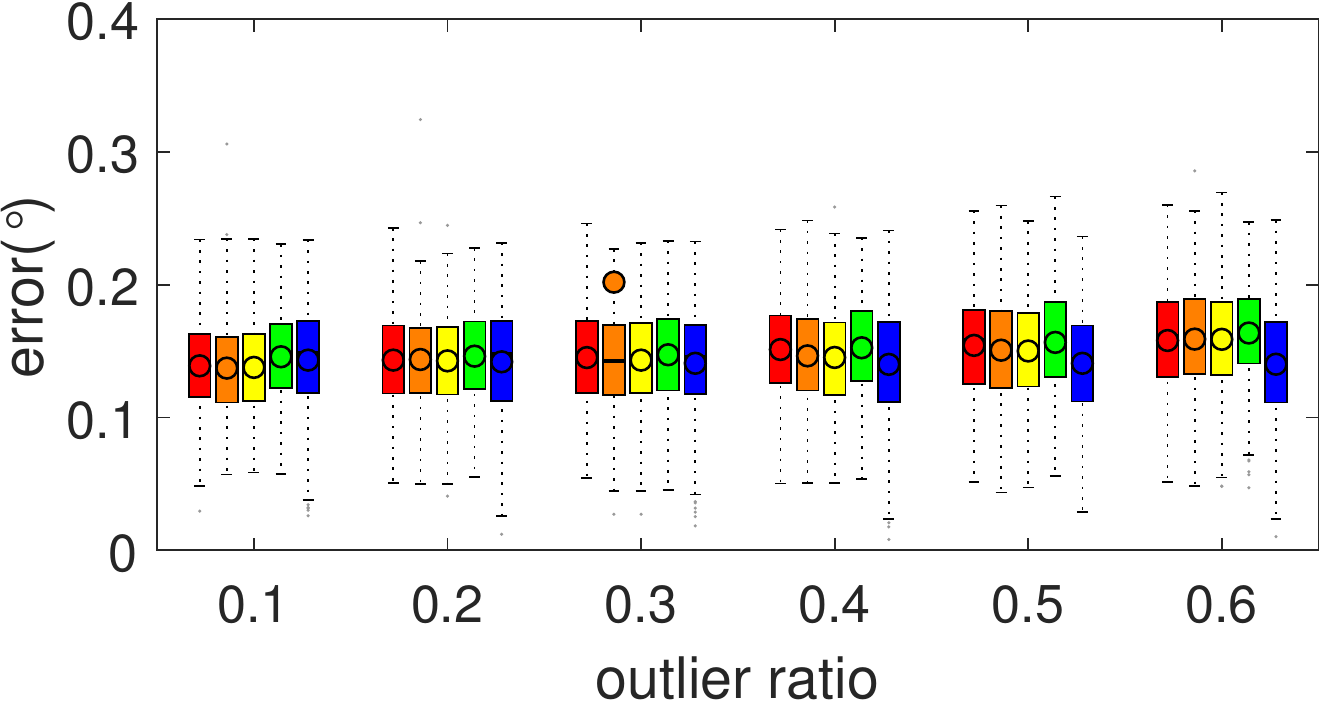}
		&\includegraphics[width=0.3\textwidth]{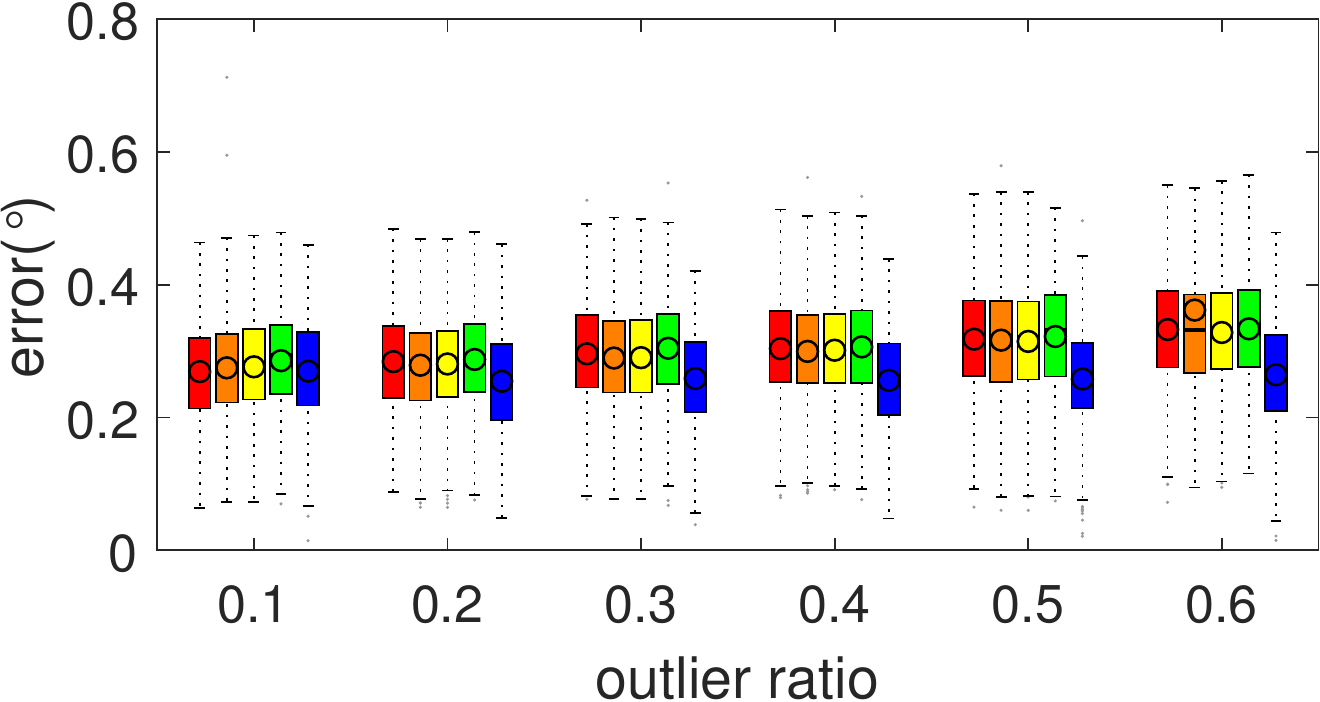}
		\\
		\includegraphics[width=0.3\textwidth]{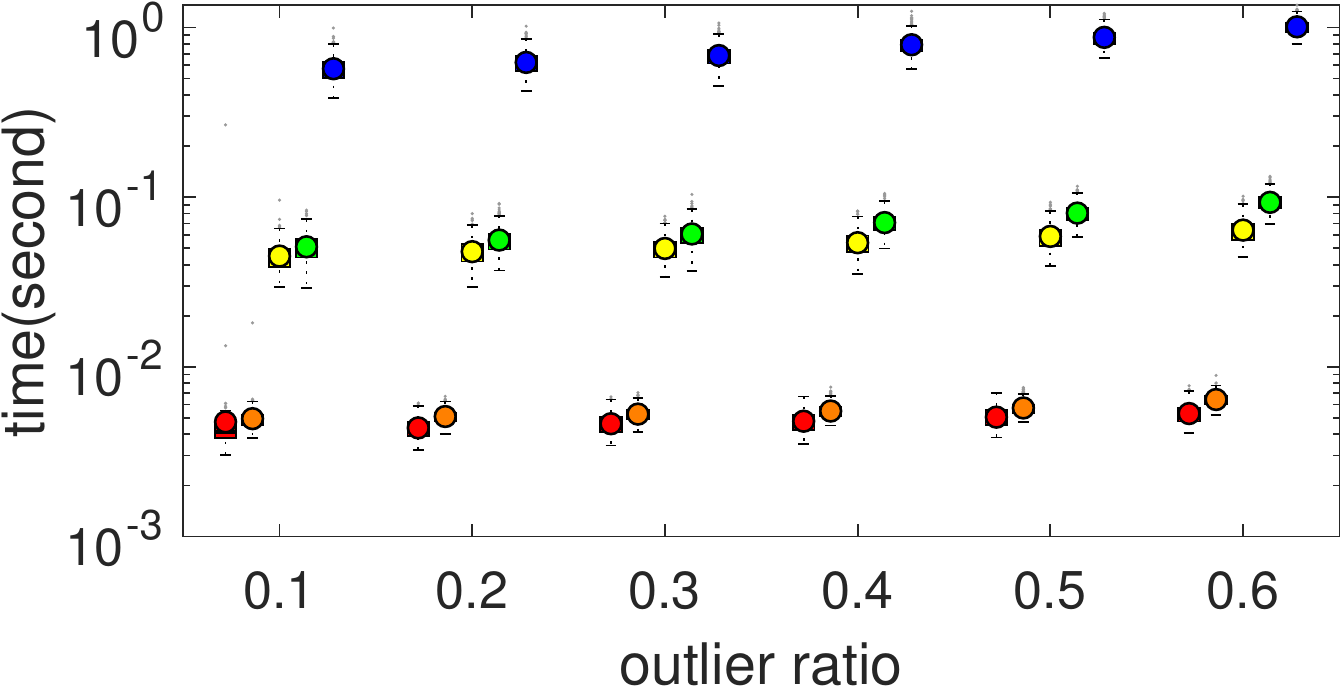}
		&\includegraphics[width=0.3\textwidth]{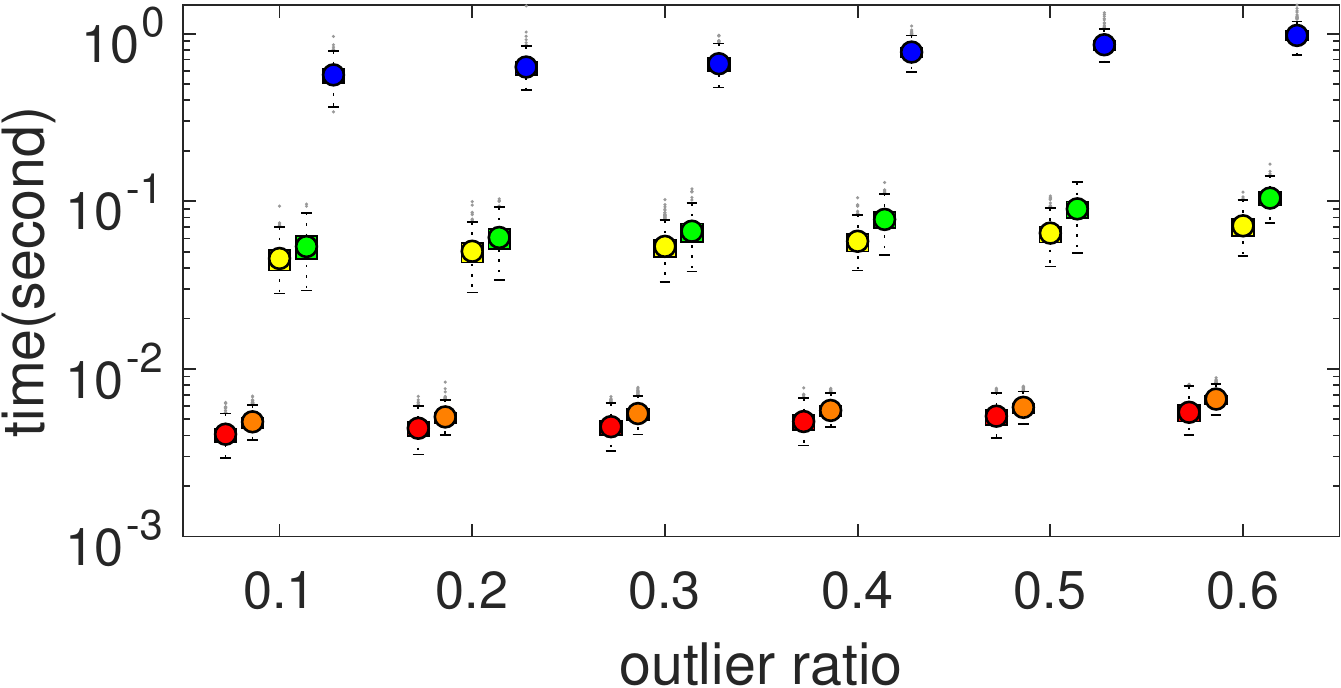}
		&\includegraphics[width=0.3\textwidth]{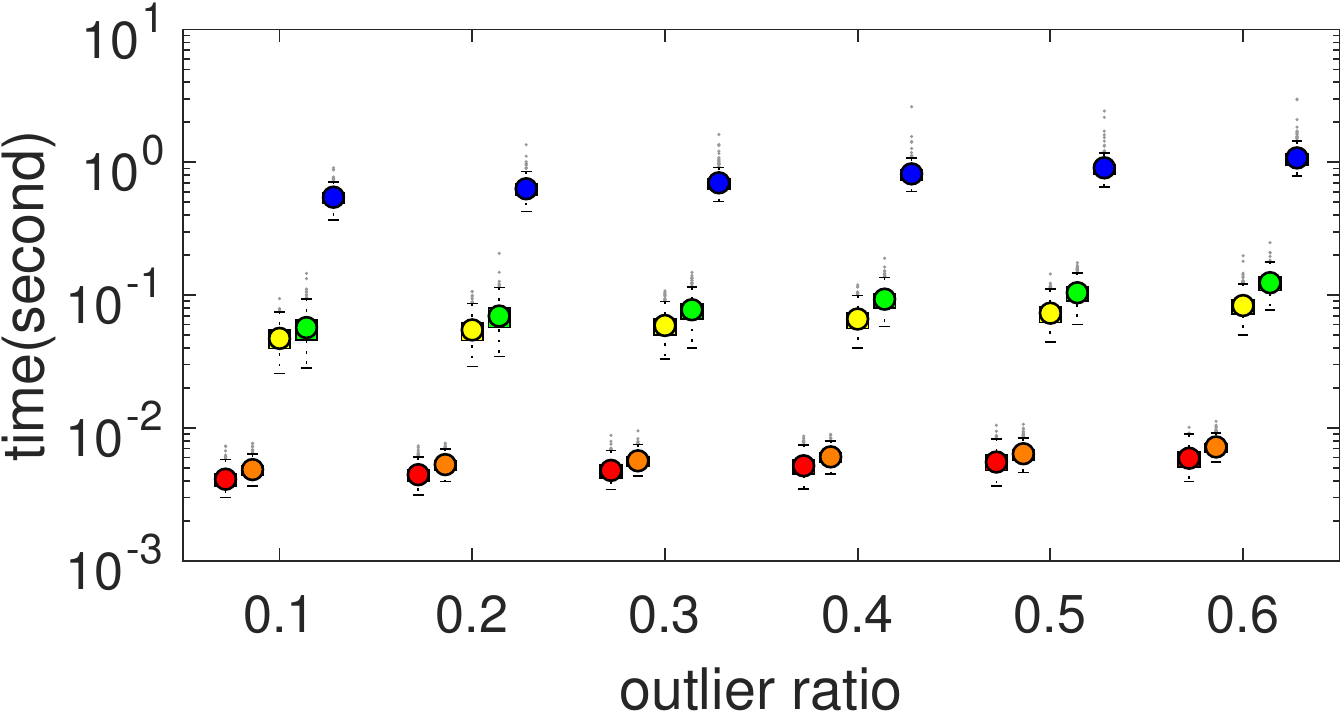}
		\\
		\includegraphics[width=0.3\textwidth]{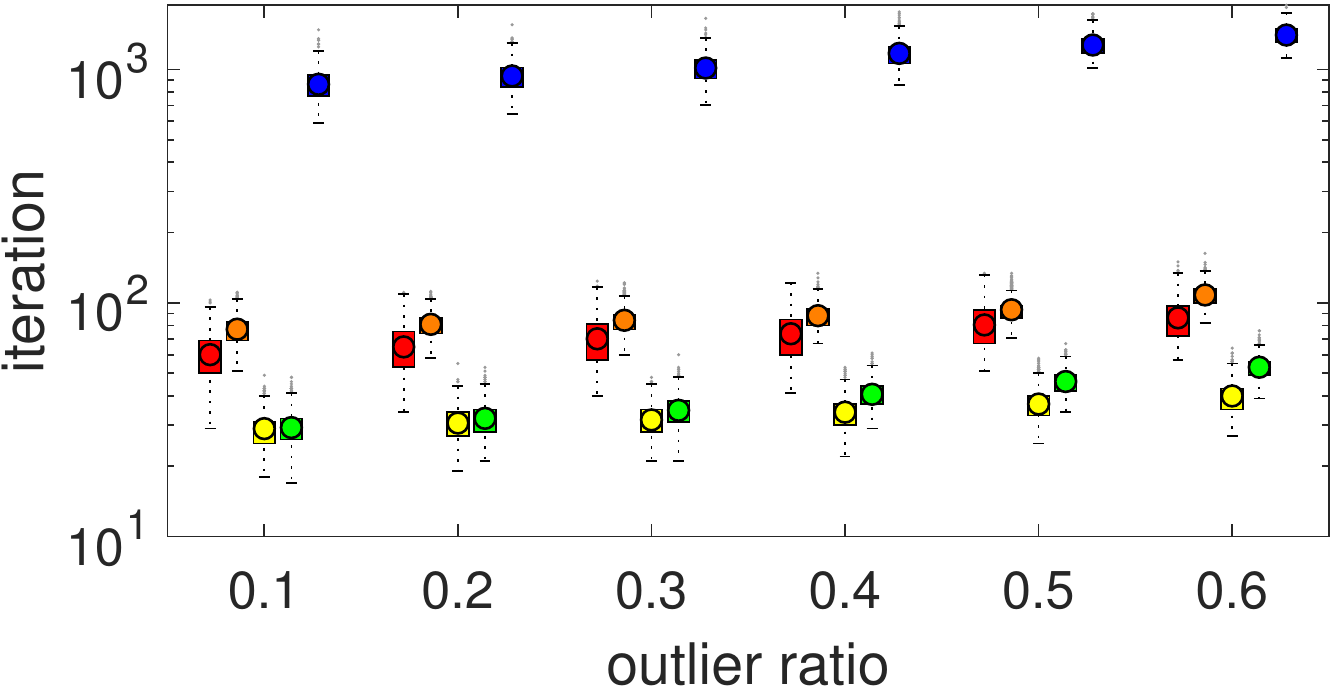}
		&\includegraphics[width=0.3\textwidth]{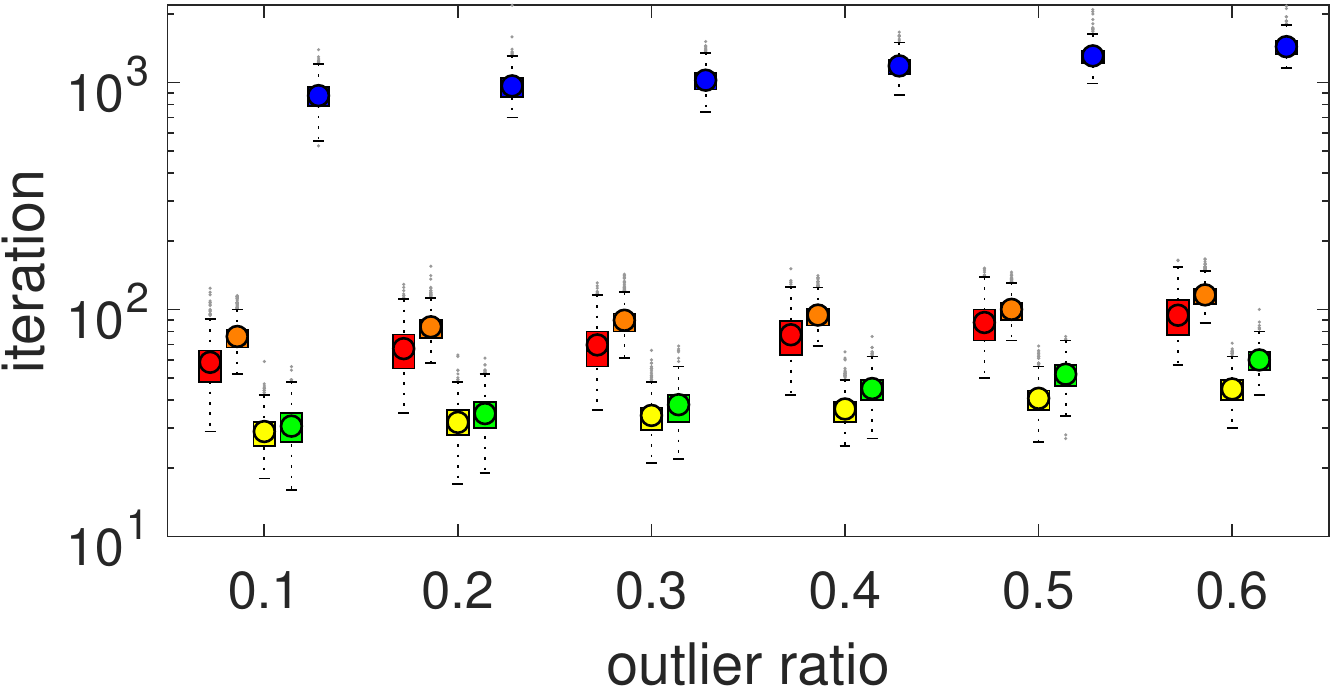}
		&\includegraphics[width=0.3\textwidth]{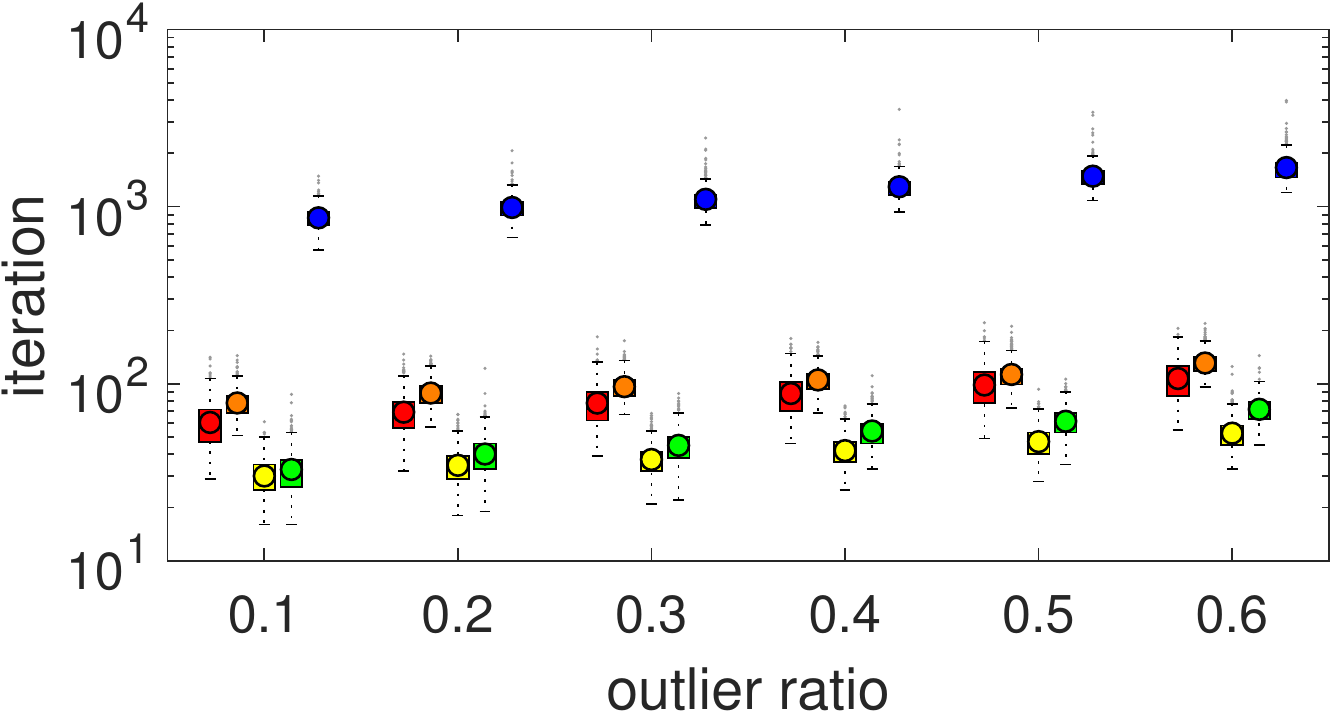}
	\end{tabular}
	\caption{Full Atlanta frame estimation experiments. The first row shows the frame error $\varepsilon$ ($^\circ$) in different noise levels. The second row shows the runtime (second) in different noise levels. The third row shows the iteration count in different noise levels.}\label{Fig:Manhattan}
\end{figure*}

There are three main reasons why rotation search is rather inefficient for vertical direction estimation.
\begin{enumerate}
	\item Multiple solutions. Since $R\bm{v}_0=\bm{v^*}$, if the initial  direction $\bm{v}_0$ and the optimal vertical direction $\bm{v^*}$ are fixed, there are numerous solutions for $R$~\cite{bustos2017guaranteed}. In other words, if $R\bm{v}_0=\bm{v^*}$ is correct, then $R_vR\bm{v}_0=R_v\bm{v^*}=\bm{v^*}$ is also correct, where $R_v$ is a rotation  about axis $\bm{v^*}$. Therefore, all $R_vR$ are solutions.  For the BnB algorithm, if there are multiple solutions, there are many branches which are very close to one of the solutions and their objective value are also close to the ground truth, then the BnB algorithm must spend lots of time pruning the branches.
	\item Higher dimensionality. Since the vertical direction is inherently a two dimensional problem, searching in higher dimension leads to lower efficiency.
	\item Conservative bound. Since rotation search bounds have a three-step geometrical relaxation, the bounds are relatively conservative.
\end{enumerate}

Furthermore, why exp-BnB and ste-circle-BnB  algorithms had more iterations  while they still run faster? This was because on one hand, tighter bounds would remove more aggressively and yield fewer iterations. However, on the other hand, using  tighter bounds in BnB might be counter-productive if calculating the bound itself took significant time.

\textbf{Challenging experiments.} We conducted more experiments on challenging data. In this part, we only tested the bounds in $\mathbb{S}^2$, as the  bounds of rotation search were obviously less efficient.  The number of input was fixed $N=500$. First, all methods were tested on different high outlier ratios $\rho=\{0.65,\cdots,0.9\}$ and different noise levels $\kappa=\{0.005,0.010,0.020\}$. The results are shown in Fig.~\ref{Fig:high-outlier}. Second, all methods were tested on different large noise levels $\kappa=\{0.050,0.100,0.200\}$ and different outlier ratios $\rho=\{0.1,\cdots,0.6\}$. The results are shown in Fig.~\ref{Fig:large-noise}. From the all the results, we can draw the following conclusions:

\begin{itemize}
	\item The exp-BnB had the highest efficiency among all BnB-based methods in such experimental settings. It is worth noting that in large outlier ratio cases ($\rho\geq 0.8$), the exp-BnB algorithm even had comparable efficiency with RANSAC.
	\item The ste-square-BnB  had the least iterations among all BnB-based methods, which showed the bounds were very tight. 
\end{itemize} 

Theoretically, both ste-square-bounds  and SCS-bounds  have no geometrical relaxations, then why SCS-bounds needed more iterations than ste-square-bounds in the
challenging experiments? This was due to the large distortion of the searching domain. For example, the domain near the optimal direction in $\mathbb{S}^2$ might be expanded to a scale-up region in azimuth-elevation rectangle, therefore, the BnB algorithm needed more iterations to prune the near-optimal branches. 

\subsubsection{Full Atlanta frame estimation}

In this part, we verified the performance of our proposed bounds in full Atlanta frame estimation problem.  For the sake of fairness, the experiments were conducted on synthetic Manhattan world and the rotation search method was from~\cite{bazin2012globally-ACCV} without EGI-acceleration. Our proposed methods first estimated the vertical frame direction, and then estimated the horizontal frames  by a one-dimension clustering method, which can be called sequential methods (see appendix~G for more details). 

To generate the input normals in Manhattan world, we randomly selected a point $R_{gt}$ in $SO(3)$ as the Manhattan frames. In other words, each column of $R_{gt}$ was corresponding to a Manhattan frame. The experimental settings were $N=500,  \kappa=\{0.005,0.010,0.020\}$ and $\rho$ was from 0.1 to 0.6. Once the frame directions had been estimated as $R^*$, the estimation error was measured by
\begin{equation}
\varepsilon_{m}=mean\big(\arccos\big(max(abs(R_{gt}^TR^*))\big)\big)
\end{equation}
where $mean({\cdot})$ was average function; $\arccos(\cdot)$ was element-wise arccosine function; $max(\cdot)$ was column-wise max function. It computed the average error of the three frames. 
Note that the solution of rotation search method inherently satisfies the $SO(3)$ constraint, while the solutions of our sequential methods were built without this  constraint, as they were formulated for general Atlanta frame estimation.   

The results are in Fig.~\ref{Fig:Manhattan}. The accuracy of rotation search method was slightly better than  that of the sequential methods in large noise level. This was because the rotation search method considered the three orthogonal constraints. Nevertheless, the runtime of rotation search was much longer than that of the sequential methods, due to the fact that  the sequential methods had lower dimensionality, tighter bounds  and fewer iterations in the BnB framework.

\subsection{Real Data Experiments}

\subsubsection{NYUv2 Data}

\begin{figure}
	\includegraphics[width=0.45\linewidth]{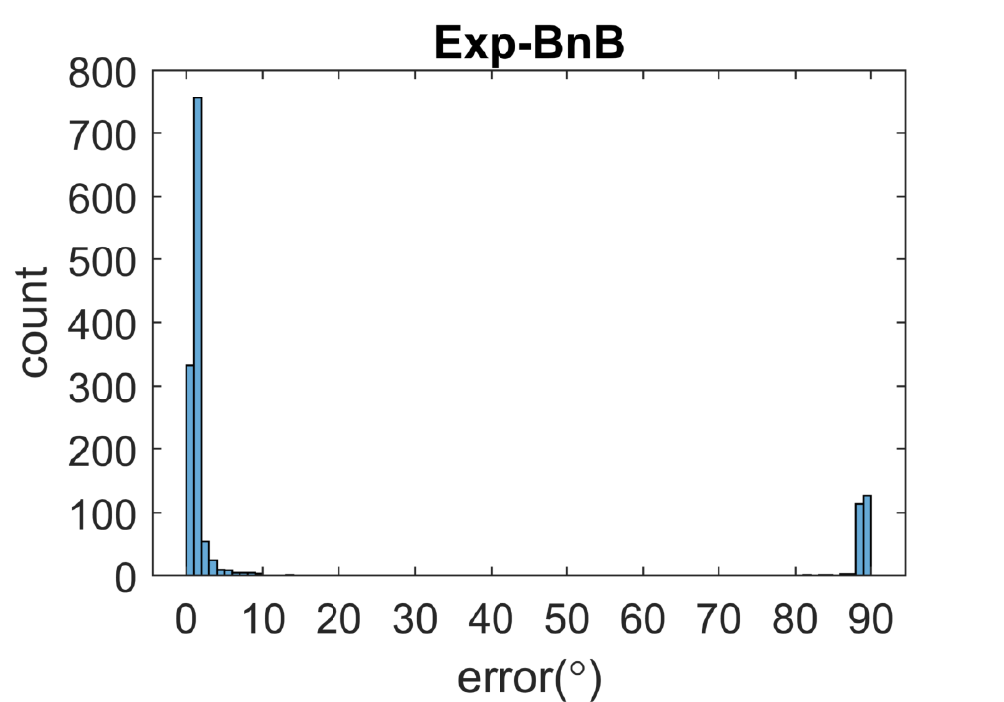}
	\includegraphics[width=0.45\linewidth]{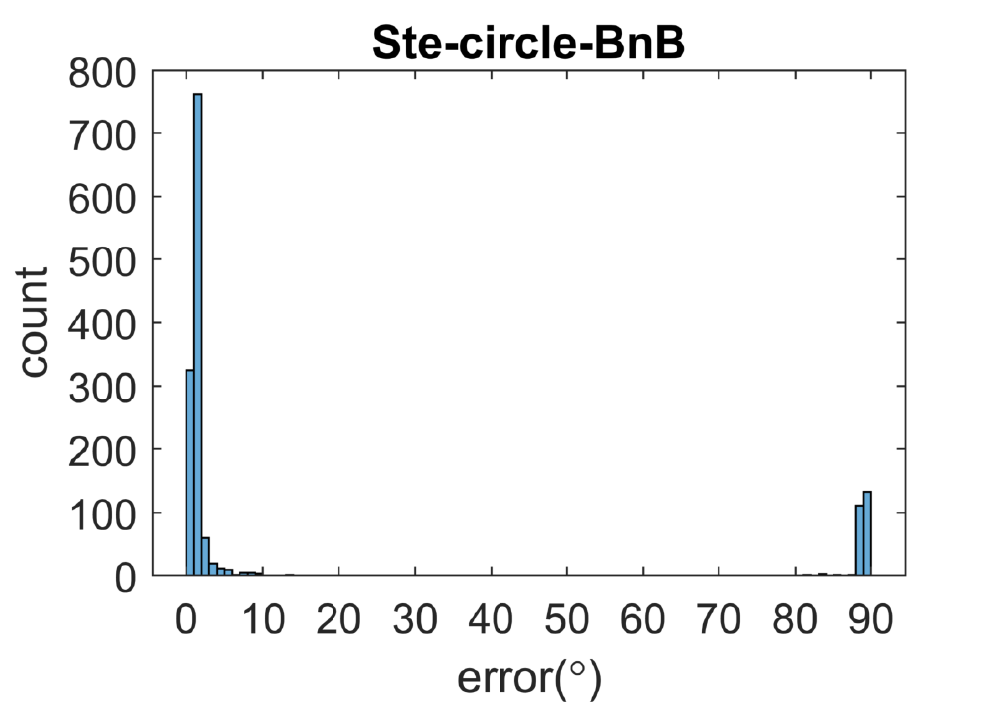}
	\\
	\includegraphics[width=0.45\linewidth]{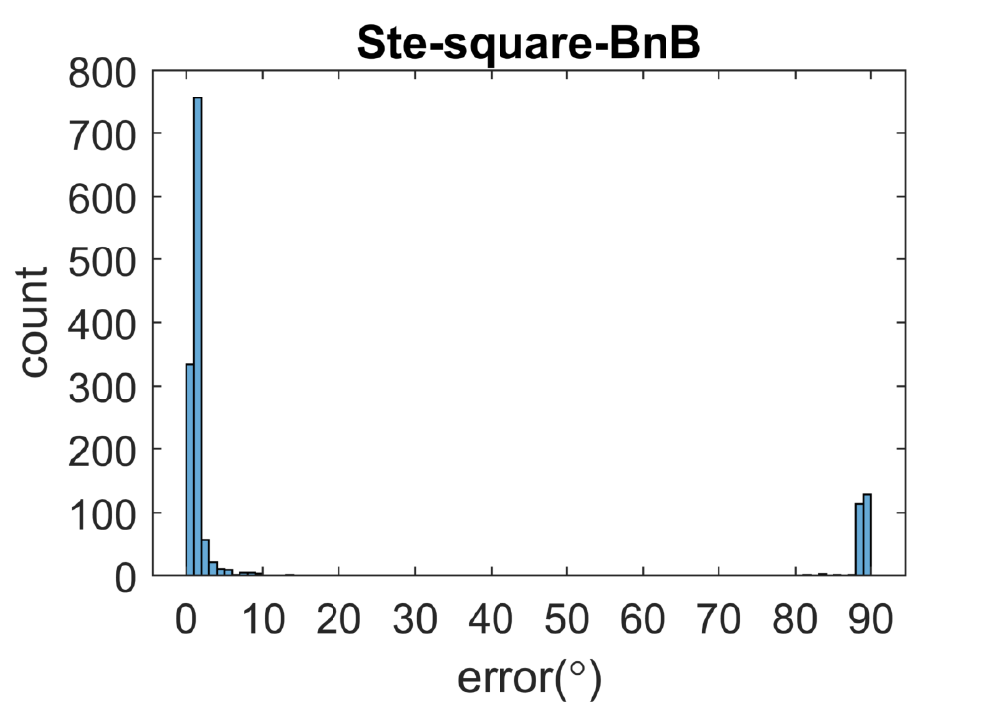}
	\includegraphics[width=0.45\linewidth]{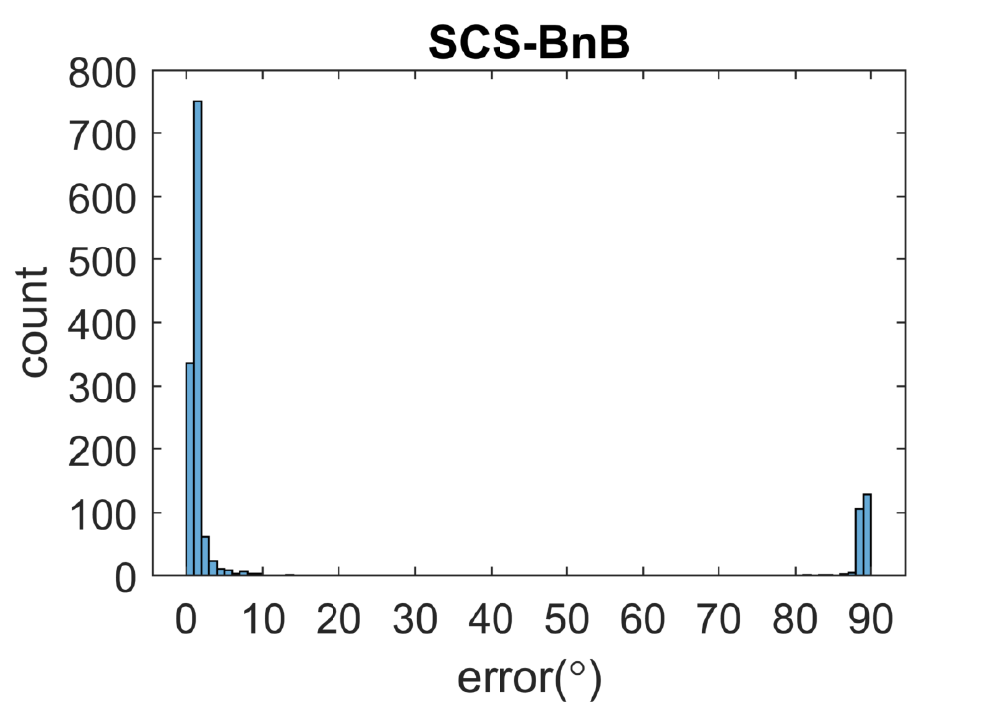}
	\\
	\includegraphics[width=0.45\linewidth]{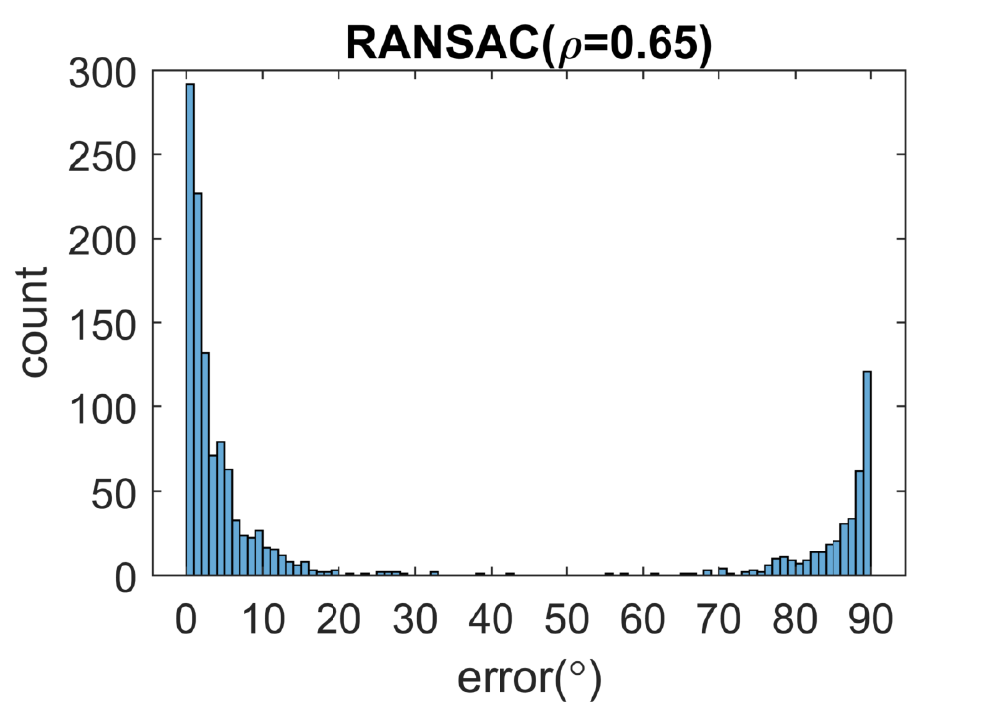}
	\includegraphics[width=0.45\linewidth]{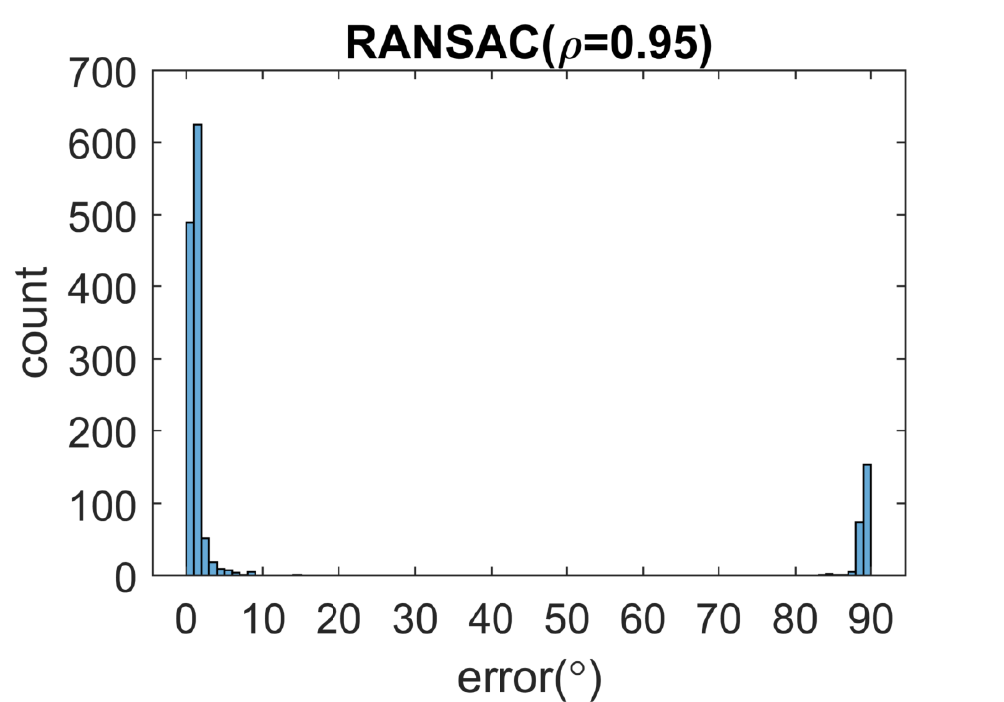}
	\caption{The distribution of error for different methods in NYUv2 data. }\label{Fig:NYU-error}
\end{figure}

\begin{figure}
	~\includegraphics[width=0.8\linewidth]{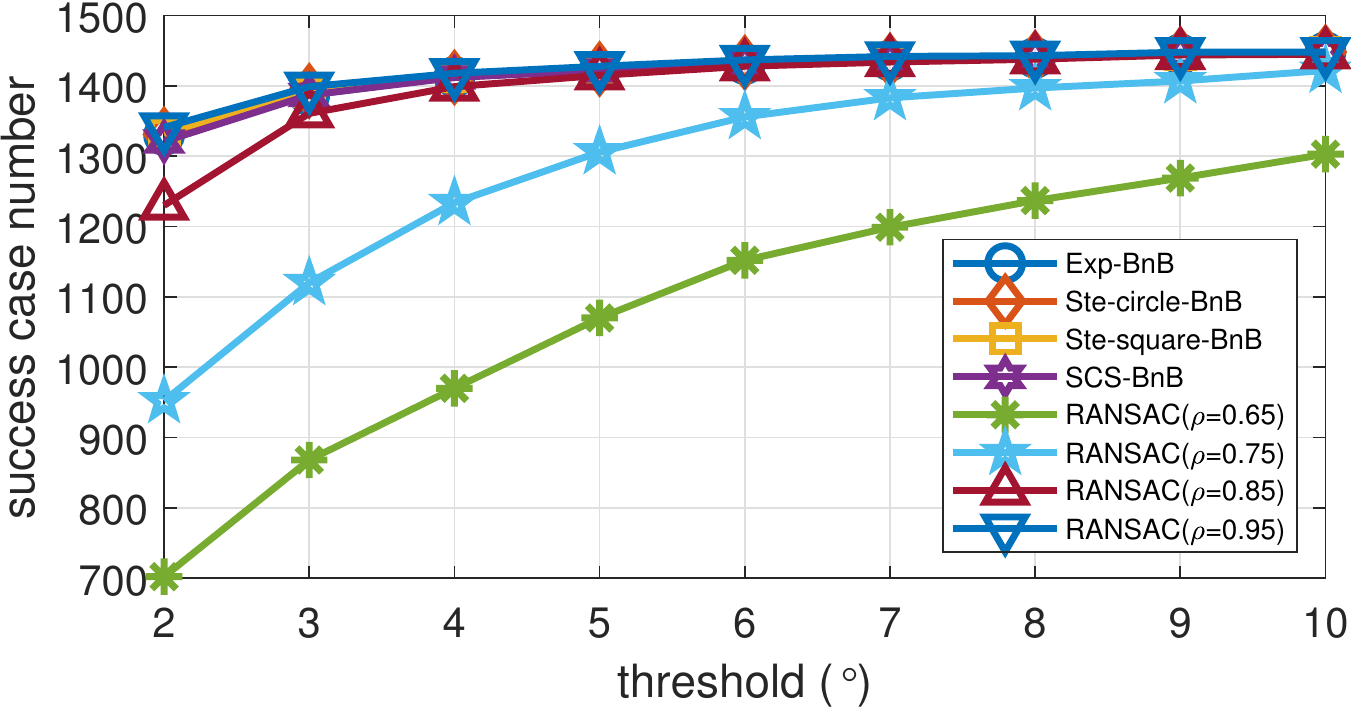}
	\caption{$\tau$-recall curve in NYUv2 data. (Higher is better) }
	\label{Fig:recall}
\end{figure}

We tested our method on the NYUv2 Dataset\cite{Silberman2012IndoorSA}, which contained 1449 RGB images, along with the corresponding depths, as well as camera information.  The data involved a variety of indoor scenes that were considered to be man-made structural world. In our experiments, we utilized the data to estimate the vertical direction of the scenes. Concretely, we generated the normals from the depth image by the Matlab built-in function \textit{pcnormals} and estimated the vertical direction from the downsampled normal data ($N\approx 3000$) for all scenes. The threshold was set to $2^\circ$ in all methods.  For RANSAC, $\rho=\{0.65,0.75,0.85,0.95\}$ were tested since the ground truth of the outlier ratio of each scene was unknown, and the sample iteration was $\Omega$, which was determined by $\rho$ (Eq.(\ref{Eq:omega})).

The distribution of error ($\varepsilon$, see~Eq.\eqref{Eq:error})  is shown in Fig.\ref{Fig:NYU-error}. The results revealed that the estimation errors of the BnB algorithms were all concentrated at $0^\circ$ and $90^\circ$. Because there were some degenerate scenes in the data set, which were degenerated into Manhattan assumption, or even worse, only two main orthogonal frames (Fig.\ref{Fig:failure case}). Estimating vertical direction in such degenerate scenes might return a frame-direction in horizontal plane. Consequently, some errors were concentrated at $90^\circ$. Furthermore, when the outlier proportion $\rho$ was set low, the estimation errors of RANSAC algorithm were not concentrated. When the outlier proportion $\rho$ was set high, the errors were also concentrated. Moreover, to demonstrate the results quantitatively, the $\tau$-recall curve was presented in Fig.~\ref{Fig:recall}, where the success case was satisfied $\varepsilon<\tau$ or $\varepsilon>90^\circ-\tau$.

\begin{figure}
	\centering
	\begin{tabular}{cc}
		\includegraphics[width=0.38\linewidth]{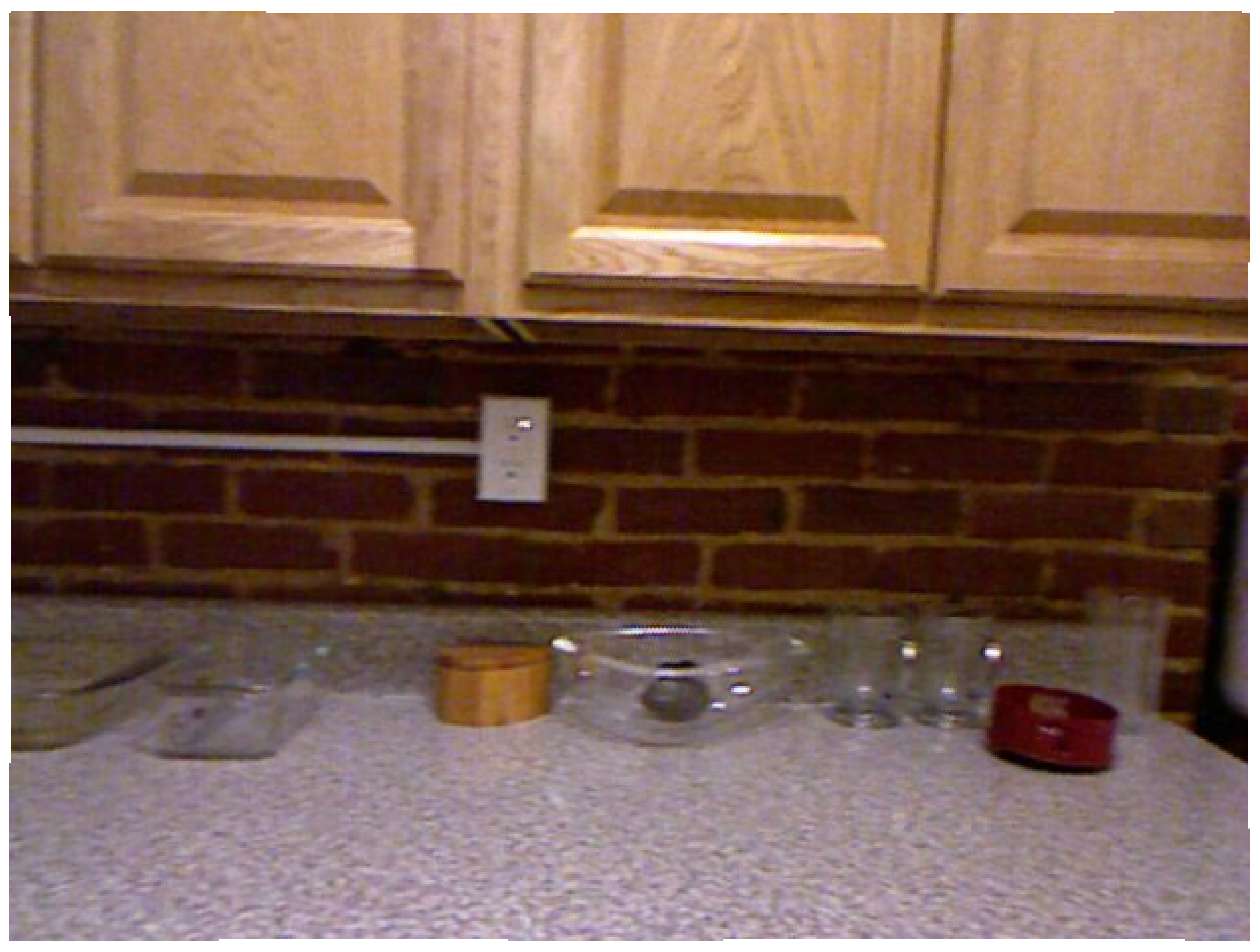}
		
		&\includegraphics[width=0.30\linewidth]{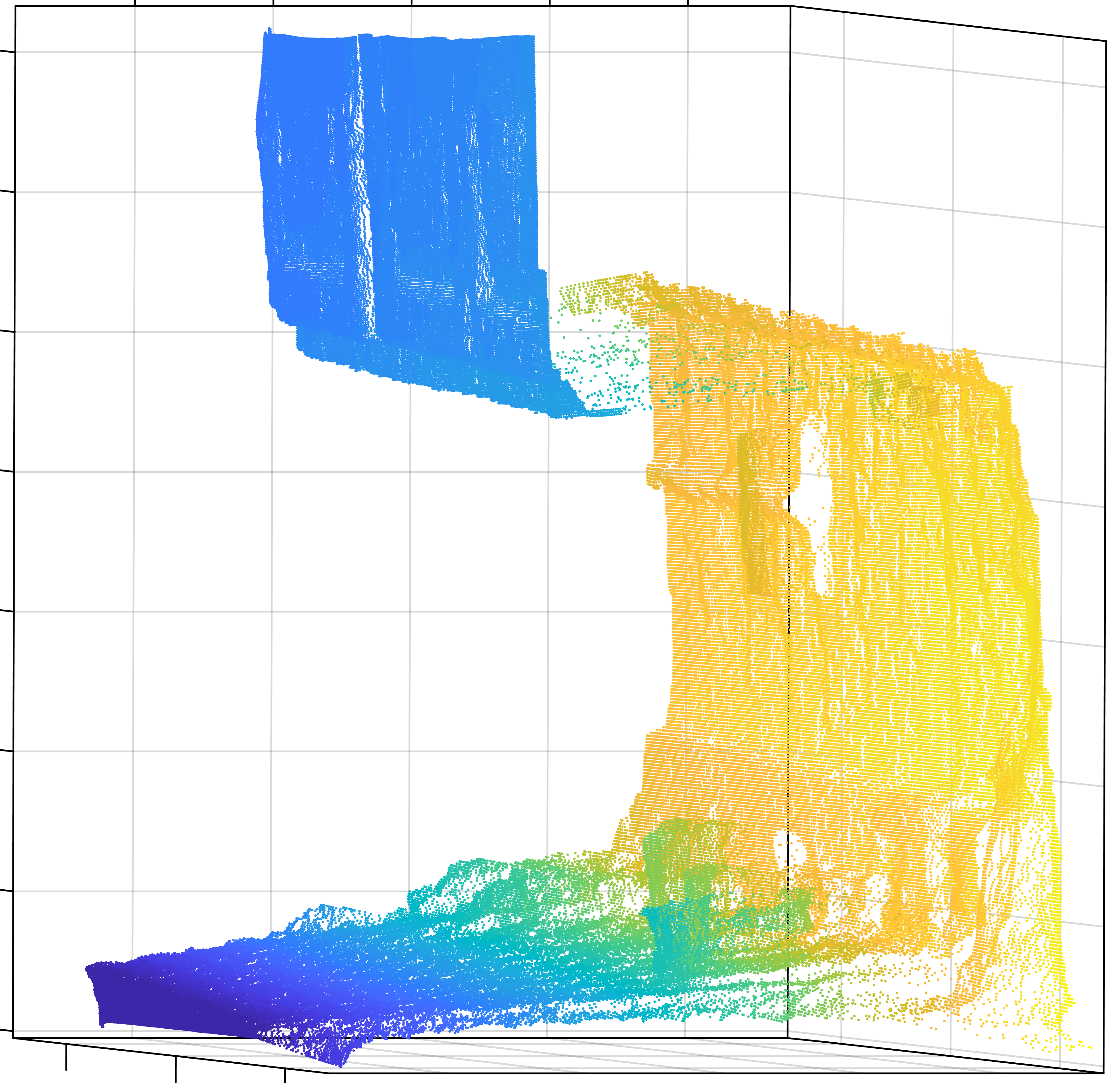}
		\\
		(a)&(b)
	\end{tabular}
	\caption{Degenerate case in NYUv2 data. (a) a degenerate scene that has only two main frames. (b) the point cloud of left scene, which is  viewed from the right side. }
	\label{Fig:failure case}
\end{figure}

\begin{figure}
	
		\includegraphics[width=0.45\linewidth]{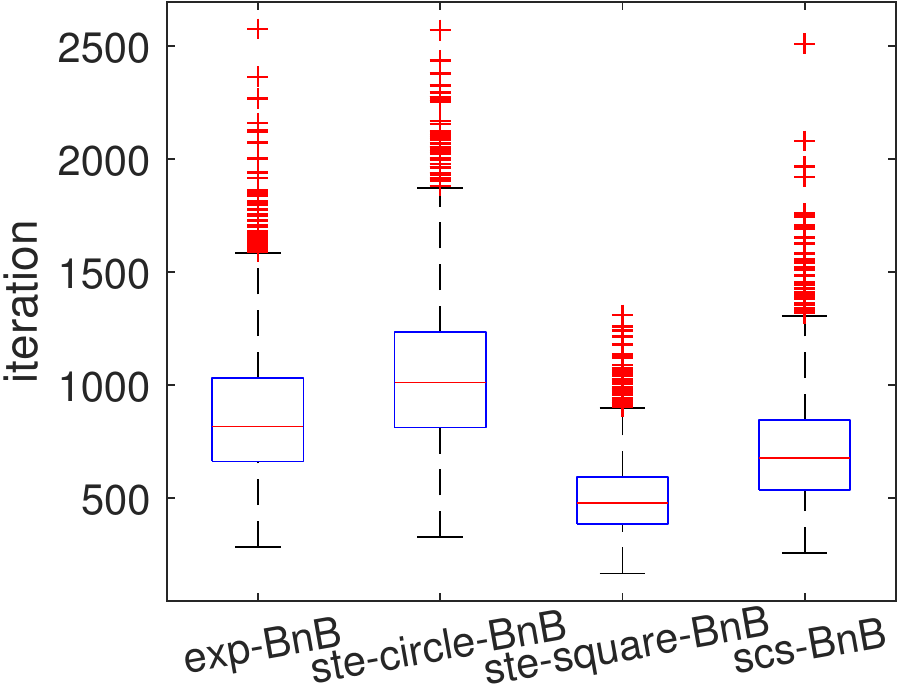}
		\includegraphics[width=0.45\linewidth]{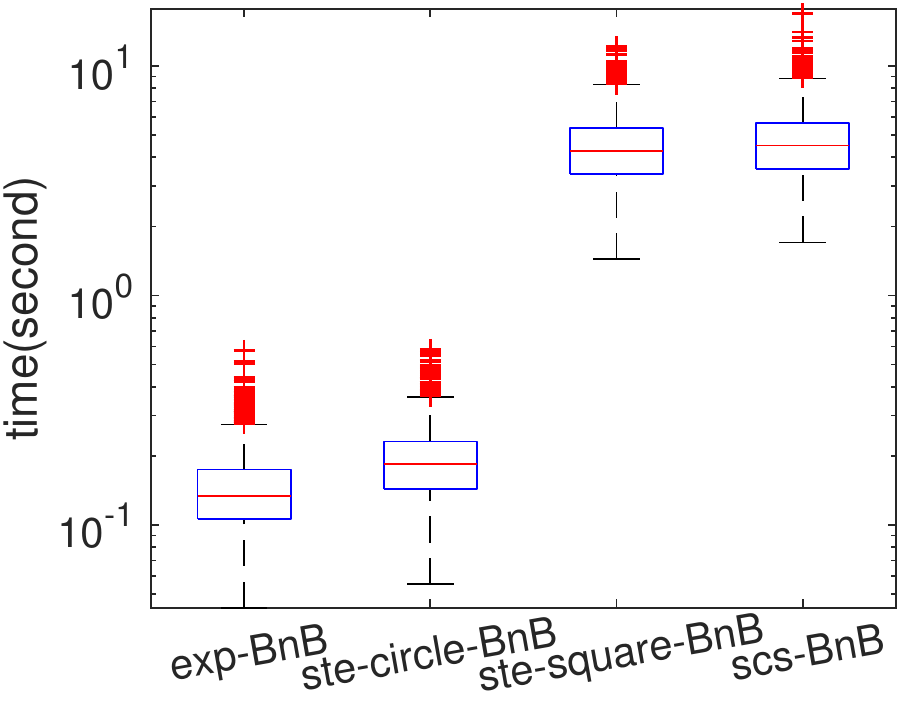}
	
	\caption{The distribution of iteration and runtime in NYUv2 data. }\label{Fig:NYU-iter}
\end{figure}

\begin{table}
\centering
\caption{Median runtime and iteration of different methods in NYUv2 data.}
\begin{tabular}{lcc}
\toprule
{Methods}& median time(s)&median iteration\\
\midrule
Exp-BnB&0.134&816\\
Ste-circle-BnB&0.184&1010\\
Ste-square-BnB&4.256&478\\
SCS-BnB&4.500&675\\
RANSAC($\rho=0.65$)&0.007&36\\
RANSAC($\rho=0.75$)&0.013&72\\
RANSAC($\rho=0.85$)&0.038&203\\
RANSAC($\rho=0.95$)&0.344&1840\\

\bottomrule
\end{tabular}
\label{TABLE:talbe-2}
\end{table}

Furthermore, the four bounds in $\mathbb{S}^2$ had different efficiency. Specifically, the distribution of iteration and runtime in NYUv2 data are in Fig.~\ref{Fig:NYU-iter}. More specifically, the median runtime and iteration can be found in Table~\ref{TABLE:talbe-2}. Obviously, the exp-BnB algorithm was the most efficient.  On the other hand, RANSAC ran very fast when the outlier proportion $\rho$ was set low, however, it might return incorrect results. If the outlier proportion $\rho$ was set high ($\rho=0.95$), its runtime was longer than that of  the exp-BnB algorithm. Besides, to compare with rotation search, we directly quote the results from~\cite{joo2019robust}. With rotation search bounds, it needs 117.06s averagely to estimate Manhattan frames for each scene without input sampling. However, with an efficient bounds computation method proposed in~\cite{joo2019robust}, it needs only 0.07s on average.

\subsubsection{Outdoor Data}

In this part, we verified the validity of our methods with the outdoor scene.
The data set~\cite{Robotic_3D_Scan_Repository} was recorded in the old town of Bremen, Germany (see Fig.(\ref{Fig:bremen})). It
contained 13 3D scans, each with up to 22,500,000 points. Estimating the vertical direction first might be useful to register the scenes~\cite{cai2019practical}. For each scene, it was considered as an Atlanta world  and $[0,0,-1]^T$ was set as the ground truth of vertical direction roughly. We firstly down-sampled the inputs using Matlab built-in function \textit{pcdownsample}. More specifically, a box grid filter, whose input \textit{gridStep} was 0.25, was used to reduce the inputs ($N\approx400,000$). After that their normals were computed by \textit{pcnormals}, and lastly the vertical direction was estimated from the obtained normals. $\rho=\{0.8,0.9\}$ were set in RANSAC and inlier threshold $\tau=1^\circ$ for all methods in this experiments. 

%

The results can be found in Table~\ref{Table-3} (see appendix~H for each scene results). Note that the ground truth for vertical direction was roughly set, and the errors were only indicating that the vertical direction estimation results were roughly correct. In this outdoor settings, all bounds in $\mathbb{S}^2$ can be nested into the BnB algorithm to globally estimate the vertical direction. Furthermore, the results showed that exp-BnB and ste-circle-BnB algorithm had similar performance and were more efficient. Ste-square-BnB had the least iterations among all the methods, however, it needed more time to calculate the bounds. Furthermore, SCS-BnB algorithm needed much more time to estimate the vertical direction in this experiments. Note that RANSAC could also obtain similar results in this experimental settings. Besides, rotation search method could not terminate in 1800s (30min) in each scene. However, according to the results in~\cite{joo2019globally}, with the help of accelerating method, it takes about 80s to estimate Atlanta frames in the whole scene.
\begin{figure}
	\centering
	\includegraphics[width=0.65\linewidth]{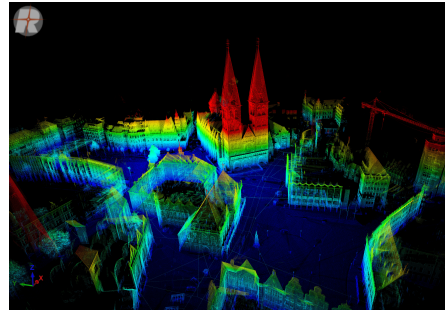}
	\caption{The whole scene of the Bremen city data, which are merged using markers as tie points.~\cite{Robotic_3D_Scan_Repository}}
	\label{Fig:bremen}
\end{figure}

 \begin{table}
 	\centering
 	\caption{Vertical direction estimation results in outdoor data.}
 	\begin{tabular}{lccc}
 		\toprule
 		Methods&median time(s)& iteration& median error($^\circ$)\\
 		\midrule
 		Exp-BnB&1.378&239&1.167\\
 		Ste-circle-BnB&1.182&223&1.141\\
 		Ste-square-BnB&88.470&118&1.141\\
 		SCS-BnB&153.356 &219&1.129 \\
 		RANSAC($\rho=0.8$)&0.941&113&1.140\\
 		RANSAC($\rho=0.9$)&3.792&459&1.173\\
 		\bottomrule
 	\end{tabular}
 \label{Table-3}
 \end{table}

\section{Conclusion}
In this paper, we propose a novel method  for estimating the vertical direction in Atlanta world. It can get the globally optimal solution by applying BnB algorithm, without requiring any prior knowledge of the number of frames. Since estimating vertical direction is inherently a two dimensional problem, we propose new bounds in $\mathbb{S}^2$ for BnB which are different from the conventional bounds in rotation search.


The experimental results show that all the bounds (in $\mathbb{S}^2$ or $SO(3)$) can be nested inside the BnB algorithm to obtain the global solution, and the bounds in $\mathbb{S}^2$ outperform the bounds in $SO(3)$, which is the state-of-the-art technique, for estimating vertical direction globally. Furthermore, these four  bounds in $\mathbb{S}^2$ have different performance. Generally, exp-BnB and ste-circle-BnB have similar performance and are more efficient. Moreover, although ste-square-BnB and SCS-BnB have tighter bounds, they are rather inefficient because of the heavy computational burden.  In addition to the quality of the bounds, appropriate parametrization of searching domain is also an important factor of the efficiency of the BnB algorithm. This is why ste-square-BnB is more efficient than SCS-BnB algorithm.

Lastly, since the ste-square-BnB  has the least iterations, there may be a hope to accelerate the calculation of the  ste-square bounds  to obtain a faster BnB algorithm in further work. In addition, since the  ste-square bounds are very tight in $\mathbb{S}^2$ according to the experimental results, similarly, there may be tighter provable bounds in rotation search ($\mathbb{S}^3$)~\cite{yang2016go,campbell2018globally}.


\ifCLASSOPTIONcompsoc
  \section*{Acknowledgments}
\else
  \section*{Acknowledgment}
\fi

The research leading to these results has partially received funding from the European Unions Horizon 2020 Research and Innovation Program under Grant Agreement No. 785907 (HBP SGA2), from the program of Tongji Hundred Talent Research Professor 2018, and from the Shanghai AI Innovative Development Project 2018. Yinlong, Liu
is funded by Chinese Scholarship Council (CSC).

\ifCLASSOPTIONcaptionsoff
  \newpage
\fi



\bibliographystyle{IEEEtran}
%

\bibliography{egbib}

\begin{thebibliography}{10}
\providecommand{\url}[1]{#1}
\csname url@samestyle\endcsname
\providecommand{\newblock}{\relax}
\providecommand{\bibinfo}[2]{#2}
\providecommand{\BIBentrySTDinterwordspacing}{\spaceskip=0pt\relax}
\providecommand{\BIBentryALTinterwordstretchfactor}{4}
\providecommand{\BIBentryALTinterwordspacing}{\spaceskip=\fontdimen2\font plus
\BIBentryALTinterwordstretchfactor\fontdimen3\font minus
  \fontdimen4\font\relax}
\providecommand{\BIBforeignlanguage}[2]{{%
\expandafter\ifx\csname l@#1\endcsname\relax
\typeout{** WARNING: IEEEtran.bst: No hyphenation pattern has been}%
\typeout{** loaded for the language `#1'. Using the pattern for}%
\typeout{** the default language instead.}%
\else
\language=\csname l@#1\endcsname
\fi
#2}}
\providecommand{\BIBdecl}{\relax}
\BIBdecl

\bibitem{straub2018manhattan}
J.~Straub, O.~Freifeld, G.~Rosman, J.~J. Leonard, and J.~W. Fisher, ``The
  manhattan frame model—manhattan world inference in the space of surface
  normals,'' \emph{IEEE transactions on pattern analysis and machine
  intelligence}, vol.~40, no.~1, pp. 235--249, 2018.

\bibitem{schindler2004atlanta}
G.~Schindler and F.~Dellaert, ``Atlanta world: An expectation maximization
  framework for simultaneous low-level edge grouping and camera calibration in
  complex man-made environments,'' in \emph{Proceedings of the 2004 IEEE
  Computer Society Conference on Computer Vision and Pattern Recognition, 2004.
  CVPR 2004.}, vol.~1.\hskip 1em plus 0.5em minus 0.4em\relax IEEE, 2004, pp.
  I--I.

\bibitem{joo2018globally}
K.~Joo, T.-H. Oh, I.~So~Kweon, and J.-C. Bazin, ``Globally optimal inlier set
  maximization for atlanta frame estimation,'' in \emph{Proceedings of the IEEE
  Conference on Computer Vision and Pattern Recognition}, 2018, pp. 5726--5734.

\bibitem{joo2019globally}
K.~Joo, T.-H. Oh, I.~S. Kweon, and J.-C. Bazin, ``Globally optimal inlier set
  maximization for atlanta world understanding,'' \emph{IEEE transactions on
  pattern analysis and machine intelligence}, 2019.

\bibitem{hedau2009recovering}
V.~Hedau, D.~Hoiem, and D.~Forsyth, ``Recovering the spatial layout of
  cluttered rooms,'' in \emph{2009 IEEE 12th international conference on
  computer vision}.\hskip 1em plus 0.5em minus 0.4em\relax IEEE, 2009, pp.
  1849--1856.

\bibitem{sunderhauf2012switchable}
N.~S{\"u}nderhauf and P.~Protzel, ``Switchable constraints for robust pose
  graph slam,'' in \emph{2012 IEEE/RSJ International Conference on Intelligent
  Robots and Systems}.\hskip 1em plus 0.5em minus 0.4em\relax IEEE, 2012, pp.
  1879--1884.

\bibitem{zhou2015structslam}
H.~Zhou, D.~Zou, L.~Pei, R.~Ying, P.~Liu, and W.~Yu, ``Structslam: Visual slam
  with building structure lines,'' \emph{IEEE Transactions on Vehicular
  Technology}, vol.~64, no.~4, pp. 1364--1375, 2015.

\bibitem{magri2016multiple}
L.~Magri and A.~Fusiello, ``Multiple model fitting as a set coverage problem,''
  in \emph{Proceedings of the IEEE conference on computer vision and pattern
  recognition}, 2016, pp. 3318--3326.

\bibitem{barath2018multi}
D.~Barath and J.~Matas, ``Multi-class model fitting by energy minimization and
  mode-seeking,'' in \emph{Proceedings of the European Conference on Computer
  Vision (ECCV)}, 2018, pp. 221--236.

\bibitem{amayo2018geometric}
P.~Amayo, P.~Pini{\'e}s, L.~M. Paz, and P.~Newman, ``Geometric multi-model
  fitting with a convex relaxation algorithm,'' in \emph{Proceedings of the
  IEEE Conference on Computer Vision and Pattern Recognition}, 2018, pp.
  8138--8146.

\bibitem{kim2017multi}
S.~Kim and R.~Manduchi, ``Multi-planar fitting in an indoor manhattanworld,''
  in \emph{2017 IEEE Winter Conference on Applications of Computer Vision
  (WACV)}.\hskip 1em plus 0.5em minus 0.4em\relax IEEE, 2017, pp. 11--19.

\bibitem{tardif2009non}
J.-P. Tardif, ``Non-iterative approach for fast and accurate vanishing point
  detection,'' in \emph{2009 IEEE 12th International Conference on Computer
  Vision}.\hskip 1em plus 0.5em minus 0.4em\relax IEEE, 2009, pp. 1250--1257.

\bibitem{bazin2012globally-ACCV}
J.-C. Bazin, Y.~Seo, and M.~Pollefeys, ``Globally optimal consensus set
  maximization through rotation search,'' in \emph{Asian Conference on Computer
  Vision}.\hskip 1em plus 0.5em minus 0.4em\relax Springer, 2012, pp. 539--551.

\bibitem{bazin2012globally-CVPR}
J.-C. Bazin, Y.~Seo, C.~Demonceaux, P.~Vasseur, K.~Ikeuchi, I.~Kweon, and
  M.~Pollefeys, ``Globally optimal line clustering and vanishing point
  estimation in manhattan world,'' in \emph{2012 IEEE Conference on Computer
  Vision and Pattern Recognition}.\hskip 1em plus 0.5em minus 0.4em\relax IEEE,
  2012, pp. 638--645.

\bibitem{joo2019robust}
K.~Joo, T.-H. Oh, J.~Kim, and I.~S. Kweon, ``Robust and globally optimal
  manhattan frame estimation in near real time,'' \emph{IEEE transactions on
  pattern analysis and machine intelligence}, vol.~41, no.~3, pp. 682--696,
  2019.

\bibitem{gupta2013perceptual}
S.~Gupta, P.~Arbelaez, and J.~Malik, ``Perceptual organization and recognition
  of indoor scenes from rgb-d images,'' in \emph{Proceedings of the IEEE
  Conference on Computer Vision and Pattern Recognition}, 2013, pp. 564--571.

\bibitem{taylor2013parsing}
C.~J. Taylor and A.~Cowley, ``Parsing indoor scenes using rgb-d imagery,'' in
  \emph{Robotics: Science and Systems}, vol.~8, 2013, pp. 401--408.

\bibitem{cai2019practical}
Z.~Cai, T.-J. Chin, A.~P. Bustos, and K.~Schindler, ``Practical optimal
  registration of terrestrial lidar scan pairs,'' \emph{ISPRS journal of
  photogrammetry and remote sensing}, vol. 147, pp. 118--131, 2019.

\bibitem{bishop2006pattern}
C.~M. Bishop, \emph{Pattern recognition and machine learning}.\hskip 1em plus
  0.5em minus 0.4em\relax springer, 2006.

\bibitem{lee2017line}
G.~H. Lee, ``Line association and vanishing point estimation with binary
  quadratic programming,'' in \emph{2017 International Conference on 3D Vision
  (3DV)}.\hskip 1em plus 0.5em minus 0.4em\relax IEEE, 2017, pp. 584--592.

\bibitem{li20073d}
H.~Li and R.~Hartley, ``The 3d-3d registration problem revisited,'' in
  \emph{2007 IEEE 11th international conference on computer vision}.\hskip 1em
  plus 0.5em minus 0.4em\relax IEEE, 2007, pp. 1--8.

\bibitem{antunes2013global}
M.~Antunes and J.~P. Barreto, ``A global approach for the detection of
  vanishing points and mutually orthogonal vanishing directions,'' in
  \emph{Proceedings of the IEEE Conference on Computer Vision and Pattern
  Recognition}, 2013, pp. 1336--1343.

\bibitem{fischler1981random}
M.~A. Fischler and R.~C. Bolles, ``Random sample consensus: a paradigm for
  model fitting with applications to image analysis and automated
  cartography,'' \emph{Communications of the ACM}, vol.~24, no.~6, pp.
  381--395, 1981.

\bibitem{choi1997performance}
S.~Choi, T.~Kim, and W.~Yu, ``Performance evaluation of ransac family,''
  \emph{Journal of Computer Vision}, vol.~24, no.~3, pp. 271--300, 1997.

\bibitem{raguram2013usac}
R.~Raguram, O.~Chum, M.~Pollefeys, J.~Matas, and J.-M. Frahm, ``Usac: a
  universal framework for random sample consensus,'' \emph{IEEE transactions on
  pattern analysis and machine intelligence}, vol.~35, no.~8, pp. 2022--2038,
  2013.

\bibitem{magri2014t}
L.~Magri and A.~Fusiello, ``T-linkage: A continuous relaxation of j-linkage for
  multi-model fitting,'' in \emph{Proceedings of the IEEE conference on
  computer vision and pattern recognition}, 2014, pp. 3954--3961.

\bibitem{toldo2008robust}
R.~Toldo and A.~Fusiello, ``Robust multiple structures estimation with
  j-linkage,'' in \emph{European conference on computer vision}.\hskip 1em plus
  0.5em minus 0.4em\relax Springer, 2008, pp. 537--547.

\bibitem{hartley2009global}
R.~I. Hartley and F.~Kahl, ``Global optimization through rotation space
  search,'' \emph{International Journal of Computer Vision}, vol.~82, no.~1,
  pp. 64--79, 2009.

\bibitem{parra2016_thesis}
{\'A}.~Parra, ``Robust rotation search in computer vision,'' Ph.D.
  dissertation, 2016.

\bibitem{yang2016go}
J.~Yang, H.~Li, D.~Campbell, and Y.~Jia, ``Go-icp: A globally optimal solution
  to 3d icp point-set registration,'' \emph{IEEE transactions on pattern
  analysis and machine intelligence}, vol.~38, no.~11, pp. 2241--2254, 2016.

\bibitem{campbell2018globally}
D.~J. Campbell, L.~Petersson, L.~Kneip, and H.~Li, ``Globally-optimal inlier
  set maximisation for camera pose and correspondence estimation,'' \emph{IEEE
  transactions on pattern analysis and machine intelligence}, 2018.

\bibitem{seo2009branch}
Y.~Seo, Y.-J. Choi, and S.~W. Lee, ``A branch-and-bound algorithm for globally
  optimal calibration of a camera-and-rotation-sensor system,'' in \emph{2009
  IEEE 12th International Conference on Computer Vision}.\hskip 1em plus 0.5em
  minus 0.4em\relax IEEE, 2009, pp. 1173--1178.

\bibitem{heller2016globally}
J.~Heller, M.~Havlena, and T.~Pajdla, ``Globally optimal hand-eye calibration
  using branch-and-bound,'' \emph{IEEE Transactions on Pattern Analysis and
  Machine Intelligence}, vol.~38, no.~5, pp. 1027--1033, 2016.

\bibitem{yang2014optimal}
J.~Yang, H.~Li, and Y.~Jia, ``Optimal essential matrix estimation via
  inlier-set maximization,'' in \emph{European Conference on Computer
  Vision}.\hskip 1em plus 0.5em minus 0.4em\relax Springer, 2014, pp. 111--126.

\bibitem{bustos2016fast}
A.~P. Bustos, T.-J. Chin, A.~Eriksson, H.~Li, and D.~Suter, ``Fast rotation
  search with stereographic projections for 3d registration,'' \emph{IEEE
  Transactions on Pattern Analysis \& Machine Intelligence}, no.~11, pp.
  2227--2240, 2016.

\bibitem{Straub_2017_CVPR}
J.~Straub, T.~Campbell, J.~P. How, and J.~W. Fisher, III, ``Efficient global
  point cloud alignment using bayesian nonparametric mixtures,'' in \emph{The
  IEEE Conference on Computer Vision and Pattern Recognition (CVPR)}, July
  2017.

\bibitem{hartley2013rotation}
R.~Hartley, J.~Trumpf, Y.~Dai, and H.~Li, ``Rotation averaging,''
  \emph{International journal of computer vision}, vol. 103, no.~3, pp.
  267--305, 2013.

\bibitem{fisher1993statistical}
N.~I. {Fisher}, \emph{Statistical Analysis of Circular Data}, 1993.

\bibitem{Johnson1977The}
D.~S. Johnson and F.~P. Preparata, ``The densest hemisphere problem.''
  \emph{Theoretical Computer Science}, vol.~6, no.~1, pp. 93--107, 1977.

\bibitem{Chin2018RobustFI}
T.-J. Chin, Z.~Cai, and F.~Neumann, ``Robust fitting in computer vision: Easy
  or hard?'' in \emph{ECCV}, 2018.

\bibitem{Morrison2016BranchandboundAA}
D.~R. Morrison, S.~H. Jacobson, J.~J. Sauppe, and E.~C. Sewell,
  ``Branch-and-bound algorithms: A survey of recent advances in searching,
  branching, and pruning,'' \emph{Discrete Optimization}, vol.~19, pp. 79--102,
  2016.

\bibitem{abbaspour2007basic}
H.~Abbaspour and M.~Moskowitz, \emph{Basic Lie Theory}.\hskip 1em plus 0.5em
  minus 0.4em\relax World Scientific Publishing Company, 2007.

\bibitem{sola2018micro}
J.~Sol{\`a}, J.~Deray, and D.~Atchuthan, ``A micro lie theory for state
  estimation in robotics,'' \emph{arXiv preprint arXiv:1812.01537}, 2018.

\bibitem{stereo-formu}
\BIBentryALTinterwordspacing
D.~R. Wilkins. M\"obius transformations and stereographic projection. [Online].
  Available:
  \url{https://www.maths.tcd.ie/~dwilkins/Courses/MA232A/MA232A_Mich2017/StWrapper.pdf}
\BIBentrySTDinterwordspacing

\bibitem{cart2sph}
\BIBentryALTinterwordspacing
 [Online]. Available:
  \url{https://ww2.mathworks.cn/help/matlab/ref/cart2sph.html}
\BIBentrySTDinterwordspacing

\bibitem{sph2cart}
\BIBentryALTinterwordspacing
 [Online]. Available:
  \url{https://ww2.mathworks.cn/help/matlab/ref/sph2cart.html}
\BIBentrySTDinterwordspacing

\bibitem{needham1998visual}
T.~Needham, \emph{Visual complex analysis}.\hskip 1em plus 0.5em minus
  0.4em\relax Oxford University Press, 1998.

\bibitem{bustos2017guaranteed}
{\'A}.~P. Bustos and T.-J. Chin, ``Guaranteed outlier removal for point cloud
  registration with correspondences,'' \emph{IEEE transactions on pattern
  analysis and machine intelligence}, vol.~40, no.~12, pp. 2868--2882, 2017.

\bibitem{Silberman2012IndoorSA}
N.~Silberman, D.~Hoiem, P.~Kohli, and R.~Fergus, ``Indoor segmentation and
  support inference from rgbd images,'' in \emph{ECCV}, 2012.

\bibitem{Robotic_3D_Scan_Repository}
\BIBentryALTinterwordspacing
Robotic 3d scan repository. [Online]. Available:
  \url{http://kos.informatik.uni-osnabrueck.de/3Dscans/}
\BIBentrySTDinterwordspacing

\end{thebibliography}
%
%

%




\end{document}